\documentclass[12pt]{article}

\usepackage{arxiv}
\usepackage{stmaryrd}
\usepackage{newpxtext}
\usetikzlibrary{calc}
\usepackage{microtype}
\usepackage{graphicx}
\usepackage{subfigure}
\usepackage{xurl}
\usepackage{booktabs} 
\usepackage{fullpage}
\usepackage{natbib}
\usepackage{hyperref}
\usepackage{tikz}
\usepackage{acronym}
\usepackage{enumitem}
\usepackage{thm-restate}
\usepackage{longtable,booktabs}
\usepackage{wrapfig}

\acrodef{llm}[LLM]{Large Language Model}
\acrodef{mle}[MLE]{Maximum Likelihood Estimate}
\acrodef{sft}[SFT]{Supervised Fine-Tuning}
\acrodef{icl}[ICL]{In-Context Learning}
\acrodef{svd}[SVD]{Singular Value Decomposition}
\acrodef{mfe}[MFE]{Mean-Field Equilibrium}
\acrodef{mfg}[MFG]{Mean-Field Game}
\acrodef{mfcg}[MFCG]{Mean-Field Congestion Game}
\acrodef{lscg}[LSCG]{LLM-Serving Congestion Game}
\acrodef{mf}[MF]{Mean-Field}
\acrodef{ls}[LS]{Least Squares}
\acrodef{snr}[SNR]{Signal-to-Noise Ratio}
\acrodef{smfcg}[SMFCG]{Stackelberg Mean-Field Congestion Game}
\acrodef{slscg}[SLSCG]{Stackelberg LLM-Serving Congestion Game}
\acrodef{smfe}[SMFE]{Stackelberg Mean-Field Equilibrium}

\newcommand{\pre}{\text{pre}}
\newcommand{\sft}{\text{SFT}}
\newcommand{\full}{\text{full}}
\newcommand{\icl}{\text{ICL}}
\newcommand{\crit}{\text{crit}}

\newcommand{\col}{\text{col}}
\newcommand{\inv}{\dagger}
\newcommand{\BR}{\mathrm{BR}}
\newcommand{\sep}{\mathrm{sep}}
\newcommand{\tilsigma}{\tilde{\sigma}}
\newcommand{\myfigwidth}{0.49\textwidth}
\newcommand{\sign}{\mathop{\mathrm{sign}}}

\usepackage{ulem}
\newcounter{finding}

\begin{document}

\title{Supervised Fine-Tuning vs. In-Context Learning:  An Equilibrium Analysis of LLM Personalization under Congestion}
\author{Fengzhuo Zhang\thanks{Yale University, 30 Hillhouse Ave., New Haven, CT 06520, USA, \texttt{%
fengzhuo.zhang@yale.edu}} \and Zhuoran Yang\thanks{Yale University, 30 Hillhouse Ave., New Haven, CT 06520, USA, \texttt{%
zhuoran.yang@yale.edu}}\and Dirk Bergemann\thanks{%
Yale University, 30 Hillhouse Ave., New Haven, CT 06520, USA, \texttt{%
dirk.bergemann@yale.edu}.} }
\date{Feb 6, 2026}
\maketitle
\renewcommand\thefootnote{}\footnotetext{The authors acknowledge financial support from the UK AI Security Institute. Dirk Bergemann acknowledges financial support from NSF SES 2519400 and ONR MURI N00014-24-1-2742. Zhuoran Yang acknowledges financial support from NSF DMS 2413243.}
\begin{abstract}
Large Language Models (LLMs) have revolutionized AI services, but a critical tension emerges: while personalization improves model performance, it consumes scarce computational resources that users must share. When should a user invest in expensive Supervised Fine-Tuning (SFT) versus lightweight In-Context Learning (ICL)? How does congestion from other users' personalization choices reshape these incentives? And what strategies should platforms adopt when offering multiple personalization algorithms?

We develop a tractable framework for LLM serving that captures the statistical-economic trade-offs users face. Our analysis yields several surprising insights. First, we show that ICL and SFT dominate in different regimes, determined by an interplay between pretraining coverage and data signal-to-noise ratios—but congestion can flip these rankings. Second, equilibrium resource consumption exhibits pronounced non-monotonicity: improving pretraining precision reduces the congestion, while broader pretraining coverage and harder tasks sometimes increase it. Third, we prove that offering both personalization methods never hurts the platform's maximal profits, despite potentially increasing computational load.

Experiments with GPT-2 on linear regression tasks validate our theoretical predictions about algorithm performance. Complementing these results, our review of documentation from 21 major AI platforms shows that the share offering both SFT and ICL increased from $9.5\%$ in 2021 to $71.4\%$ in 2025, consistent with our platform-design implications.
\end{abstract}



\section{Introduction}
Every day, millions of users customize \acp{llm} for specialized tasks—financial analysis, clinical decision support, and legal research~\citep{chatterji2025people}. Such personalization is essential because a pretrained model can underperform on domain-specific tasks. However, personalization also creates a fundamental economic tension. \ac{sft} on private data can yield superior performance but requires intensive GPU computation, while \ac{icl} offers a lightweight alternative—providing examples in the prompt—at the cost of limited adaptability. When thousands of users simultaneously personalize models on shared infrastructure, congestion emerges: each user’s choice of method affects not only their own costs but also others’ waiting times.

This paper asks the following questions:
\begin{itemize}
    \item[Q1.] In the absence of congestion (i.e., with unconstrained compute), how should users choose between \ac{sft} and \ac{icl}?
    \item[Q2.] With shared, constrained resources, how do users trade off personalization quality against congestion costs in equilibrium?
    \item[Q3.] When computational resources are scarce, how should platforms design their service menu and pricing?
\end{itemize}
Answering these questions faces three obstacles.
First, existing economics work on LLMs relies on scaling laws or stylized behavioral assumptions, missing the crucial differences between personalization algorithms~\citep{agarwal2025designing,bergemann2025economics}.
Second, \ac{llm} platforms serve millions of users; modeling strategic interactions at this scale requires moving beyond finite-player games. Third, personalization quality depends on statistical primitives—pretraining coverage, task difficulty, and data signal-to-noise ratios—that must be integrated with economic incentives.

In this paper, we develop a continuum-user model of LLM serving built on linear approximations of personalization algorithms (Section~\ref{sec:model}). This framework lets us (i) derive precise error formulas for \ac{icl} and \ac{sft}, (ii) aggregate individual behavior into tractable equilibrium conditions, and (iii) study how platform pricing and algorithm offerings shape system-level outcomes. The linear approximation is not merely convenient: recent empirical work suggests that LLM behavior on linear tasks is broadly representative, and we validate our theoretical predictions experimentally using GPT-2~\citep{akyureklearning,garg2022can}.

\vspace{10pt}
\noindent

\textbf{Main Results.} First, we analyze linear models of personalization (Section~\ref{sec:stat_analysis}). We derive error decomposition results showing that \ac{icl} can only reduce error within the subspace covered by pretraining, while \ac{sft} updates the entire parameter space (Theorem~\ref{thm:err_analysis}). This yields a precise threshold: \ac{sft} dominates when the pretraining coverage exceeds a critical ratio that depends on the \ac{snr} of the personalization data. Below this threshold, \ac{icl}'s conservatism in uncovered directions actually helps (Proposition~\ref{prop:comp}). This answers Q1.

Second, we characterize users’ equilibrium given the platform’s strategy (Section~\ref{sec:mfcg}). We prove existence and uniqueness of the congestion level via a monotonicity argument analogous to classical mean-field monotonicity conditions~\citep{lasry2007mean,graber2023monotonicity} (Theorem~\ref{thm:exist_unique}). Comparative statics are often non-monotone: higher prior precision always reduces congestion, but expanding pretraining coverage can increase congestion when coverage is intermediate and compute is cheap; harder tasks can reduce congestion when noise is high enough to deter personalization; and making \ac{sft} more expensive can initially increase congestion before shifting users to lighter alternatives (Propositions~\ref{prop:pretrain_influence}–\ref{prop:r_sft}). These patterns arise from the interaction between the extensive margin (algorithm choice) and the intensive margin (sample intensity), governed by the coverage–\ac{snr} mechanism in Proposition~\ref{prop:comp}. We also identify when equilibrium congestion is pinned by a single user type (Proposition~\ref{prop:only_one_type}). This answers Q2.

Third, we study the platform’s Stackelberg strategy and algorithm menu (Section~\ref{sec:smfcg}). We show that equilibrium congestion (aggregate resource demand) decreases in the unit resource price (Theorem~\ref{thm:r_decreasing_p}), and that the profit-maximizing price is finite under regularity conditions (Proposition~\ref{prop:s_exist}). We also prove that offering both \ac{sft} and \ac{icl}—rather than only \ac{icl}—never reduces the platform's maximal profit (Theorem~\ref{thm:sft_better}), formalizing why platforms increasingly provide multiple personalization options despite added complexity. This answers Q3.

Finally, we connect the model to practice (Section~\ref{sec:exp}). Training GPT-2 on linear regression tasks—a representative testbed for \acp{llm} behavior~\citep{garg2022can,von2023transformers}—validates our theory: \ac{icl} error plateaus at a bias driven by uncovered dimensions (with linear growth in that subspace), and the \ac{sft}/\ac{icl} ranking flips with sample size as predicted by our threshold. We also summarize the serving strategies of 21 major AI platforms, documenting rapid adoption of \ac{sft} APIs, consistent with our prediction that offering both \ac{sft} and \ac{icl} does not reduce platform profit.

This work contributes to the emerging AI-economics literature by developing a behaviorally grounded model of personalization that goes beyond per-token pricing to capture how algorithm characteristics shape strategic interaction. Our congestion framework connects to classic work on peak-load pricing and infrastructure management~\citep{small2024economics,redding2015transportation}, where network topology, capacity constraints, and pricing jointly determine congestion outcomes. In contrast, our setting highlights a new source of congestion: effective capacity is shaped not only by fixed physical constraints, but also by the algorithmic properties of personalization methods.

\vspace{10pt}
\noindent
\textbf{Related Literature.} Our work contributes to the rapidly growing literature on the economics of \acp{llm}. Prior work studies \acp{llm} from diverse perspectives, including auction and mechanism design for \ac{llm}-generated content~\citep{duetting2024mechanism,soumalias2024truthful,dubey2024auctions,hajiaghayi2024ad,bergemann2025data}, as well as market competition and token pricing~\citep{fish2024algorithmic,bergemann2025economics}. While we also study the market for \acp{llm}, our focus differs in emphasizing how intrinsic \ac{llm} characteristics, captured through a representative behavioral model, shape market outcomes, a dimension largely absent from the existing literature. For a broader survey of the economics of AI, see~\citep{hadfield2025economy}.

Our analysis of resource congestion also relates to the literature on congestion and pricing of congestible infrastructure. A large body of work studies the interaction between peak-load demand, pricing, and service quality in transportation systems~\citep{henderson1974road,kreindler2024peak,patriksson2015traffic,duranton2011fundamental}, where network topology, capacity constraints, spatiotemporal demand patterns, and pricing jointly determine congestion outcomes~\citep{small2024economics,redding2015transportation}. In contrast, our paper studies congestion in \ac{llm} serving, where congestion is driven not by physical network constraints but by the characteristics of the pretrained \ac{llm} and the personalization mechanisms applied to it.

Our continuum user model is closely related to the broad literature on \acp{mfg} and their applications~\citep{carmona2013probabilistic,carmona2018probabilistic}. Work on \acp{mfg} has extensively studied the existence and uniqueness of equilibrium solutions~\citep{lasry2007mean,carmona2016mean,lacker2017limit}. In particular, \cite{guo2019learning} and \cite{lasry2007mean} establish equilibrium uniqueness under contraction and monotonicity conditions, respectively. Our contribution is conceptually related but distinct: we develop an analogue of the monotonicity condition for the congestion mapping, which ensures the uniqueness of the equilibrium congestion level rather than the full equilibrium strategy profile.  A complementary line of work studies learning dynamics and algorithmic convergence to equilibria in \acp{mfg}~\citep{perrin2020fictitious,yardim2023policy,anahtarci2023q,zhang2023learning,zhang2024learning}; see~\cite{carmona2018probabilistic,lauriere2022learning} for a comprehensive survey.

\vspace{10pt}
\noindent
\textbf{Roadmap.} Section~\ref{sec:model} introduces the statistical models of personalization methods and formulates users' personalization choices as a congestion game. Section~\ref{sec:stat_analysis} derives statistical error bounds for the personalization methods and compares them across different pretraining regimes, characterizing users' choices in the absence of congestion. Section~\ref{sec:mfcg} studies the equilibrium behavior induced by resource congestion, establishing equilibrium existence and the uniqueness of the equilibrium congestion level. It further shows that this congestion level can vary non-monotonically with pretraining properties, highlighting the interaction between statistical accuracy and serving congestion. Section~\ref{sec:smfcg} analyzes the platform's Stackelberg strategy and algorithm menu in \ac{llm} serving, showing that offering both \ac{sft} and \ac{icl}, rather than only \ac{icl}, never reduces maximal profit. Finally, Section~\ref{sec:exp} empirically validates the relevance of the statistical model in Sections~\ref{sec:model} and~\ref{sec:stat_analysis}, as well as the platform-strategy conclusions in Section~\ref{sec:smfcg}.

\section{Models of \ac{llm}s Serving with Personalization}\label{sec:model}
We model personalized \ac{llm} services from two complementary perspectives. Statistically, we introduce a stylized linear models of LLM personalization  that capture the core performance–resource trade-offs of various  methods of personalization. From a game-theoretic perspective, we develop a continuum user framework that captures strategic interactions between a large user population and the platform under congestion and pricing, linking algorithmic personalization choices to system-level equilibrium outcomes.

\subsection{Statistical Models of \ac{llm}s Personalization}\label{sec:model_personal}

Before presenting the formal models, we briefly review pretraining, the basis of personalization, and two personalization algorithms: \ac{sft} and \ac{icl}. Pretraining equips an \ac{llm} with broad general knowledge but limited domain specialization. \ac{icl} adapts the model at inference time via prompt-based demonstrations without updating parameters, while \ac{sft} tunes model parameters on domain-specific data, incurring higher computational cost but yielding more persistent personalization.

\vspace{10pt}
\noindent\textbf{Pretraining.} During pretraining, \ac{llm}s are trained to predict the next token given a prefix over a large corpus~\citep{brown2020language,touvron2023llama}, thereby acquiring broad knowledge and high-dimensional language representations~\citep{wang2025muon,meng2022locating}. 
From a modeling perspective, pretraining can be interpreted as learning a prior distribution over linguistic patterns and world knowledge. For example, given the prefix ``The migratory salmon in the Pacific Ocean is'', the likely continuation depends on the latent task of the current context. If the task is taxonomy, the model may assign high probability to continuations such as ``a member of the family Salmonidae''; if the task is ecological outlook, the likely continuation may instead describe population trends, climate risks, or conservation status. Thus, by learning next-token distributions conditional on prefixes, the pretrained \ac{llm} implicitly captures the uncertainty and diversity of tasks in the pretraining corpus. After pretraining, the \ac{llm} generates the next token according to this learned conditional distribution, so different sampled continuations may reflect different latent tasks.

For statistical analysis, we abstract pretraining as a linear regression problem. For each training sample, we represent the prefix by a covariate vector $x_i \in \mathbb{R}^d$, the next-token response by a scalar $y_i\in\mathbb{R}$, and the implicit task parameter by $\theta_i\in \mathbb{R}^d$. The pretraining dataset of size $N$ is generated as follows. First, the task parameters $\theta_i$ are drawn independently from the task prior $\calN(\mu_\theta, \Sigma_\theta)$, where $\mu_\theta\in\mathbb{R}^{d}$ and $\Sigma_\theta\in\mathbb{R}^{d\times d}$. Second, the covariate vectors $x_i$ may be either fixed or randomly sampled from an underlying Gaussian distribution. For clarity of exposition, we focus on the fixed-design setting, in which the covariates are treated as deterministic, for example by being selected in a round-robin manner from an orthonormal basis. When the covariates are sampled i.i.d.\ from a Gaussian distribution, all results hold with high probability; see Appendix~\ref{app:random_cov}. Finally, the response $y_i$, e.g., a scalar representation of ``Salmonidae'', is generated according to
\begin{align}
y_i = x_i^\top \theta_i + \epsilon_i,
\text{ and } \epsilon_i \sim \calN(0, \sigma^2), \text{ for }i\in[N],\label{eq:pretrain_data}
\end{align}
where the noise terms $\epsilon_i$ are i.i.d.\ Gaussian and $[N]={1,\ldots,N}$. The resulting pretraining dataset is $\{(x_i,y_i)\}_{i=1}^{N}$, with
$X = [x_1,\ldots,x_N]^\top \in \mathbb{R}^{N\times d}$ and
$Y = [y_1,\ldots,y_N]^\top \in \mathbb{R}^N$. We emphasize that different data points $(x_i,y_i)$ may correspond to different latent task parameters $\theta_i$, allowing the \ac{llm} to learn shared structure across a distribution of tasks. 

\acp{llm} are pretrained to learn the next-token distribution conditioned on the prefix, thereby implicitly acquiring task-relevant knowledge. Specifically, given a prefix $x\in\bbR^d$, we model the scalar next-token distribution as a Gaussian distribution $\calN(x^{\top}\theta_{\pre}, \|x\|_{\Sigma_{\pre}}^2)$, where $\theta_{\pre}\in\bbR^d$ represents the learned prior mean induced by pretraining, $\Sigma_{\pre}\in\bbR^{d\times d}$ reflects the uncertainty of this learned prior, and $\|x\|_{\Sigma} = \sqrt{x^\top \Sigma x}$ denotes the $\Sigma$-norm induced by any positive semidefinite matrix $\Sigma$. Equivalently, this Gaussian prediction can be interpreted as follows: the \ac{llm} first samples a task-parameter estimate $\htheta_{\pre}\sim\calN(\theta_{\pre},\Sigma_{\pre})$ from the learned task prior and then generates the scalar response $\htheta_{\pre}^\top x$. The parameter $\theta_{\pre}$ is estimated from the pretraining dataset $\{(x_i,y_i)\}_{i=1}^{N}$ by minimizing the prediction error:
\begin{align*}
\theta_{\pre}
= \arg\min_{\theta \in \mathbb{R}^d} \|Y - X\theta\|^2
= (X^\top X)^{\inv} X^\top Y,
\end{align*}
where $(\cdot)^{\inv}$ denotes the Moore--Penrose pseudoinverse. Since $\Sigma_{\pre}$ captures the estimation uncertainty of the learned prior, it is given by the covariance of the estimator $\theta_{\pre}$:
\begin{align}
\Sigma_{\pre}
= (X^\top X)^{\inv} X^\top \Omega X (X^\top X)^{\inv},
\quad
\Omega = \sigma^2 I_N
+ \diag\big(\|x_1\|_{\Sigma_\theta}^2,\ldots,\|x_N\|_{\Sigma_\theta}^2\big),\label{eq:pre_cov}
\end{align}
where $I_{N} \in \mathbb{R}^{N \times N}$ denotes the identity matrix.
The derivation of \eqref{eq:pre_cov} is provided in Appendix~\ref{app:cov_pre}. We next introduce personalization methods that adapt the learned distribution $\calN(\theta_\pre,\Sigma_\pre)$ using personalization data.

\vspace{10pt}

\noindent \textbf{In-Context Learning.} \ac{icl} is a mechanism for personalizing the behavior of \ac{llm}s \emph{without} modifying their parameters~\citep{zzzz23,xie2021explanation}. 
It personalizes model predictions by conditioning the learned prior on task demonstrations provided in the prompt. 
For example, to instruct the model to return the taxonomic family of an animal, one may supply the sequence
\begin{align*}
  [(\text{``salmon''}, \text{``salmonidae''}), (\text{``otter''}, \text{``mustelidae''}), (\text{``panda''}, \, ?)].
\end{align*}
The first three covariate--response pairs constitute the \emph{context} that specifies the target task, while the final input ``panda'' serves as the query. Through these demonstrations, the \ac{llm} infers the underlying task—here, mapping animals to their taxonomic families—and generates a response accordingly. 
Prior work shows that \ac{icl} can be interpreted as approximate Bayesian inference under the learned prior: the model updates its belief about the task given the context and then predicts based on the resulting posterior~\citep{zzzz23,xie2021explanation,hu2024unveiling}.

To formalize the Bayesian inference underlying personalization, we first specify the data model for personalization data.
We model the target task $\theta^* \in \mathbb{R}^d$ as a random variable drawn from $\mathcal{N}(0,\Sigma^*)$. Personalization data consist of covariate--response pairs generated by this fixed task,
\begin{align}
    \tily_i = \tilx_i^\top \theta^{*} + \tilde{\epsilon}_i, \text{ and } \tilde{\epsilon}_i \sim \calN(0, \tilsigma^2),\text{ for } i \in [\tilN],\label{eq:task_data}
\end{align}
where $\tilde{\epsilon}_i$ is i.i.d. Gaussian noise in the personalization data. As in the pretraining model, we assume that the covariates $\tilx_i$ are fixed.
The key difference from the pretraining data is that all personalization samples share the same task parameter $\theta^*$, whereas in pretraining each sample is associated with its own latent task parameter $\theta_i$.
For independently generated samples $\{(\tilx_i,\tily_i)\}_{i=1}^{\tilN}$, we collect
$\tilX = [\tilx_1,\ldots,\tilx_{\tilN}]^\top \in \mathbb{R}^{\tilN \times d}$ and
$\tilY = [\tily_1,\ldots,\tily_{\tilN}]^\top \in \mathbb{R}^{\tilN}$.

According to \cite{xie2021explanation,zzzz23}, when prompted with personalization data $\{(\tilx_i,\tily_i)\}_{i=1}^{\tilN}$, an \ac{llm} performs inference under the learned prior $\calN(\theta_{\pre},\Sigma_{\pre})$, assuming the data are generated by a latent task $\theta \sim \calN(\theta_{\pre},\Sigma_{\pre})$.
That is, the personalized distribution induced by \ac{icl},
$\calN(\theta_{\icl},\Sigma_{\icl})$ (with the dependence on $\tilN$ suppressed for notational simplicity), is the posterior of the learned prior conditioned on the observed personalization data.
Under the Gaussian linear model \eqref{eq:pretrain_data} and \eqref{eq:task_data}, this posterior is given by
\begin{align}
    \theta_{\icl} &= \theta_{\pre} + \Sigma_{\pre}\tilX^{\top} (\tilX \Sigma_{\pre} \tilX^{\top}+\tilsigma^2 I_{\tilN})^{-1}(\tilY-\tilX\theta_{\pre}),\nonumber\\
    \Sigma_{\icl} & = \Sigma_{\pre} - \Sigma_{\pre}\tilX^{\top} (\tilX \Sigma_{\pre} \tilX^{\top}+\tilsigma^2 I_{\tilN})^{-1}\tilX\Sigma_{\pre}.\nonumber
\end{align}
The derivation follows from standard Bayesian updating (Appendix~\ref{app:icl_bayes}). Geometrically, \ac{icl} updates only directions \emph{supported} by the learned prior. If $\theta^*-\theta_{\pre}\in\text{col}(\Sigma_{\pre})$, the target task is fully represented; otherwise, \ac{icl} effectively projects $\theta^*$ onto $\theta_{\pre}+\text{col}(\Sigma_{\pre})$, updating beliefs within this subspace based on observed data. We quantify this effect in Section~\ref{sec:stat_analysis}. As in pretraining, \ac{icl} estimate $\hat{\theta}_{\icl}\sim\calN(\theta_{\icl},\Sigma_{\icl})$ induces predictions for a new query $x$ distributed as $\calN(\theta_{\icl}^\top x,\|x\|_{\Sigma_{\icl}}^2)$.

\vspace{10pt}

\noindent\textbf{Supervised Fine-Tuning.} \ac{sft} adapts a pretrained \ac{llm} to a specific target task by updating model parameters using task-specific data. For example, a model’s performance on animal taxonomy can be improved by fine-tuning it on biological data.
Formally, \ac{sft} is implemented by optimizing a task-level loss starting from the pretrained model as the initialization. Compared with \ac{icl}, this parameter optimization process makes \ac{sft} consume significantly more GPU resources, since parameter updates require additional FLOPs and memory storage.

To model this process, given personalization data
$\{(\tilx_i,\tily_i)\}_{i=1}^{\tilN}$ of the target task $\theta^*\sim\mathcal{N}(0,\Sigma^*)$ generated according to~\eqref{eq:task_data},
\ac{sft} updates the task estimation via a regularized estimator centered at the pretrained \ac{llm} as
\begin{align}
    \!\!\theta_{\sft}
    \!=\! \arg\min_{\theta\in\mathbb{R}^{d}}
    \frac{1}{\tilsigma^2}\|\tilY \!-\! \tilX \theta\|^2
    \!+\! \lambda \|\theta\!-\!\theta_{\pre}\|_{\Sigma_{\pre}^{\inv}}^2
    \!\!\!=\!\! \bigg(\frac{1}{\tilsigma^2}\tilX^{\top}\tilX\!+\!\lambda \Sigma_{\pre}^{\inv}\! \bigg)^{\inv}
      \!\bigg(\frac{1}{\tilsigma^2}\tilX^{\top}\tilY\!+\!\lambda \Sigma_{\pre}^{\inv}\theta_{\pre}\!\bigg).
    \label{eqn:theta_sft}
\end{align}
The regularization term encourages the fine-tuned parameters to remain close to the pretrained model, and
$\lambda \ge 0$ controls the strength of this anchoring. The regularization term acts as a Lagrangian penalty that softly constrains $\theta_{\sft}$ to remain close to the pretrained parameter $\theta_{\pre}$, thereby limiting how far \ac{sft} can adapt the model away from the pretrained solution. Since $\Sigma_\sft$ captures the uncertainty in $\theta_\sft$, it is given by the covariance of $\theta_\sft$ as
\begin{align*}
    \Sigma_{\sft}
    = \bigg(\frac{1}{\tilsigma^2}\tilX^{\top}\tilX+\lambda \Sigma_{\pre}^{\inv}\bigg)^{\inv}
      \bigg(\lambda^2\Sigma_{\pre}^\dagger
      +\frac{1}{\tilsigma^2}\tilX^{\top}\tilX\bigg)
      \bigg(\frac{1}{\tilsigma^2}\tilX^{\top}\tilX+\lambda \Sigma_{\pre}^{\inv} \bigg)^{\inv},
\end{align*}
with the derivation provided in Appendix~\ref{app:cov_sft}.
Analogous to pretraining and \ac{icl}, \ac{sft} model
$\hat{\theta}_{\sft} \sim \mathcal{N}(\theta_{\sft}, \Sigma_{\sft})$
induces the predictive distribution
$\mathcal{N}(\theta_{\sft}^\top x, \|x\|_{\Sigma_{\sft}}^2)$ for a new query $x$.

The parameter $\lambda$ governs the trade-off between task-specific data and the pretrained model.
When $\lambda=0$, the estimator reduces to least squares based solely on the personalization data.
As $\lambda$ goes to infinity, the estimator ignores the task data and collapses to the pretrained parameters $\theta_{\pre}$.

\subsection{A Continuum User Model for LLM Serving}\label{sec:model_serving}
Each user on the platform chooses a personalization algorithm—\ac{icl} or \ac{sft}—and the number of personalization samples $\tilN$ to use. User experience, modeled as user cost, consists of two components: the quality of the \ac{llm}'s response and the monetary and time costs incurred to obtain it.

The first component of user cost depends on the pretrained model, the target task, the personalization algorithm, and the number of personalization samples. We capture this dependence via a user type $t\in\calT$, which summarizes properties of the pretrained model $(\mu_\theta,\Sigma_\theta,\sigma,X)$ and the target task $(\Sigma^*,\tilsigma,\tilX)$. The type space $\calT$ is compact with Borel $\sigma$-algebra $\mathscr{T}$, and user types are distributed according to a probability measure $T$. Each user chooses an action $\bara=(a,\tilN)\in\bar{\calA}=\calA\times[0,\infty)$, where $\calA={\icl,\sft}$ and $\tilN$ denotes the number of personalization samples. For a user of type $t$ taking action $(a,\tilN)$, we define the learning error of the \ac{llm} as
\begin{align}
    E_a(t,\tilN)
    = \bbE_{\theta^{*},\,\hat{\theta}_a(\tilN)}
      \big[\|\theta^{*}-\hat{\theta}_a(\tilN)\|^2\big],\label{eq:err}
\end{align}
where $\hat{\theta}_a(\tilN)\sim\calN(\theta_a(\tilN),\Sigma_a(\tilN))$ is the estimate produced by algorithm $a$ using $\tilN$ samples, and the expectation is taken over both the target task $\theta^{*}\sim\calN(0,\Sigma^*)$ and the estimator. Although $\hat{\theta}_a(\tilN)$ is not directly observable for \ac{llm}s, this error can be interpreted as the average prediction error obtained by querying the \ac{llm} along an orthonormal basis $\{e_i\}_{i=1}^d$, i.e., the average error of $e_i^\top \hat{\theta}_a(\tilN)$. When a user of type $t$ adopts a policy $\pi_t \in \Delta(\bar{\mathcal{A}})$, where $\Delta(\cdot)$ denotes the set of all distributions over a set, the resulting response quality is $\bbE_{(a,\tilN)\sim\pi_t}[E_a(t,\tilN)]$.

The second component of user cost depends on the chosen personalization algorithm and aggregate user behavior. A user taking action $\bar a=(a,\tilN)$ requests $\tilN$ samples using algorithm $a$, which requires $R_a$ units of resource per sample. As discussed in Section~\ref{sec:model_personal}, \ac{sft} is more resource-intensive than \ac{icl}, i.e., $R_{\sft}>R_{\icl}$. The platform charges $p>0$ per unit of resource, yielding a monetary cost of $R_a\tilN p$.
Users also incur a congestion-induced time cost: letting $R$ denote total resource demand, a congestion function $h:\mathbb{R}_{\geq 0}\to[0,\infty)$ maps $R$ to waiting time, yielding a cost $R_a\tilN h(R)$. We impose following \emph{regularity conditions} on $h$: it is non-negative, strictly increasing, and unbounded ($\lim_{x\rightarrow\infty}h(x)=\infty$). These ensure that congestion raises compute costs and that extreme demand becomes prohibitively expensive, consistent with observed service outages~\citep{openai2025serviceoutage}.
Accordingly, a user of type $t$ taking action $\bar a$ at congestion level $R$ incurs a total cost as 
\begin{align}
    C(t,\bar a,R)
    = E_a(t,\tilN) + R_a\tilN\big(p+h(R)\big).\label{eq:cost}
\end{align}
For convenience, define $\bar h(R,p)=p+h(R)$.
When a type-$t$ user adopts a policy $\pi_t$ over $\bar{\calA}=\calA\times\mathbb R_{+}$, her expected cost is $$C(t,\pi_t,R)
    = \sum_{a\in\calA}\int_{0}^{\infty}
      C\bigl(t,(a,\tilN),R\bigr)\,\pi_t(a,\rmd\tilN).$$

To evaluate aggregate resource demand, the population is modeled as a continuum of users, where type-$t\in\calT$ users have mass $T(\rmd t)$. A policy profile is a measurable mapping $\pi:\calT\to\Delta(\bar{\calA})$, assigning each type $t$ a distribution $\pi_t$ over actions. Given a policy profile $\pi$, the congestion level, i.e., aggregate resource demand, is the population-average requested resources: 
\begin{align}
    R(\pi)
    = \sum_{a\in\calA} R_a
      \int_{0}^{\infty} \tilN
      \int_{\calT} \pi_t(a,\rmd\tilN)\, T(\rmd t)
    = \sum_{a\in\calA} R_a
      \int_{0}^{\infty} \tilN\, \bar{\pi}(a,\rmd\tilN),\label{eq:congestion}
\end{align}
where $\bar{\pi}=\int_{\calT}\pi_t T(\rmd t)$ is the induced ``mean'' policy. Each user therefore interacts with the population only through $\bar{\pi}$. We refer to this interaction as a \ac{lscg}. This model belongs to the class of mean-field games~\citep{lasry2007mean,carmona2016mean}.

Within \ac{lscg}, users choose policies to trade off statistical error $E_a(t,\tilN)$ against congestion costs: more resource-intensive algorithms or larger $\tilN$ reduce error but increase aggregate congestion $R(\pi)$ and the associated cost $h(R)$. These trade-offs lead to a equilibrium concept.
\begin{definition}[Equilibrium]
    A pair of a policy and a congestion level $(\pi^*,R^{*})$, where $\pi^{*}:\calT\rightarrow \Delta(\bar{\calA})$ and $R^{*}\in\bbR$, is the  equilibrium if 
    \begin{itemize}
        \item (Optimality) The policy $\pi^{*}$ is optimal given the congestion level $R^{*}$ for almost every user, i.e., for $T$-a.e.\ $t$, we have $\pi^*_t(\BR(t,R^*))=1$, where the best response operator is 
        \begin{align*}
            \BR(t,R)=\arg\min_{(a,\tilN)\in\bar{\mathcal A}} C\big(t,(a,\tilN),R\big).
        \end{align*}
        \item (Consistency) The congestion level $R^{*}$ is induced by the policy $\pi^{*}$, i.e., $R^{*}=R(\pi^{*})$.
    \end{itemize}
\end{definition}
This definition formalizes a Nash fixed point in a continuum setting. Optimality requires that, given congestion level $R^{*}$, almost every user type minimizes total cost $C(t,(a,\tilN),R^{*})$, while consistency requires that $R^{*}$ equals the congestion level induced by the population policy $\pi^{*}$. Equivalently, an equilibrium satisfies $\operatorname{supp}(\pi_t^{*})\subseteq \BR(t,R(\pi^{*}))$ for $T$-a.e.\ $t$.

\section{Statistical Analysis of Personalization Algorithms}\label{sec:stat_analysis}
Which personalization method minimizes cost when users care only about answer quality? This section provides a sharp answer: the comparison depends on two primitives—pretraining coverage of the task space and the signal-to-noise ratio of the user’s data. We will first derive closed-form error expressions $E_a(t,\tilN)$ for each method in~\eqref{eq:err} (Section~\ref{sec:err}), and then establish a precise comparison between \ac{sft} and \ac{icl}(Section~\ref{sec:comp}). Throughout this section, we fix the user type $t$ and the sample size $\tilN$, and suppress this dependence in the notation. 
\subsection{Error Characterizations of \ac{icl} and \ac{sft}}\label{sec:err}
To isolate key differences between \ac{sft} and \ac{icl}, we first introduce several simplifying assumptions.

\begin{assumption}[Spectral alignment]\label{assump:svd}
We assume that the pretraining Gram matrix $X^\top X$, the \ac{sft} Gram matrix $\tilde X^\top \tilde X$, and the prior covariance matrix $\Sigma^\ast$ of $\theta^{*}$ are spectrally aligned: there exists an orthonormal matrix 
$V=[v_1,\cdots,v_d]\in\mathbb{R}^{d\times d}$ such that all three matrices are diagonal in it, i.e.,
\begin{align*}
X^\top X\! =\! V \mathrm{diag}(\pi_1,\cdots\!,\pi_r,0,\cdots\!,0) V^\top,
\tilde X^\top \tilde X \!=\! V \mathrm{diag}(s_1,\dots\!,s_d) V^\top,
\Sigma^\ast \!=\! V \mathrm{diag}(\tau_1,\dots\!,\tau_d) V^\top.
\end{align*}
The parameter $r\leq d$ specifies the rank of pretraining input $X$. For $i \in [r]$, $\pi_i > 0$ are the eigenvalues of $X^\top X$; for $i \in [d]$, $s_i \geq 0$ and $\tau_i \geq 0$ are the eigenvalues of $\tilde X^\top \tilde X$ and $\Sigma^\ast$, respectively.
\end{assumption}

Assumption~\ref{assump:svd} posits that the pretraining Gram matrix $X^\top X$, the \ac{sft} Gram matrix $\tilde X^\top \tilde X$, and the task prior covariance $\Sigma^\ast$ share a common orthonormal eigenbasis $V$. This spectral alignment allows all three matrices to be diagonalized simultaneously, decomposing the errors of pretraining, \ac{sft}, and \ac{icl} into $d$ independent one-dimensional problems with parameters $(\pi_i,s_i,\tau_i)$. It essentially requires that the pretrained and SFT models share the same feature subspace in their hidden representations. Practically, this means that when the pretrained model is prompted with a question from the SFT data, even if it cannot answer the question correctly, it may still encode the semantic structure of the question through aligned feature directions. This is plausible for modern LLMs, whose pretrained representations often capture rich semantic information. Such alignment is standard in the literature, including the case $V=I_d$~\citep{mallinar2024minimum,li2023transformers}.

Rank deficiency of the pretraining design, $\text{rank}(X)=r<d$, implies that pretraining observes only an $r$-dimensional subspace of $\mathbb{R}^d$. Directions orthogonal to $\col(X)$ are weakly identified, so pretrained representations may fail to cover directions emphasized by the downstream design $\tilX$ or task prior $\Sigma^*$. For instance, an \ac{llm} pretrained with limited cybersecurity content may underperform on cybersecurity tasks~\citep{weerawardhena2025llama}. From a personalization perspective, $d$ is the intrinsic task dimension, while $r$ is the effective dimension covered by pretraining; directions outside $\col(\tilX)$ and $\col(\Sigma^*)$ are irrelevant for the task and cannot be exploited without personalization.

\begin{assumption}[Pretraining inputs normalization]\label{assump:norm}
The pretraining covariates are normalized to have a constant $\Sigma_\theta$-norm, in the sense that $\|x_i\|_{\Sigma_\theta}^2 = x_i^\top \Sigma_\theta x_i = c^2$ for all samples  $i\in[N].$
\end{assumption}
The constant-$\Sigma_\theta$-norm assumption can be enforced by normalizing pretraining inputs, e.g., via length control or feature rescaling, and affects only the scalar factor $\bar\sigma^2=\sigma^2+c^2$ in the pretraining covariance~\eqref{eq:pre_cov}. In practice, this assumption means that different topics are equally represented in the pretraining data. With heterogeneous norms, the formulas would replace $\bar\sigma^2$ by an appropriate average of $\|x_i\|_{\Sigma_\theta}^2$ without altering qualitative comparisons among \ac{sft}, and \ac{icl}. We impose this assumption to keep error expressions interpretable and to emphasize the role of the spectral quantities $(\pi_i,s_i,\tau_i)$. Existing work imposes similar assumptions, e.g., $x_i$ are i.i.d. samples~\citep{lu2025asymptotic,li2023transformers}.
\begin{theorem}[Error Analyses of \ac{icl} and \ac{sft}]\label{thm:err_analysis}
   Under Assumptions~\ref{assump:svd} and~\ref{assump:norm}, when \ac{sft} and \ac{icl} use data $(\tilX,\tilY)$ generated according to~\eqref{eq:task_data}, the mean-squared errors of \ac{sft}, and \ac{icl} in \eqref{eq:err} are given as
    \begin{align*}
        E_{\icl} & = \sum_{i=1}^{r}\bigg[\alpha_{i}^2\cdot \underbrace{(E_{\pre,i}-\bar{\sigma}^2\pi_i^{-1})}_{\text{pretraining err. w/o uncertainty}}+  \underbrace{\frac{\tilsigma^2\bar{\sigma}^4s_i}{(\bar{\sigma}^2s_i+\tilsigma^2\pi_i)^2}}_{\text{data noise}}+\underbrace{\frac{\tilsigma^2\bar{\sigma}^2}{\bar{\sigma}^2s_i+\tilsigma^2\pi_i}}_{\text{\ac{icl} uncertainty}}\bigg] +\underbrace{\sum_{i=r+1}^{d}\tau_i}_{\text{task var.}}\\
        E_{\sft}^{\lambda} &= \sum_{i=1}^{r}\bigg[ (1-\beta_{i}^{2})\underbrace{E_{\pre,i}}_{\text{pretraining err.}} +\beta_{i}^{2} \big(\underbrace{\tilsigma^2s_{i}^{-1}}_{\text{LS var.}}+\!\!\!\!\underbrace{\tilsigma^2s_{i}^{-1}}_{\text{\ac{sft} uncertainty}}\!\!\!\!\!\! \big)\bigg]+\sum_{i=r+1}^{d} \big(\underbrace{\tilsigma^2s_{i}^{-1}}_{\text{LS var.}}+\!\!\!\underbrace{\tilsigma^2s_{i}^{-1}}_{\text{\ac{sft} uncertainty}}\!\!\!\!\!\!\!\big),
    \end{align*}
    where $E_{\pre,i}$ denotes the error of the pretrained \ac{llm} along the $i$-th dimension, detailed in Appendix~\ref{app:error_analysis}. Here $\bar{\sigma}^2 = \sigma^2 + c^2$ denotes the effective noise in pretraining, combining the observation noise variance $\sigma^2$ and the input power $c^2$. The coefficients $\{\alpha_i\}_{i=1}^r$ for \ac{icl} and $\{\beta_i\}_{i=1}^r$ for \ac{sft} are given by\looseness=-1
    \begin{align*}
        \alpha_i = \frac{\tilsigma^2\pi_i}{\bar{\sigma}^2s_i+\tilsigma^2\pi_i}, \qquad \beta_i = \frac{\bar{\sigma}^{2}s_i}{\bar{\sigma}^{2}s_i+\lambda \tilsigma^2\pi_i},
        \quad i\in[r].
    \end{align*}
\end{theorem}

A formal statement is provided in Appendix~\ref{app:error_analysis}. For each algorithm $a\in\{\icl,\sft\}$, the estimator admits the decomposition $\hat{\theta}_a=\theta_a+\omega_a$, where $\omega_a\sim\calN(0,\Sigma_a)$ captures intrinsic model uncertainty. The mean-squared error then decomposes as
\begin{align*}
    \bbE \big[\|\htheta_{a}-\theta^{*}\|^2\big] = \bbE \big[\|\theta_{a}-\theta^{*}\|^2\big]+\tr(\Sigma_{a}) = \text{algo. mean est. err.}+\text{uncertainty},
\end{align*}
separating the algorithm’s mean estimation error from its uncertainty. Different algorithms yield different expressions for these two components.


To build intuition for the \ac{sft} and \ac{icl} error terms, suppose $\tilX$ consists of $\tilN/d$ repetitions of an orthonormal basis of $\bbR^d$, so that $s_i=\tilN/d=\Theta(\tilN)$. As shown in Section~\ref{sec:model_personal}, \ac{icl} performs a Bayesian update under the pretrained prior and therefore inherits the information encoded in that prior. Importantly, the uncertainty of the pretrained model is not part of this prior, since it arises from our probabilistic modeling of the pretrained model itself. Thus, the \ac{icl} error inherits the pretraining error excluding uncertainty, i.e., $E_{\pre,i}-\bar{\sigma}^2\pi_i^{-1}$, which is weighted by $\alpha_i^2$ for $i\le r$. The coefficients $\{\alpha_i\}_{i=1}^{r}$ depend on the pretraining precision $\pi_i$, the effective pretraining noise $\bar{\sigma}^2$, the personalization data strength $s_i$, and the personalization data noise $\tilsigma^2$. As the personalization data strength increases, with $s_i=\Theta(\tilN)$, we have $\alpha_i=O(\tilN^{-1})$ and hence $\alpha_i^2=O(\tilN^{-2})$. Thus, the contribution of the inherited pretraining error decays as more in-context data become available. However, the residual task variation $\sum_{i=r+1}^{d}\tau_i$ in the remaining $d-r$ dimensions cannot be reduced by additional in-context samples, because the pretrained prior does not support these directions. Beyond the inherited pretraining terms, \ac{icl} also incurs a data-noise term and an \ac{icl}-specific uncertainty term. The data-noise term captures the effect of noise in the personalization data during the Bayesian update, while the \ac{icl}-specific uncertainty term captures the posterior uncertainty after conditioning on the personalization data. Both terms depend on the personalization data strength $s_i$ and the noise level $\tilsigma^2$, and both vanish as $s_i\to\infty$ or $\tilsigma^2\to 0$.

Under \eqref{eqn:theta_sft}, \ac{sft} outputs a direction-wise weighted combination of the pretrained prior and the \ac{ls} estimator. Accordingly, the \ac{sft} error $E_{\sft}^\lambda$ in Theorem~\ref{thm:err_analysis} combines inherited pretraining error terms and personalization data-driven error terms, with the trade-off controlled by the coefficients $\{\beta_i\}_{i=1}^{r}$. The \ac{ls} error consists of two parts: the variance of the \ac{ls} estimator and the \ac{sft} uncertainty induced by fine-tuning on the personalization data. Both terms decrease with the personalization data strength $s_i$. The coefficients $\{\beta_i\}_{i=1}^{r}$ determine how much weight is placed on the data-driven estimator relative to the pretrained prior, and they increase with $s_i$. When $s_i=\Theta(\tilN)$ and the personalization sample size $\tilN$ grows, we have $\beta_i\to 1$, so the inherited pretraining error vanishes, and the total error is dominated by the data-driven terms: the \ac{ls} variance and the \ac{sft} uncertainty. The weighting by $\beta_i$ applies only to the $r$ pretrained directions, since the pretrained model only covers these subspaces. In the remaining $d-r$ coordinates, the error consists solely of personalization data-driven terms. Overall, $E_{\sft}^\lambda$ decays at rate $\Theta(\tilN^{-1})$ when $s_i=\Theta(\tilN)$.

Overall, the errors of \ac{icl} and \ac{sft} both contain a component inherited from the pretrained model and a component induced by the personalization data. The key difference lies in how the two methods weight the pretrained information and how personalization data reduces the resulting error. In particular, \ac{icl} can only reduce error within the subspace covered by the pretrained prior, because it updates the model through Bayesian conditioning under that prior. By contrast, \ac{sft} can reduce error beyond the pretrained subspace, provided that the personalization data contain sufficient signal in those directions.

\subsection{Comparisons of Personalization Algorithms}\label{sec:comp}

Before comparing personalization methods, we first explain why personalization algorithms are preferable to directly using pretrained models. To this end, we compare \ac{icl} with the pretrained \acp{llm}. Since \ac{icl} updates the pretrained prior using personalization data, pretraining corresponds to the zero-sample case of \ac{icl}, implying $E_{\pre}=E_{\icl}$ when $s_i=0$ for all $i\in[r]$. With additional samples, \ac{icl} weakly dominates pretraining in estimation error.\looseness=-1

\begin{corollary}[\ac{icl} Weakly Dominates Pretraining]\label{coro:pre_icl}
    If Assumptions~\ref{assump:svd} and \ref{assump:norm} hold, then $E_{\icl} \leq E_{\pre}$. Moreover, this inequality is strict in the non-degenerate case, i.e., when $\sigma + c > 0$ and there exists $i \in [r]$ such that $s_i > 0$.
\end{corollary}
Intuitively, the degenerate case arises when the pretrained \acp{llm} either provides no useful information or already fully determines the task, both of which are unlikely in practice. Consequently, this result rationalizes the use of personalization algorithms.

We next compare the errors of \ac{sft} and \ac{icl}. To capture the main intuition and simplify the analysis, we impose the following isotropicity assumption.

\begin{assumption}[Isotropicity]\label{assump:homo}
Across all covered directions, the energies of the pretraining samples and the samples of the desired task are isotropic: there exist scalars $\tau, s, m > 0$ such that $\tau_i = \tau$, $s_i = s$, and $\lvert v_i^{\top}\mu_\theta\rvert = m$ for every $i \in [d]$, and there exists a scalar $\pi > 0$ such that $\pi_i = \pi$ for every $i \in [r]$.
\end{assumption}
The isotropicity assumption is a structural simplification, not a claim of truly isotropic pretrained features. Although ${\tau_i}$, ${s_i}$, ${|v_i^\top\mu_\theta|}$, and ${\pi_i}$ are generally heterogeneous, our analysis focuses on the relative performance of \ac{sft} and \ac{icl}. The key qualitative results—comparative error performance, optimal \ac{sft} regularization, and signal propagation within the covered subspace—depend on aggregate signal-to-noise properties rather than fine-grained anisotropy. This assumption is standard in analyses of \acp{llm}~\citep{lu2025asymptotic,su2025isotropic}.

\begin{proposition}[Comparison between \ac{icl} and \ac{sft}]\label{prop:comp}
Under Assumptions~\ref{assump:svd}, \ref{assump:norm}, and \ref{assump:homo}, we define the coverage coefficient as $R=r/(d-r)$, and the signal-to-noise ratio in the subspace not covered by the pretrained model as $\kappa=s\tau/(2\tilsigma^2)$. Then, \ac{sft} with the optimal regularization parameter $\lambda^*$ achieves weakly lower error than \ac{icl}, i.e., $E_{\sft}^{\lambda^*} \leq E_{\icl}$, if and only if the following condition holds.
\begin{align}\label{ieq:R_cond}
    R\geq \frac{(1-\kappa)\big(2\bar\sigma^2\kappa+\pi\tau\big)^2\big(2\bar\sigma^2\kappa+\pi\kappa(\tau+m^2)+\pi\tau\big)}
{\pi^{3}\,\kappa^2\,\tau\,(\tau+m^2)^{2}}.
\end{align}
We define the right-hand side of \eqref{ieq:R_cond} as $R_{\crit}$.
\end{proposition}
The proposition characterizes when \ac{sft} or \ac{icl} attains lower error, governed by two quantities: pretrained coverage $R=r/(d-r)$ and the ``null-space'' \ac{snr} $\kappa=s\tau/(2\tilsigma^2)$. Within the pretraining-covered subspace, \ac{sft} always dominates \ac{icl}, since its tunable shrinkage can match or improve upon \ac{icl}’s Bayesian update (shown in our proof). Thus, \ac{icl} can outperform \ac{sft} only through its conservatism in the null space: \ac{icl} ignores uncovered directions and incurs a fixed cost $\tau$, whereas \ac{sft} attempts to learn there and pays $2\tilsigma^2/s$, which is detrimental when $\kappa<1$. The threshold $R_{\crit}$ in \eqref{ieq:R_cond} quantifies the coverage required for \ac{sft}$’$s row-space advantage to offset this penalty: if $\kappa\ge1$, \ac{sft} strictly dominates for any $R>0$; if $0<\kappa<1$, \ac{sft} is optimal iff $R\ge R_{\crit}$. 

\begin{tcolorbox}[enhanced jigsaw,breakable,frame hidden, left = 0.1mm, right = 0.1mm, top = 0.1mm, bottom = 1.5mm]
\refstepcounter{finding}\label{find:icl}
\textbf{Finding \thefinding:}
When the personalization task is poorly covered by the pretrained prior and the personalization data are noisy or limited, \ac{icl} can achieve lower error than \ac{sft} under the linear model. In this regime, \ac{icl} is safer because it updates only within the subspace supported by the pretrained prior, while \ac{sft} may incur large variance when learning poorly supported directions. Conversely, when the personalization data are sufficiently informative, \ac{sft} can achieve lower error by learning beyond the pretrained subspace.

\end{tcolorbox}


\subsection{Comparison between Realistic and Linearly Approximated \acp{llm}}

\begin{table}[t]
\centering
\begin{tabular}{|p{2.2cm}|p{6.2cm}|p{6.2cm}|}
\hline
 & Realistic \ac{llm} & Linear approximation of \ac{llm} \\ \hline
Pretraining 
& Pretraining data cover diverse topics, allowing the model to learn a broad prior over them. 
& Data diversity: \eqref{eq:pretrain_data};  A broad prior: \eqref{eq:pre_cov}.
\\ \hline
\ac{icl} 
& \ac{icl} often outperforms pretraining alone, but it is difficult for \ac{icl} to exploit knowledge that is absent from pretraining. 
& Improved performance: Corollary~\ref{coro:pre_icl}; Constrained improvement: $\tau_i$ of $E_{\icl}$ in Theorem~\ref{thm:err_analysis}. 
\\ \hline
\ac{sft} 
& \ac{sft} may suffer from catastrophic forgetting, but it can also inject new knowledge into the pretrained model. 
& Catastrophic forgetting: when $s_i$ is small, $E_{\sft}^{\lambda}$ in Theorem~\ref{thm:err_analysis} is large; Knowledge injection: when $s_i$ is large, $E_{\sft}^{\lambda}$ tends to $0$. 
\\ \hline
\end{tabular}
\caption{Comparison between realistic \acp{llm} and their linear approximation under pretraining, \ac{icl}, and \ac{sft}.}
\label{tab:compare}
\end{table}
In the following, we explain how our linear approximation captures the behavior of realistic \acp{llm} in personalization tasks. Consider a user who wants to adapt a pretrained \ac{llm} to a target task using personalization data. Whether \ac{icl} or \ac{sft} is more effective depends on two factors: how well the target task is already covered by the pretrained model, and how informative the personalization data are. Table~\ref{tab:compare} compares this behavior in realistic \acp{llm} and in our linear approximation. During pretraining, realistic \acp{llm} learn from diverse corpora and acquire a broad prior over tasks~\citep{touvron2023llama,grattafiori2024llama}. In the linear model, this is captured by the heterogeneous task parameters in the pretraining data distribution \eqref{eq:pretrain_data} and by the learned prior covariance in \eqref{eq:pre_cov}. Given a personalization task, \ac{icl} uses demonstrations to identify and exploit the task-relevant knowledge already present in this prior, without changing the model parameters. This explains why \ac{icl} often improves over pretraining alone, as formalized by Corollary~\ref{coro:pre_icl}. However, because \ac{icl} updates only within the subspace supported by the pretrained prior, its improvement is limited when the target task requires knowledge poorly covered by pretraining; this limitation appears in Theorem~\ref{thm:err_analysis} through the residual task-variation terms in $E_{\icl}$. In contrast, \ac{sft} updates the model parameters using personalization data and can therefore learn beyond the pretrained subspace. This makes \ac{sft} more powerful when the personalization data are sufficiently informative. At the same time, when the personalization data are noisy or limited, \ac{sft} may fit poorly supported directions and incur large estimation error, reflecting the practical trade-off between retaining pretrained knowledge and injecting new task-specific knowledge~\citep{luo2025empirical}. Theorem~\ref{thm:err_analysis} captures this trade-off through $E_{\sft}^{\lambda}$: small personalization data strength leads to large data-driven error, whereas sufficiently large personalization data strength allows the \ac{sft} error to decrease and eventually learn the target task effectively.


\section{LLM-Serving Congestion Game of Users}\label{sec:mfcg}
Section~\ref{sec:stat_analysis} characterizes which personalization method is statistically optimal. However, users do not choose in isolation; they share a platform with limited computational resources. When many users personalize simultaneously, congestion arises, leading to higher latency as aggregate demand increases. This section asks: what equilibrium emerges when users strategically choose personalization methods and sample sizes while facing congestion?

We establish three main results. First, an equilibrium exists and the equilibrium congestion level is unique (Section~\ref{sec:exist}). Second, the pretrained \acp{llm} and target tasks exert nuanced, non-monotone effects on equilibrium congestion (Section~\ref{sec:homo_mfcg}). Finally, the pretrained \acp{llm} can cause the equilibrium to anchor on a particular user type (Section~\ref{sec:hetero_mfcg}).

\subsection{Existence and Congestion-Level Uniqueness of Equilibrium}\label{sec:exist}
We first characterize user types and derive the corresponding expressions for $E_a(t,\tilN)$ in~\eqref{eq:err}. Under Assumptions~\ref{assump:svd}, \ref{assump:norm}, and \ref{assump:homo}, the effective eigenvalue $s$ of $\tilX^\top\tilX$ scales with the number of personalization samples $\tilN$. We model this as $s=\tilN^{\alpha}$ for some $0<\alpha\le1$, capturing potentially sublinear information growth due to data redundancy. When the rows of $\tilX$ are i.i.d.\ standard Gaussian, $s=\tilN$ holds with high probability (see Appendix~\ref{app:random_cov}).

All remaining parameters in Assumptions~\ref{assump:svd}, \ref{assump:norm}, and \ref{assump:homo}—$d,r,\bar{\sigma},\tilsigma,m,\pi,$ and $\tau$—are intrinsic to the personalization problem. We therefore define the user type as $t=(d,r,\bar{\sigma},\tilsigma,m,\pi,\tau)$, with dependence on the sample size captured by $s=\tilN^{\alpha}$. The type space $\mathcal T$ is any compact, feasible subset of these parameters, e.g., $(d,r)\in\{(x,y)\given 500\le x\le1000,;500\le y\le x\}$. Under these assumptions, Theorem~\ref{thm:err_analysis} and Proposition~\ref{prop:comp} characterize the \ac{icl} and optimal \ac{sft} errors as
\begin{align}
    E_{\icl}(t,\tilN)
    &\!=\! \frac{r\tilsigma^2\big(2\bar{\sigma}^4 \tilN^{\alpha}+\tilsigma^2\pi^2\zeta\big)}
    {(\bar{\sigma}^2 \tilN^{\alpha}+\tilsigma^2\pi)^2}
    +(d-r)\tau,\, E_{\sft}(t,\tilN)
    \!=\! \frac{2r\tilsigma^2\zeta}{\tilN^{\alpha}\zeta+2\tilsigma^2}
    +2(d-r)\frac{\tilsigma^2}{\tilN^{\alpha}},\label{eq:sft_N}
\end{align}
where $\zeta = \tau + 2\bar{\sigma}^2\pi^{-1} + m^2$. Here we set the \ac{sft} regularization parameter to its optimal value $\lambda^{\ast}$, and omit the superscript $\lambda^{\ast}$ for notational simplicity. With these specifications, the cost function in~\eqref{eq:cost} admits explicit expressions. We then analyze the resulting equilibrium properties.



\begin{theorem}[Existence and Congestion-Level Uniqueness of Equilibrium]\label{thm:exist_unique}
    Under regularity conditions of $h$, an equilibrium $(\pi^{\ast},R^{\ast})$ exists. Moreover, the equilibrium congestion level $R^{\ast}$ is unique. In other words: if $(\pi_1^{\ast},R_{1}^{\ast})$ and $(\pi_2^{\ast},R_{2}^{\ast})$ are two equilibria, then $R_{1}^{\ast}=R_{2}^{\ast}$.
\end{theorem}
Economically, the existence result rules out pathological cases in which demand fails to clear or conjectured congestion levels cannot be supported by any population behavior. The proof also yields an explicit upper bound on equilibrium congestion $R^{\ast}$, ensuring that congestion remains finite and that the system never requires unbounded computational resources, consistent with real-world \ac{llm} service constraints.

Unlike much of the mean-field game literature~\citep{lasry2007mean,guo2019learning,zhang2023learning}, which focuses on uniqueness of the equilibrium policy $\pi^{\ast}$, Theorem~\ref{thm:exist_unique} establishes uniqueness of the induced congestion level $R^{\ast}$. Thus, even when multiple equilibrium policies exist, the model delivers a unique prediction for congestion. This is crucial because $R^{\ast}$ determines latencies and effective per-unit compute costs $p+h(R)$ and is the key state variable for platform demand and revenue forecasting (Section~\ref{sec:smfcg}). 

The uniqueness proof relies on a no-crossing property (Lemma~\ref{lem:strong_mono}): higher congestion raises the per-unit cost $p+h(R)$ and shifts best-response resource demands downward. As a result, aggregate demand is decreasing in conjectured congestion, yielding a unique fixed point for $R^{\ast}$ (Theorem~\ref{thm:exist_unique}). This mirrors the classical monotonicity condition in mean-field games~\citep{lasry2007mean,graber2023monotonicity}, but at the congestion level: monotonicity of best responses in congestion ensures uniqueness of the equilibrium congestion.

Having established existence and uniqueness of the equilibrium congestion level, we next study two special cases—homogeneous users and two user types—to examine how \ac{llm} characteristics and user heterogeneity shape equilibrium outcomes.

\subsection{Analysis of LSCG with Homogeneous Users}\label{sec:homo_mfcg}


\paragraph{Characterization of the Equilibrium.} To build intuitions, we first study the \ac{lscg} with \emph{homogeneous} users, where all users share the same type and the type index $t$ is omitted from the notation. To highlight the core insights, we simplify the \ac{icl} and \ac{sft} error expressions in~\eqref{eq:sft_N} while preserving their asymptotic behavior as $\tilN\to0$ and $\tilN\to\infty$. When $\tilN$ is large, as is typical in practice, we have\looseness=-1
\begin{align*}
    E_{\icl}(\tilN) = (d-r)\tau +\Theta(\tilsigma^2r \tilN^{-\alpha}),\quad E_{\sft}(\tilN) = \Theta(\tilsigma^2d \tilN^{-\alpha}).
\end{align*}
In contrast, as $\tilN\downarrow 0$, we have that $E_{\icl}(0) = (d-r)\tau + r\zeta$, and $ \lim_{\tilN\downarrow 0}E_{\sft}(\tilN)=+\infty$. To simplify expressions while preserving these asymptotic behaviors, we approximate error functions \eqref{eq:sft_N} by
\begin{align}
    E_{\icl}(\tilN) =\frac{2\tilsigma^2r}{\tilN^\alpha + 2\tilsigma^2 \zeta^{-1}}+(d-r)\tau,\quad 
    E_{\sft}(\tilN) =\frac{2\tilsigma^2 d}{\tilN^\alpha}.\label{eq:simplified}
\end{align}
We emphasize that this simplification preserves the main insights underlying the comparison between \ac{icl} and \ac{sft} in Proposition~\ref{prop:comp}.
In particular, when the signal-to-noise ratio of the personalization data is large, i.e.,
$\kappa = \tau \tilN^\alpha/(2\tilsigma^2) > d/(d-r)$,
\ac{sft} achieves a lower error than \ac{icl}.

In the remainder of this section, we set $\alpha=1$, corresponding to the case of i.i.d.\ Gaussian covariates, as discussed in Appendix~\ref{app:random_cov}. For analytical convenience, we further adopt the quadratic congestion function $h(R)=R^2$, so that $\bar h(R,p)=p+R^2$.
This quadratic form is adopted for the ease of calculation. However, our results in this section hold for a wide range of convex and non-negative congestion functions $h$. We empirically verify that all our results hold for $h(x)=\exp(x)$ and $h(x)=\max\{0,x\}$ in Appendix~\ref{app:add_exp}. The following assumption is imposed to reflect the relative computational requirements of \ac{sft} and \ac{icl}.
\begin{assumption}[Relative Resource Consumption of \ac{sft} and \ac{icl}]\label{assump:large}
    The consumed resource of \ac{sft} per sample is larger than that of \ac{icl}, i.e., $R_{\sft}> R_{\icl}$.
\end{assumption}
Discussed in Section~\ref{sec:model}, \ac{sft} updates the \ac{llm}’s parameters using personalization data, requiring data prefill, gradient backpropagation, and updates to both parameters and optimizer states. In contrast, \ac{icl} involves only data prefill. Thus, \ac{sft} is more computationally demanding than \ac{icl}.

We analyze equilibrium behavior under homogeneous users by first characterizing outcomes when users are restricted to a single algorithm, and then comparing \ac{sft} and \ac{icl} to study algorithm switching and derive the full equilibrium. We begin by characterizing user behavior under a fixed algorithm.
For any $a\in\mathcal{A}=\{\sft,\icl\}$ and per-unit resource cost $H$, which aggregates monetary and time costs, we define the minimal cost and corresponding optimal sample size as
\begin{align*}
\Phi_a(H)=\min_{N\geq 0}\big\{E_a(N)+R_aNH\big\}, \quad N_a(H) = \argmin_{N\geq 0}\big\{E_a(N)+R_aNH\big\}.
\end{align*}
For both algorithms $a\in\{\icl,\sft\}$, the statistical error $E_a(N)$ in~\eqref{eq:simplified} is strictly convex in $N$, while the resource cost $R_a N H$ is linear in $N$ for fixed $H$; hence the optimal sample size $N_a(H)$ is uniquely defined. If users are restricted to a single algorithm $a\in\mathcal A$, they optimally choose $N_a(H)$ at congestion level $R = \sqrt{H-p}$. We characterize behavior under this restriction by the within-algorithm fixed point:
\[
H_a^{\ast}=\bar h\!\left(R_a N_a(H_a^{\ast}),p\right)=(R_a N_a(H_a^{\ast}))^2+p,
\]
where $R_a^{\ast}=\sqrt{H_a^{\ast}-p}$ is the induced equilibrium congestion. Existence and uniqueness of $H_a^{\ast}$ are established later. To connect these within-algorithm equilibria to the full game, we compare minimal costs across algorithms. Let $\psi(H)=\Phi_{\sft}(H)-\Phi_{\icl}(H)$, and define $H_{\sep}^{\ast}$ as the solution to $\psi(H)=0$, which separates regions where users prefer \ac{sft} or \ac{icl}.


\begin{theorem}[Equilibrium with Homogeneous Users]\label{thm:homo}
    The game specified by \eqref{eq:simplified}, $\alpha=1$, $h(x)=x^2$, and $R_{\sft}>R_{\icl}$ has the following properties.
    \begin{itemize}[leftmargin=1em]
        \item[1.] The optimal \ac{sft} sample size requires more resources than \ac{icl}: \mbox{$\!\!R_{\sft}N_{\sft}(H)\!\!>\!R_{\icl}N_{\icl}(H),\!\forall H\!\!\geq\!0$.}
        \item[2.] The fixed points $H_{\sft}^{\ast}$ and $H_{\icl}^{\ast}$ exists and are unique. In addition, $H_{\sft}^{\ast}>H_{\icl}^{\ast}\geq p$.
        \item[3.] The function $\psi(H)$ is strictly increasing in $H$, and has unique root $H_{\sep}^{\ast}>0$.
        \item[4.] The equilibrium $(\pi^{\ast},R^{\ast})$ is specified for the unique user type $ $ with $T(t)=1$ as follows.
        \begin{align*}
            &H^{\ast} = \max\big\{H_{\icl}^{\ast},\min\{H_{\sep}^{\ast},H_{\sft}^{\ast}\}\big\}, \quad R^{\ast} = \sqrt{H^{\ast}-p},\\
            &\pi_{t}^{\ast}\Big(\big(\sft,N_{\sft}(H^{\ast})\big)\Big) =\frac{\sqrt{H^{\ast}-p}-R_{\icl}N_{\icl}(H^{\ast})}{ R_{\sft}N_{\sft}(H^{\ast})-R_{\icl}N_{\icl}(H^{\ast})},\\
            &\pi_{t}^{\ast}\Big(\big(\icl,N_{\icl}(H^{\ast})\big)\Big)= 1-\pi^{\ast}\Big(\big(\sft,N_{\sft}(H^{\ast})\big)\Big).
        \end{align*}
    \end{itemize}
\end{theorem}
\begin{wrapfigure}{r}{\myfigwidth}
    \centering
    \includegraphics[width=\myfigwidth]{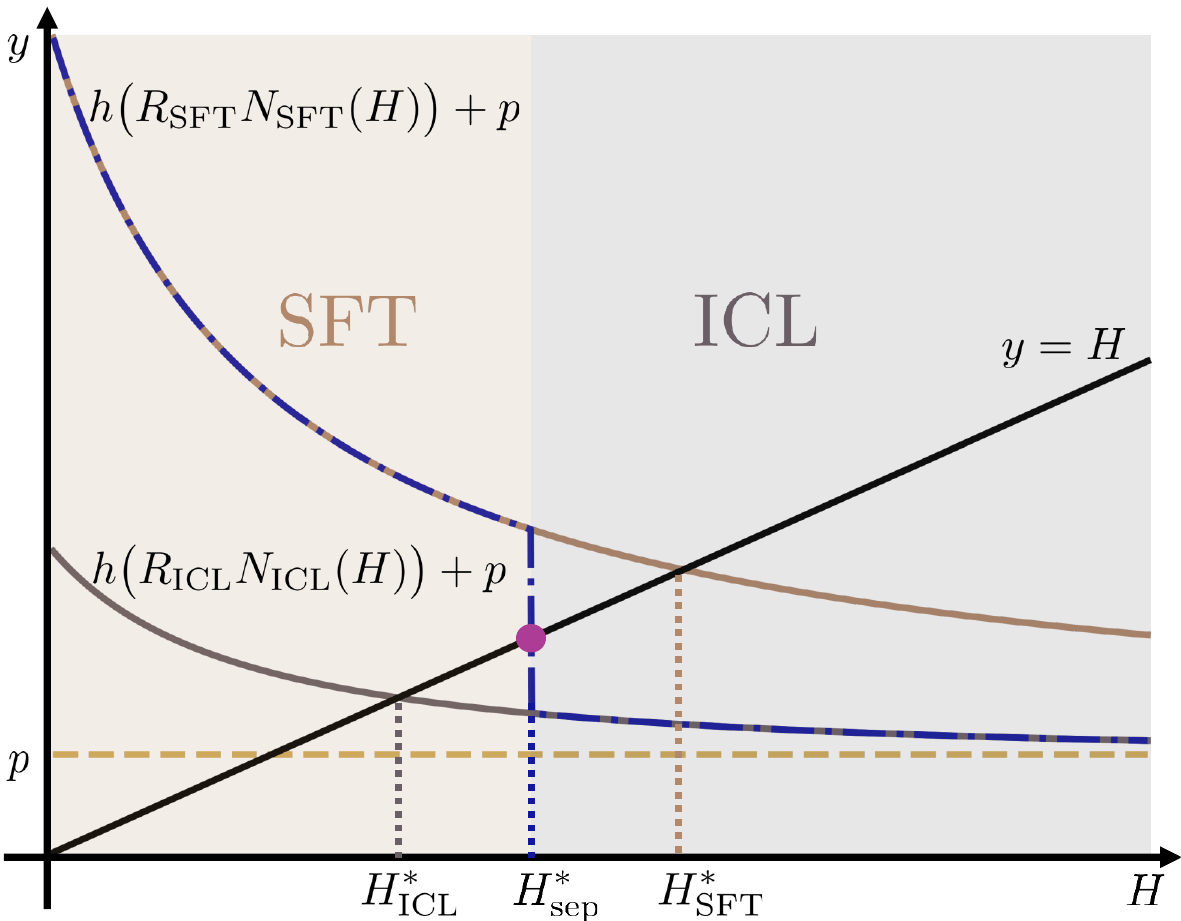}
    \caption{This figure illustrates the structure of the equilibrium with homogeneous users. The intersection point of the line $y=H$ and the blue dotted curve corresponds to the equilibrium of the game.}
    \label{fig:homo}
\end{wrapfigure}
Theorem~\ref{thm:homo} characterizes the equilibrium with homogeneous users by reducing the equilibrium problem to a one-dimensional fixed point in the effective per-unit compute cost $H=p+h(R)$, illustrated in Figure~\ref{fig:homo}. 

Parts (1)–(3) establish the structure underlying equilibrium selection. First, \ac{sft} induces strictly higher resource demand than \ac{icl} at any given $H$, as implied by Assumption~\ref{assump:large} and illustrated in Figure~\ref{fig:homo}, where $h(R_{\sft}N_{\sft}(H))>h(R_{\icl}N_{\icl}(H))$ for all $H$. Second, for each algorithm $a\in{\sft,\icl}$, the within-algorithm fixed point $H_a^{\ast}$ exists and is unique. As a consequence of part (1), we have $H_{\sft}^{\ast}>H_{\icl}^{\ast}$. Third, a unique threshold $H_{\sep}^{\ast}$ governs algorithm choice: users prefer \ac{sft} when $H<H_{\sep}^{\ast}$ and \ac{icl} otherwise, echoing Proposition~\ref{prop:comp}. When $H$ is small, Lemma~\ref{lem:strong_mono} implies higher optimal data usage, raising the signal-to-noise ratio $\kappa$ and eventually favoring \ac{sft} over \ac{icl}.

Part (4) shows that the equilibrium $H^{\ast}$ is obtained by clipping $H_{\sep}^{\ast}$ between the two within-algorithm fixed points, yielding a closed-form characterization of $R^{\ast}=\sqrt{H^{\ast}-p}$ and the induced mixing between \ac{sft} and \ac{icl}. Illustrated in Figure~\ref{fig:homo}, the equilibrium corresponds to the intersection of $y=H$ with the induced curve $p+h(R)$. Economically, the result highlights how congestion endogenously disciplines high-compute personalization:
when congestion is low (small $H$), users tilt toward \ac{sft};
when congestion is high (large $H$), users tilt toward \ac{icl};
and in the intermediate regime, the equilibrium features a mixture that exactly clears the congestion externality.

\begin{figure}[t]
\centering
\subfigure[The values of $R^{\ast}$ with various $\pi$.]{\includegraphics[width=\myfigwidth]{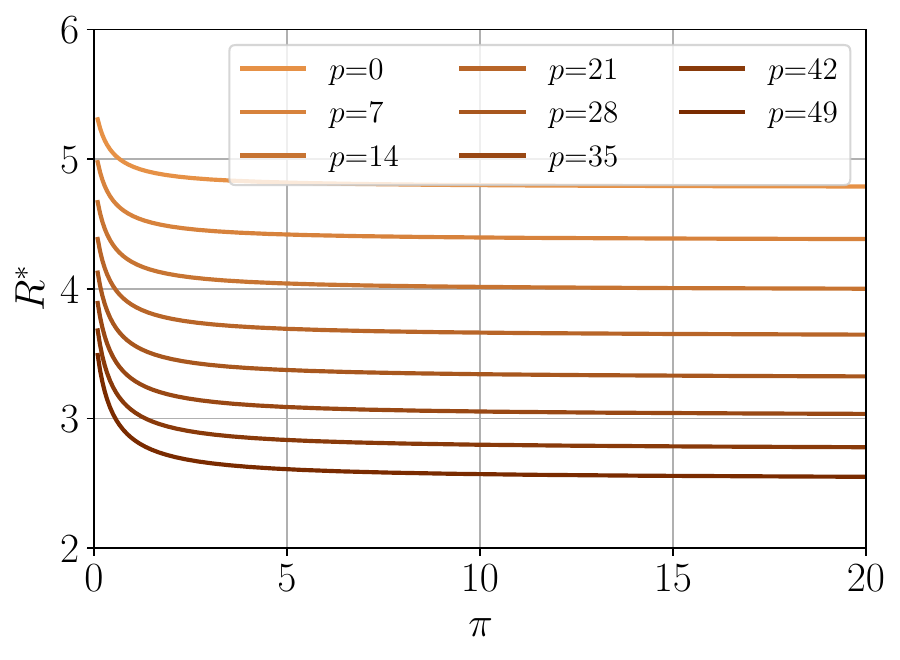}\label{fig:homo_pi}}
\hspace{0.3em}
\subfigure[The values of $R^{\ast}$ with various $r$.]{\includegraphics[width=\myfigwidth]{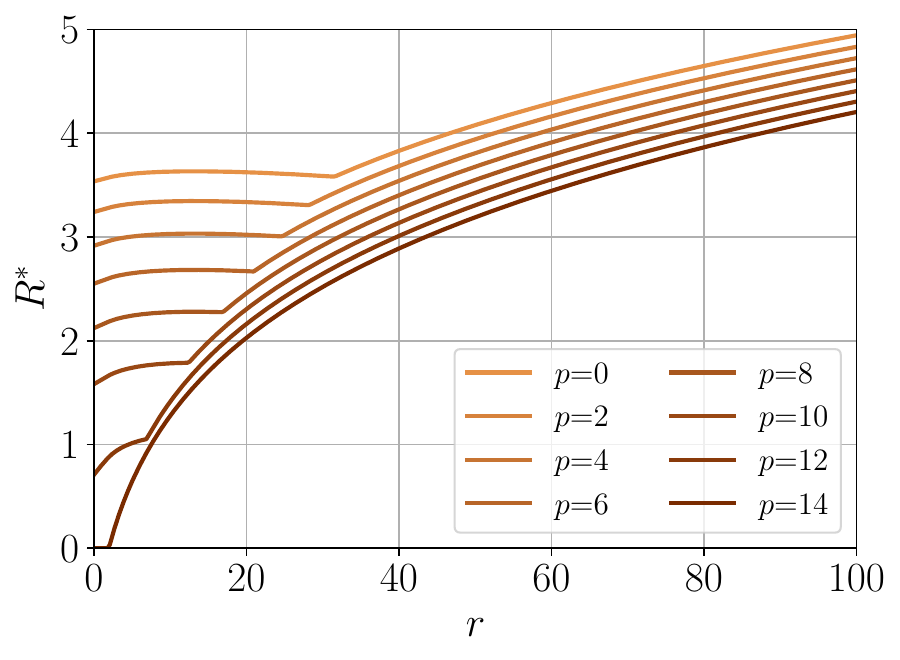}\label{fig:homo_r}}
\caption{The figures show how the equilibrium congestion level $R^{\ast}$ varies with the prior precision $\pi$ and the pretraining coverage $r$ of the pretrained \acp{llm}. Higher prior precision monotonically reduces congestion, whereas broader pretraining coverage can have a non-monotone effect on $R^{\ast}$.}
\label{fig:homo_1}
\end{figure}

\paragraph{Comparative Statics of Equilibrium Congestion Level.} In what follows, we study how the equilibrium varies with problem parameters. To illustrate our results, we run simulations for the \ac{lscg} defined in~\eqref{eq:simplified} with congestion function $h(R)=R^2$ under various parameter configurations.
Details of the experimental setup are provided in Appendix~\ref{app:exp_details}. Throughout, a function $f(x)$ is said to be increasing in $x$ if $f(x_2)\ge f(x_1)$ for all $x_2>x_1$.
We first focus on parameters related to the pretrained model, namely $\pi$ and $r$.

\begin{proposition}[Influence of Pretrained \ac{llm}s on Equilibrium]\label{prop:pretrain_influence}
    Under models \eqref{eq:simplified} with $\alpha=1$, $\barh(R,p)=p+R^2$, and $R_{\sft}>R_{\icl}$, the following holds.
    \begin{itemize}
        \item[1.] The congestion-level $R^{\ast}$ decreasing in $\pi$.
        \item[2.] Assume that $2\tau>\zeta$. If $p$ and $R_{\icl}$ are sufficiently small, then there exist thresholds $r^{*}<r^{**}<d$ such that the equilibrium congestion $R^{*}$ is increasing in $r$ on $[0,r^{*}] \cup [r^{**},d]$ and decreasing in $r$ on $[r^{*},r^{**}]$ as $r$ ranges from $0$ to $d$. Otherwise, $R^{*}$ is increasing in $r\in[0,d]$.

    \end{itemize}
\end{proposition}
A full statement is in Appendix~\ref{app:pretrain_influence}. These results characterize how improvements in pretraining affect the equilibrium congestion level $R^{\ast}$ in Figure~\ref{fig:homo_1}, i.e., the resources amount used at the equilibrium. The comparative statics operate via two channels: prior precision $\pi$ and coverage $r$.

First, increasing $\pi$ unambiguously reduces the equilibrium congestion level $R^{\ast}$. Since $\pi$ measures the precision of the pretrained prior, a higher $\pi$ lowers the marginal value of additional personalization samples, reducing users’ best-response resource demands and shifting the congestion fixed point downward. Economically, this reflects a substitution from online serving compute toward offline pretraining: improving pretrained model quality alleviates congestion at inference time.

The effect of expanding coverage $r$ is more nuanced, as it affects both the scope and effectiveness 
of personalization. When $p$ and $R_{\icl}$ are sufficiently small, users' personalization choices are highly elastic and the equilibrium responds on both an intensive margin (optimal personalization intensity conditional on the method) and an extensive margin (switching between \ac{icl} and \ac{sft}). For small $r$, expanding $r$ also enlarges the \emph{biased} portion of the covered representation---features learned in pretraining but misaligned with the target task---so the gains from correcting these biases via \ac{sft} rise; \ac{icl} remains at a corner with zero optimal samples, and aggregate compute demand (hence $R^{\ast}$) increases. Over an intermediate range $[r^{\ast},r^{**}]$, improved coverage reshuffles relative returns across methods so that substitution on the extensive margin dominates, reducing total resource demand and lowering congestion. For large $r$, the intensive margin reasserts itself: additional coverage makes marginal \ac{icl} samples highly productive on a broad covered subspace, users scale up sample sizes, and $R^{\ast}$ rises again, yielding the non-monotone ``increase--decrease--increase'' relationship.

In contrast, when $p$ or $R_{\icl}$ is large enough, higher effective costs compress optimal sample intensities. The equilibrium is then governed primarily by the intensive-margin adjustments within \ac{sft} and \ac{icl}, and expanding coverage raises the marginal returns to personalization on covered dimensions without generating sizable compositional shifts across methods. As a result, under $2\tau>\zeta$, equilibrium congestion $R^{\ast}$ is increasing in $r$ over $[0,d]$: broader coverage translates into higher aggregate compute demand rather than being offset by substitution effects.

Then we analyze the task-related parameter $\tilsigma$.
\begin{proposition}[Influence of Target Tasks on Equilibrium]\label{prop:task_sigma}
    There exist two numbers $\tilsigma_1\leq \tilsigma_2$, depending on all the other parameters, i.e., $R_{\sft}, R_{\icl}, d, r, \bar{\sigma}, m, \pi, \tau, s$, such that the following holds.\looseness=-1
    \begin{itemize}
        \item[1.] When $\tilsigma\in[0,\tilsigma_1]$, $R^{\ast}$ is increasing in $\tilsigma$.
        \item[2.] When $\tilsigma\in[\tilsigma_2,\infty)$, $R^{\ast}$ is decreasing in $\tilsigma$.
    \end{itemize}
\end{proposition}
The proposition shows that the equilibrium congestion $R^{\ast}$ can be non-monotone in task noise $\tilsigma$ (Figure~\ref{fig:tilsigma}). Since $\tilsigma$ governs the productivity of personalization, the result identifies two thresholds $\tilsigma_1\le\tilsigma_2$ at which the direction of the congestion response reverses.

When $\tilsigma$ is small, personalization remains effective: increasing noise raises baseline error without significantly reducing sample informativeness, inducing users to increase personalization intensity and raising equilibrium congestion. Thus, $R^{\ast}$ increases with $\tilsigma$ on $[0,\tilsigma_1]$. When $\tilsigma$ is large, noise sharply reduces the marginal returns to personalization, leading users to scale back compute usage and lowering aggregate demand; accordingly, $R^{\ast}$ decreases with $\tilsigma$ on $[\tilsigma_2,\infty)$.


\begin{figure}[t]
\centering
\subfigure[The values of $R^{\ast}$ with various $\tilde{\sigma}$.]{\includegraphics[width=\myfigwidth]{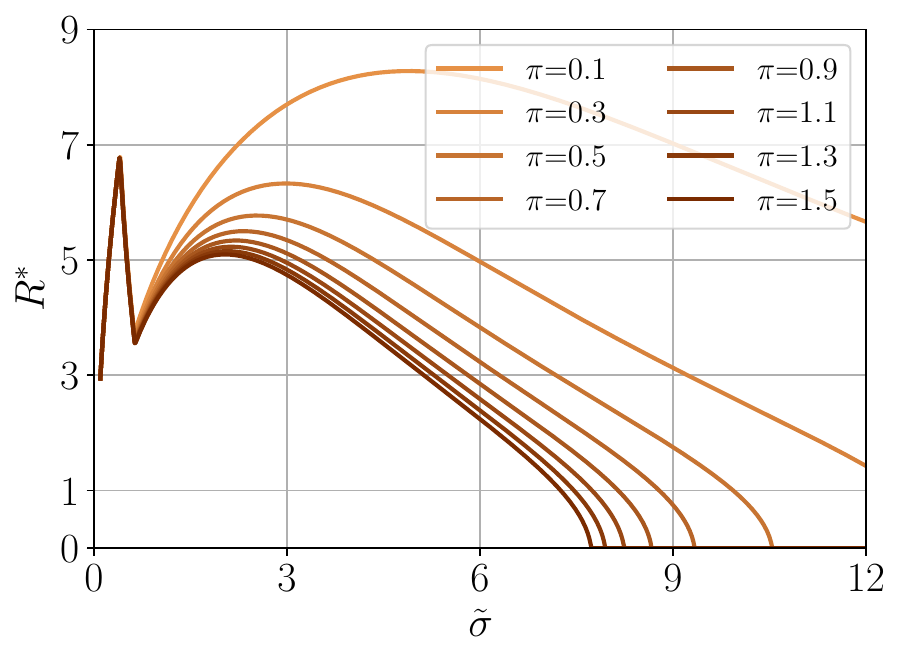}\label{fig:tilsigma}}
\hspace{0.3em}
\subfigure[The values of $R^{\ast}$ with various $R_\sft$.]{\includegraphics[width=\myfigwidth]{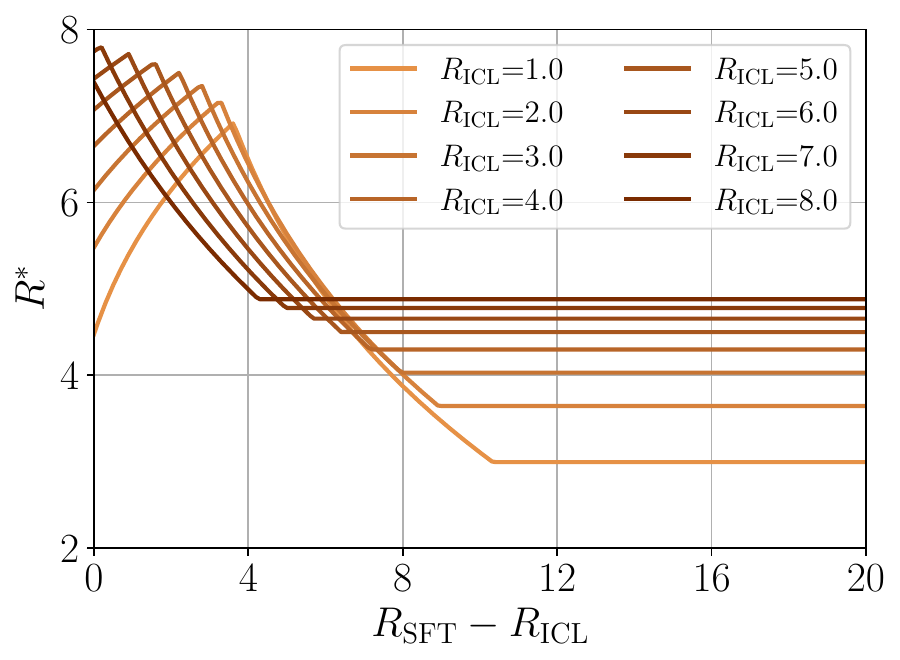}\label{fig:r_sft}}
\caption{The figures show how the equilibrium congestion level $R^{\ast}$ varies with the personalization data noise $\tilsigma$ and the per-sample resource requirements of \ac{sft} and \ac{icl}, $R_{\sft}$ and $R_{\icl}$. The equilibrium congestion level $R^{\ast}$ is non-monotone in the data noise $\tilsigma$ and can also be non-monotone in the \ac{sft} resource requirement $R_{\sft}$.}
\label{fig:homo_2}
\end{figure}

We then analyze how $R_\sft$ and $R_\icl$ affect the equilibrium. 
\begin{proposition}[Influence of Resource Consumption on Equilibrium]\label{prop:r_sft}
    There exists a threshold $R_{\icl}^{\ast}$, depending onmodel parameters $\tilsigma, d, r, \bar{\sigma}, m, \pi, \tau,$ and $s$, such that the following holds. 
    \begin{itemize}
        \item[1.] If $R_{\icl} < R_{\icl}^{\ast}$, then as $R_{\sft}$ increases from $R_{\icl}$ to $\infty$,
    the equilibrium congestion level $R^{\ast}$ in the equilibrium first increases and then decreases.
        \item[2.] If $R_{\icl} \ge R_{\icl}^{\ast}$, then as $R_{\sft}$ increases from $R_{\icl}$ to $\infty$, the equilibrium congestion level $R^{\ast}$ in the equilibrium decreases monotonically.
    \end{itemize}
\end{proposition}
This proposition characterizes how equilibrium congestion responds to the per-sample resource intensity $R_{\sft}$, shown in Figure~\ref{fig:r_sft}.
An increase in $R_{\sft}$ raises the marginal cost of \ac{sft}, affecting both users’ algorithm choice and optimal sample sizes, with the resulting effect on $R^{\ast}$ depending on the relative compute efficiency of \ac{icl}, as captured by $R_{\icl}$.

If $R_{\icl}<R_{\icl}^{\ast}$, then \ac{icl} is sufficiently compute-efficient that the equilibrium features an economically meaningful substitution margin. As $R_{\sft}$ rises from $R_{\icl}$, equilibrium congestion can initially increase because users’ best responses adjust along the intensity margin while \ac{sft} remains competitive for a subset of users; aggregate demand rises at the fixed point. When $R_{\sft}$ becomes large, \ac{sft} is priced out and users substitute toward \ac{icl}, reducing aggregate resource usage and driving $R^{\ast}$ down. If $R_{\icl}\ge R_{\icl}^{\ast}$, since $R_{\sft}$ increases from $R_{\icl}$, its initial value is high enough . Raising $R_{\sft}$ mainly deters \ac{sft} without inducing compensating increases in equilibrium personalization intensity, so aggregate demand shifts downward and $R^{\ast}$ decreases monotonically with $R_{\sft}$.

\begin{tcolorbox}[enhanced jigsaw,breakable,frame hidden, left = 0.1mm, right = 0.1mm, top = 0.1mm, bottom = 1.5mm]
\refstepcounter{finding}\label{find:homo_game}
\textbf{Finding \thefinding:}
The equilibrium congestion level $R^*$ responds non-monotonically to changes in pretraining coverage, task difficulty, and the relative cost of \ac{sft}, because users adjust both the personalization method they choose (extensive margin) and the amount of personalization data they use (intensive margin). For example, when personalization data become more valuable, users may use fewer samples to achieve the same error level, or they may switch to \ac{sft} and use more samples to further reduce error. Which force dominates depends on the regime determined by the pretrained model’s coverage and the quality of the personalization data.
\end{tcolorbox}


\subsection{Analysis of LSCG with Heterogeneous Users}\label{sec:hetero_mfcg}
In this section, we study an \ac{lscg} with two types of users, $t\in\{t_1,t_2\}$, with population shares $T(t_1)=q$, and $ T(t_2)=1-q.$ To distinguish parameters across types, we use the subscript $i$ to index user types. For example, $\tilsigma_1$ and $\tilsigma_2$ denote the noise amplitudes of type~1 and type~2 users, respectively. As in Section~\ref{sec:homo_mfcg}, we adopt the simplification for \ac{icl} and \ac{sft} errors given in~\eqref{eq:simplified}.

\paragraph{Characterization of the Equilibrium.} Analogous to the homogeneous case, for each user type $t$ and algorithm $a\in\mathcal A$ we define the minimal cost and the corresponding optimal number of samples as
\begin{align*}
\Phi_a(t,H)=\min_{N\geq 0}\big\{E_a(t,N)+R_aNH\big\}, \quad N_a(t,H) = \argmin_{N\geq 0}\big\{E_a(t,N)+R_aNH\big\}.
\end{align*}
If all users are of type $t$ and are restricted to a single algorithm $a$, the resulting equilibrium is characterized by the unique solution $H_a^{\ast}(t)$ to the fix-point equation 
\begin{align*}
    H = \barh(R_a\cdot N_a(t,H),p) =(R_a\cdot N_a(t,H))^2+p.
\end{align*}
In addition, to compare algorithms for type-$t$ users, we define the difference of \ac{sft} and \ac{icl} in minimal costs as $\psi(t,H)=\Phi_{\sft}(t,H)-\Phi_{\icl}(t,H).$ By the homogeneous analysis in Section~\ref{sec:homo_mfcg}, the equation $\psi(t,H)=0$ admits a unique solution for each type $t$. We denote this threshold by $H_{\sep}^{\ast}(t)$ and, without loss of generality, assume $H_{\sep}^{\ast}(t_2)\le H_{\sep}^{\ast}(t_1)$. Before stating our main result, define 
\begin{align*}
e(H)  &= qR_{\sft}N_{\sft}(t_1,H) + (1-q)R_{\sft}N_{\sft}(t_2,H),\\ f(H) &= qR_{\sft}N_{\sft}(t_1,H) + (1-q)R_{\icl}N_{\icl}(t_2,H),\\
g(H)  &= qR_{\icl}N_{\icl}(t_1,H) + (1-q)R_{\icl}N_{\icl}(t_2,H).
\end{align*}
Here, $N_{\sft}(t,H)$ and $N_{\icl}(t,H)$ denote the optimal numbers of personalization samples chosen by a type-$t$ user under \ac{sft} and \ac{icl}, respectively, at unit resource cost $H$. Accordingly, $e(H)$ represents aggregate resource demand when all users adopt \ac{sft}; $f(H)$ represents aggregate resource demand when type-$t_1$ users adopt \ac{sft} and type-$t_2$ users adopt \ac{icl}; and $g(H)$ represents aggregate resource demand when all users adopt \ac{icl}. In each case, users choose their sample sizes optimally given $H$. We next characterize the equilibrium of \eqref{eq:simplified} under $\alpha=1$, $T(t_1)=q$, $T(t_2)=1-q$, $h(x)=x^2$, and $R_{\sft}>R_{\icl}$.


\begin{theorem}[Equilibrium with Two-type Users]\label{thm:hetero}
    When $H_{\sep}^{\ast}(t_2)< H_{\sep}^{\ast}(t_1)$, the equilibrium has a threshold structure:  at low congestion all users choose \ac{sft}, at high congestion all users choose \ac{icl}, and at intermediate congestion type-$t_1$ users choose \ac{sft} while type-$t_2$ users choose \ac{icl}, with possible mixing only at the two indifference thresholds. Within each regime, $H^*$ is characterized by the fixed point $H=\Phi(H)^2+p$ for some function $\Phi\in\{e,f,g\}$, and $R^*=\sqrt{H^*-p}$.
\end{theorem}
The full characterization, including the case $H_{\sep}^*(t_2)=H_{\sep}^*(t_1)$, is provided in Appendix~\ref{app:hetero}. The key insight is that heterogeneity makes equilibrium congestion a regime-switching outcome governed by a small number of type-specific thresholds, illustrated in Figure~\ref{fig:hetero}. The separator $H_{\sep}^*(t)$ summarizes a type’s willingness to pay in resource costs to use the more resource-intensive method. The ordering  $H_{\sep}^*(t_2)<H_{\sep}^*(t_1)$ implies that 
\begin{wrapfigure}{r}{\myfigwidth}
    \centering
    \includegraphics[width=\myfigwidth]{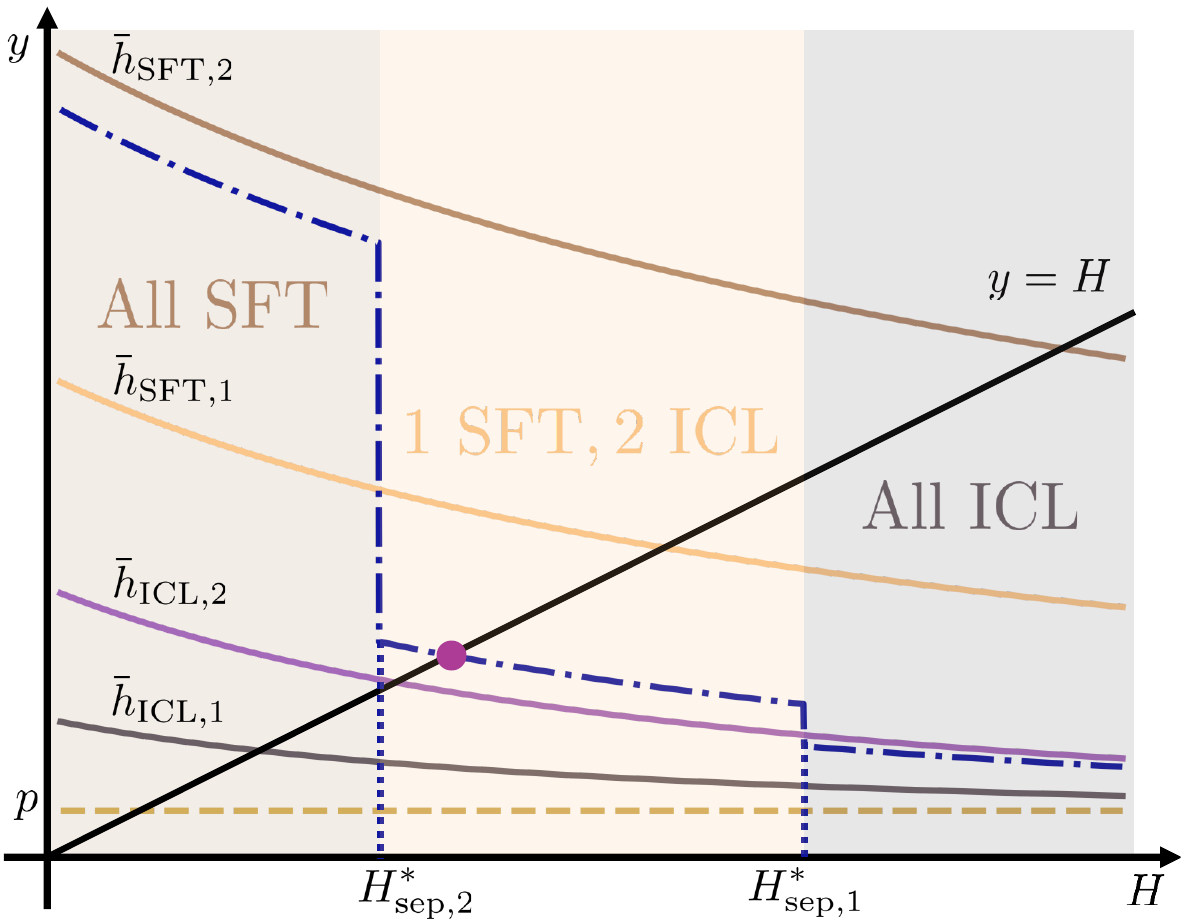} 
    \caption{This figure illustrates the equilibrium with two user types. 
For $i\in\{1,2\}$ and $a\in\calA$, we define $\bar h_{a,i}(H)= p+h(R_a N_a(t_i,H))$. 
The intersection of the line $y=H$ with the blue dotted curve corresponds to the equilibrium.}
    \label{fig:hetero}
\end{wrapfigure}
type $t_2$ switches first as congestion prices change. The aggregate resource usage is determined either in interior regimes with strict method choices or in boundary regimes where the marginal type mixes and absorbs parameter changes through its mixing probability. Consequently, comparative statics can be locally muted in mixing regimes, where $H^*$ is pinned at a separator, but shift sharply when the economy crosses a threshold, as small parameter changes induce discrete shifts in method adoption. 
The extension to $K$ user types is direct.


\paragraph{Comparative Statics of Equilibrium Congestion Level.} To study user interactions in the \ac{lscg}, we restrict heterogeneity to the noise parameter $\tilsigma$, holding all other parameters fixed. The equilibrium congestion level can then be written as $R^*=R^*(\tilsigma_1,\tilsigma_2,q)$. We say that equilibrium congestion is \emph{anchored} by type $\tilsigma_2$ if, holding $\tilsigma_2$ fixed, small changes in $\tilsigma_1$ or the population share $q$ leave $R^*$ unchanged, i.e., $\partial_{\tilsigma_1}R^*=\partial_q R^*=0$. The following proposition characterizes when anchoring obtains for fixed $\tilsigma_2>0$ with $q\in(0,1)$ and $\tilsigma_1<\tilsigma_2$.

\begin{proposition}[Necessary and sufficient conditions for $\partial_{\tilde{\sigma}_1}R^{\ast}=\partial_q R^{\ast}=0$]\label{prop:only_one_type}
The equilibrium congestion level is anchored by type $\tilsigma_2$, i.e.,  $\partial_{\tilde{\sigma}_1}R^*=\partial_qR^*=0$ if and only if the equilibrium lies in a locally ``flat'' regime:
either (i) the fixed point is \emph{pinned} at the type-2 separator, $H^*=H_{\sep}^*(\tilde{\sigma}_2)$
(ii) the equilibrium is the corner $R^*=0$.

The following sufficient conditions ensure the pinning regime in (i):
\begin{itemize}
\item The data noise level $\tilde{\sigma}_2$ lies in an intermediate band, so the type-2 separator falls in the interior mixing region rather than near the zero-congestion corner or an extreme-congestion regime.
\item The pretraining precision proxy $\pi/\zeta$ and the residual scale $(d-r)\tau$ are large, so the relevant demand mappings are well separated.
\item The type-1 population share $q$ is small, so type 1 does not shift the fixed point away from the type-2 separator.
\end{itemize}

\end{proposition}
A formal statement is in Appendix~\ref{app:only_one_type}. Proposition~\ref{prop:only_one_type} shows that the stationarity condition $\partial_{\tilsigma_1}R^{\ast}=\partial_q R^{\ast}=0$ arises only when equilibrium congestion is pinned, rather than determined by smooth interior adjustments. In the first regime, $H^{\ast}=h(R^{\ast})+p$ is pinned at the type-$2$ separator: type-$2$ users are indifferent between \ac{sft} and \ac{icl}, and small changes in $(\tilsigma_1,q)$ can be absorbed through their mixing behavior. In the second 
\begin{wrapfigure}{r}{\myfigwidth}
    \centering
    \includegraphics[width=\myfigwidth]{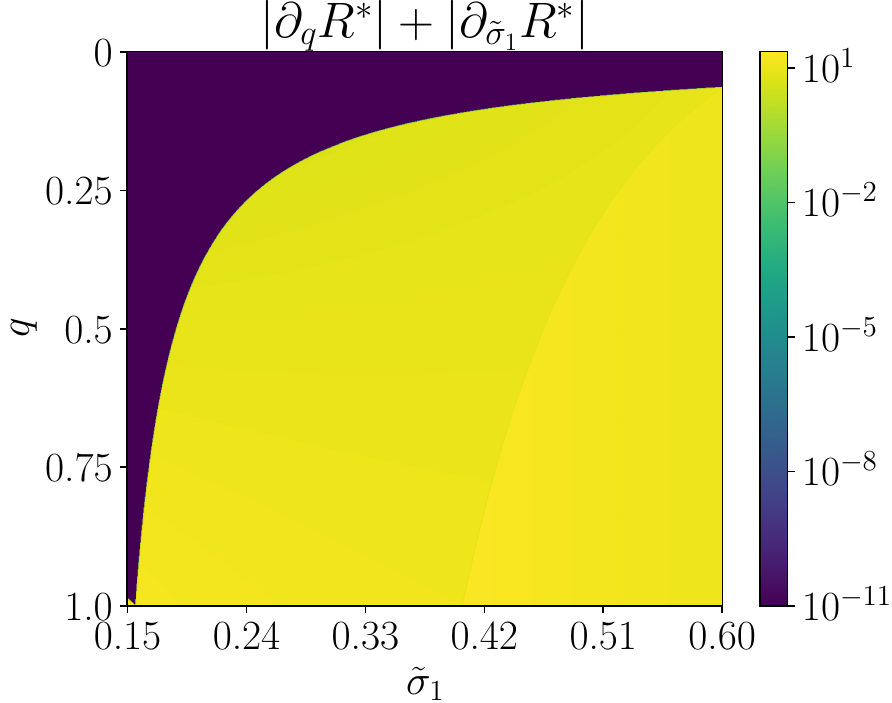} 
    \caption{The figure reports $|\partial_q R^{\ast}|+|\partial_{\tilsigma_1}R^{\ast}|$ across different values of the type-$1$ task noise $\tilsigma_1$ and population share $q$. In the dark-blue region, the near-zero sensitivity indicates that the equilibrium congestion level is pinned by the type-$2$ separator and is locally insensitive to changes in $\tilsigma_1$ and $q$.}
    \vspace{-1.2em}
    \label{fig:p1_grad}
\end{wrapfigure}
regime, $R^{\ast}=0$ is a stable zero-congestion corner: both types strictly prefer \ac{icl} and choose $\tilN_{\icl}=0$.

The sufficient conditions place the game in a robust type-2-pinned mixing regime. The bounds on $\tilsigma_2$ position $H_{\sep}^{\ast}(\tilsigma_2)$ away from both the zero-congestion corner and the extreme-congestion regime, ensuring an interior mixing wedge. A small pretraining-error index $\zeta$ makes \ac{icl} effective, lowering marginal \ac{icl} demand and widening the resource gap between the type-2 \ac{icl} and \ac{sft} choices at the separator. A large uncovered-error floor $(d-r)\tau$ further ensures strict inequalities, so the separator is a genuine switching locus rather than a knife-edge. Finally, the upper bound on $q$ keeps the type-1 population sufficiently small, and hence leaves enough type-2 mass for type-2 mixing behavior to pin aggregate congestion; the admissible type-1 share decreases with $\tilsigma_1$. This setting is illustrated in Figure~\ref{fig:p1_grad}.

\begin{tcolorbox}[enhanced jigsaw,breakable,frame hidden, left = 0.1mm, right = 0.1mm, top = 0.1mm, bottom = 1.5mm]
\refstepcounter{finding}\label{find:hetero_game}
\textbf{Finding \thefinding:}
The equilibrium becomes anchored in two cases: either one type of users is exactly indifferent between \ac{icl} and \ac{sft}, or users choose not to use any personalization data because computation is too costly. In these cases, small parameter changes are absorbed by adjustments in mixing behavior or by continued inactivity, so aggregate compute demand remains locally unchanged.

\end{tcolorbox}

\section{Stackelberg LLM-Serving Congestion Game Between a Platform and Users}\label{sec:smfcg}
\paragraph{Equilibrium Definition.} Section~\ref{sec:mfcg} characterizes equilibrium properties under a fixed platform serving strategy, such as a fixed price $p$. In practice, however, the platform seeks to maximize its profit by optimizing the serving strategy. This section asks: which strategies maximize the platform’s profit at the induced user equilibrium? The interaction between the platform and users proceeds as follows.
\begin{itemize}
    \item[1.] The platform sets the unit resource price $p$.
    \item[2.] Given $p$, users reach the equilibrium $(\pi^{*}(p),R^{*}(p))$ of the induced game via \eqref{eq:cost} and \eqref{eq:sft_N}; we explicitly indicate the dependence on $p$.
    \item[3.] The platform collects profit $I(p)=(p-1)\cdot R^{*}(p)$.
\end{itemize}
In the platform’s profit, the constant $1$ represents the normalized cost of one unit of computational resource. We highlight that the profit $I(p)$ is well-defined: by Theorem~\ref{thm:exist_unique}, for each price $p$ the followers’ game admits an equilibrium with a unique congestion level, so the induced aggregate consumption $R^*(p)$ and the resulting profit $I(p)$ are single-valued.  This interaction can be formulated as a \ac{slscg}, with the platform acting as the leader that sets the price and users as followers. The equilibrium concept is defined as follows.

\begin{definition}
    The triplet $(p^{*},\pi^{*},R^{*})$ is the Stackelberg equilibrium if
    \begin{itemize}
        \item (Follower equilibrium) The pair $(\pi^{*},R^{*})$ is the equilibrium of the \ac{lscg} induced by price $p^*$.
        \item (Leader optimality) The price $p^*$ maximizes profit $I(p)=(p-1)R^*(p)$ over $p>0$.

    \end{itemize}
\end{definition}
Intuitively, the platform moves first and anticipates how users will respond to its price. For any price $p$, users choose their personalization policies and reach an equilibrium resource demand $R^*(p)$. The platform then selects the price that maximizes its profit, given by the per-unit margin $p-1$ multiplied by equilibrium resource consumption. Thus, the Stackelberg equilibrium combines users’ optimal responses with the platform’s optimal pricing decision.
\paragraph{Price Effects on Equilibrium and Profit.} We then analyze how the price $p$ affects the equilibrium congestion level $R^{*}(p)$ and profit $I(p)$.

\begin{theorem}[Monotonicity of equilibrium congestion in price]
\label{thm:r_decreasing_p}
Under regularity conditions of $h$, the unique equilibrium congestion level
$R^*(p)$ is decreasing in the unit resource price $p$. That is, for any $p_1<p_2$, we have that $R^*(p_2)\le R^*(p_1)$.
\end{theorem}
This theorem formalizes congestion pricing: a higher unit resource price $p$ raises users’ marginal cost, shifting best responses toward less resource-intensive behavior. As a result, users demand fewer personalization samples and substitute away from compute-heavy algorithms, reducing aggregate demand at every conjectured congestion level and implying that the unique equilibrium congestion $R^*(p)$ is decreasing in $p$.

Although $R^*(p)$ decreases in $p$, a natural question is whether the profit-maximizing price $p^*$ is finite. For analytical tractability, we focus on the case $\alpha=1$ for \eqref{eq:sft_N}.

\begin{proposition}[Boundedness of Platform Profit]\label{prop:s_exist}
    When regularity conditions of $h$ holds, $\alpha=1$, and the following holds
    \begin{align*}
        \sup_{t\in\calT} ((d-r)\tau + r\zeta)^{2}(d-r)^{-1}\tilsigma^{-2}+\bar{\sigma}^2r(\pi\zeta-\bar{\sigma}^2)\pi^{-2}\tilsigma^{-2}<\infty,
    \end{align*}
    the Stackelberg equilibrium $(p^{*},\pi^{*},R^{*})$ exists with $p^*<\infty$ and $I(p^*)<\infty$.
\end{proposition}
This proposition establishes conditions under which the platform’s Stackelberg problem is well posed and admits a finite optimal price $p^*<\infty$. The key requirement is to rule out $p\to\infty$ as an optimal strategy by showing that, for sufficiently large $p$, equilibrium congestion $R^*(p)$ decays fast enough that the revenue term $(p-1)R^*(p)$ eventually decreases. Intuitively, high prices dominate users’ incentives, leading them to reduce personalization sample sizes and substitute toward compute-light methods (\ac{icl}); beyond a threshold, even \ac{icl} is not worth sampling, so users optimally choose $\tilN_{\icl}=0$ and $R^*(p)\to 0$. The boundedness condition ensures this demand collapse occurs uniformly across types, implying that the objective cannot be maximized at $p=\infty$ and hence that a finite optimal $p^*$ exists.

\paragraph{Menu Effects on Equilibrium and Profit.} Beyond pricing, platforms can also choose which personalization algorithms to offer. Early \ac{llm} services provided only \ac{icl}~\citep{openai2021service}. We therefore study whether offering both \ac{sft} and \ac{icl} yields higher platform profit than offering \ac{icl} alone.
\begin{assumption}[Relative Resource Consumption of \ac{sft} and \ac{icl}]\label{assump:large_large}
    The resource consumed of \ac{sft} per sample is larger than that of \ac{icl} with a multiplier: $R_{\sft}\ge R_{\icl}\cdot \sup_{t\in\calT}(r\beta(t)/(d-r))^{1/\alpha}$, where $\beta(t)=1+\pi(\tau+m^2)/\bar\sigma^2$.
\end{assumption}
This assumption requires that \ac{sft} and \ac{icl} differ substantially in per-sample resource consumption. It imposes a uniform compute gap between $R_{\sft}$ and $R_{\icl}$ so that, in equilibrium and for all user types, \ac{sft} is the genuinely compute-heavy personalization method.

We refer to a game as the \emph{full game} when the platform offers both \ac{sft} and \ac{icl}, and as the \ac{icl} game when only \ac{icl} is available. Given price $p$, the corresponding equilibrium congestion levels are denoted by $R^*(p)$ and $R_{\icl}^*(p)$, respectively.
\begin{theorem}[Choice Diversity Does Not Harm Profit]\label{thm:sft_better}
    We denote the congestion levels at the equilibrium of the full game and the \ac{icl} game by $R^{*}(p)$ and $R_{\icl}^{*}(p)$, respectively. Under regularity conditions of $h$ and Assumption~\ref{assump:large_large}, it holds that $R^{*}(p)\geq R_{\icl}^{*}(p)$ for any $p>0$. Therefore, adding \ac{sft} to the platform’s personalization menu cannot reduce optimal profit, i.e., $\sup_{p>0}(p-1)R^{*}(p)\geq \sup_{p>0}(p-1)R_{\icl}^{*}(p)$.
\end{theorem}
The key insight is that adding \ac{sft} as an available personalization option cannot reduce equilibrium congestion. Economically, \ac{sft} is the high-resource method, so whenever some users find \ac{sft} privately cost-effective, they switch into a more compute-intensive mode of personalization, raising aggregate demand and thus the congestion fixed point. In short, expanding the menu of personalization methods can increase platform load even though it gives users more flexibility, because the new option is compute-heavy.
\begin{tcolorbox}[enhanced jigsaw,breakable,frame hidden, left = 0.1mm, right = 0.1mm, top = 0.1mm, bottom = 1.5mm]
\refstepcounter{finding}\label{find:stackelberg}
\textbf{Finding \thefinding:}
Adding \ac{sft} cannot lower the equilibrium congestion level because it expands users’ options with a compute-heavy personalization method, so any adoption raises aggregate resource demand and platform load.
\end{tcolorbox}

\section{Mapping the Model to Practice}\label{sec:exp}
In this section, we connect our theoretical analysis to practical \ac{llm} serving by validating the model in Section~\ref{sec:model} and testing our theoretical predictions.

\begin{figure}[t]
\centering
\subfigure[Estimation erros of \ac{icl} with different $\tilN$.]{\includegraphics[width=\myfigwidth]{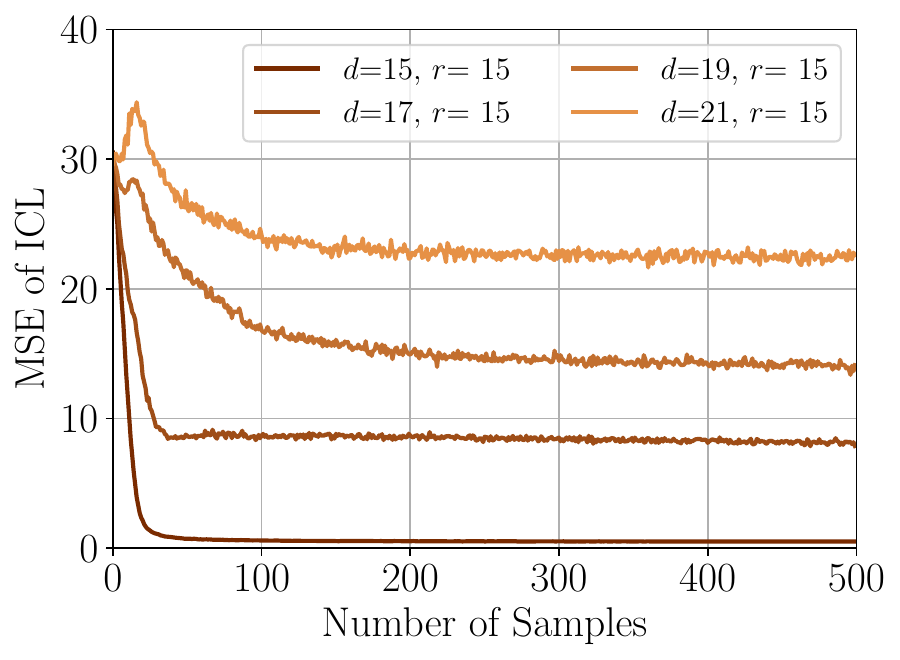}\label{fig:icl}}
\hspace{0.3em}
\subfigure[\ac{icl} errors with $500$ samples and coverage $r=15$.]{\includegraphics[width=\myfigwidth]{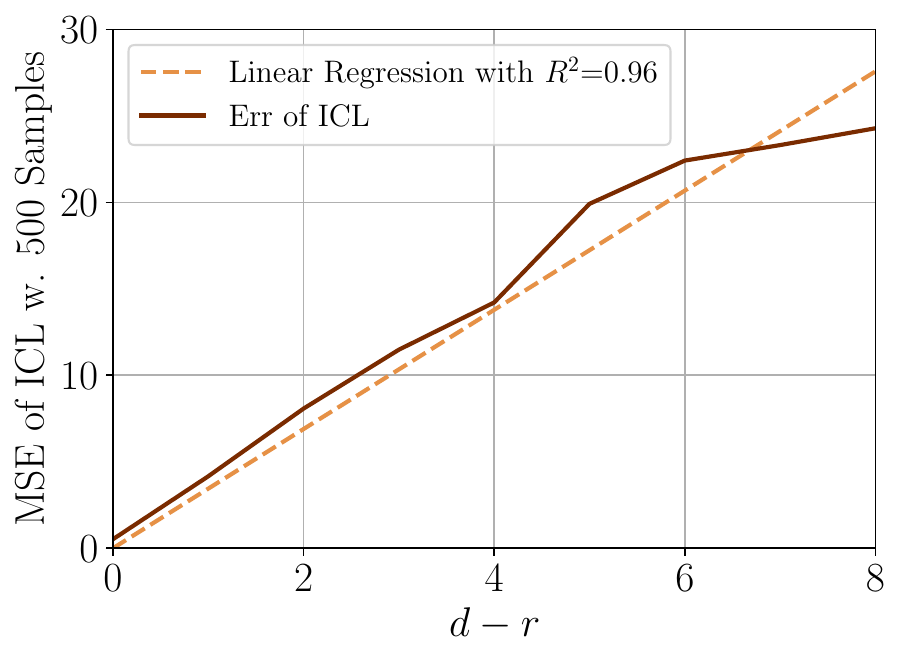}\label{fig:icl_500}}
\caption{The figures show the estimation error of \ac{icl} across different numbers of in-context samples $\tilN$ and ambient dimensions $d$. The results confirm that \ac{icl} cannot eliminate the error in the subspace uncovered by pretraining, even with infinitely many samples. This irreducible error scales linearly with the number of uncovered dimensions, $d-r$.}
\end{figure}

\vspace{10pt}
\noindent\textbf{Experimental Setting.}
We train a $22$M-parameter GPT-2 model~\citep{radford2019language} and evaluate it on in-context linear regression, a standard benchmark for \acp{llm} behavior~\citep{garg2022can,von2023transformers}. The model is prompted with $\tilN$ i.i.d.\ samples $\{(\tilx_i,\tily_i)\}_{i=1}^{\tilN}$ generated by~\eqref{eq:task_data} and a query $\tilx_{\mathrm{q}}$, and outputs a prediction $\hat y$ of $\tilx_{\mathrm{q}}^\top\theta^*$, evaluated by squared error. To isolate the differences between \ac{sft} and \ac{icl} predicted by Theorem~\ref{thm:err_analysis} and Proposition~\ref{prop:comp}, we pretrain the model from scratch under a rank-deficient design with $r<d$, where only the first $r$ coordinates of $x_i\in\mathbb{R}^d$ are nonzero, and set $r=15$. After pretraining, personalization is performed either via \ac{icl}, by prompting with $\{(\tilx_i,\tily_i)\}_{i=1}^{\tilN}$, or via \ac{sft}, by fine-tuning on the same data for at least 100 steps, as fewer steps produce negligible deviation from the pretrained model. Additional details are provided in Appendix~\ref{app:exp_details}.

For \ac{icl}, Theorem~\ref{thm:err_analysis} yields two implications. First, even as the number of personalization samples diverges, the error does not vanish but converges to an irreducible bias induced by directions not covered in pretraining. Second, under the isotropic setting (Assumption~\ref{assump:homo}), this bias equals $(d-r)\tau$ and is linear in the uncovered dimension. Figure~\ref{fig:icl} confirms the first result, showing that error decreases with sample size before plateauing, while Figure~\ref{fig:icl_500} confirms the second, with the limiting bias regressed on $(d-r)$ yielding the coefficient of determination as $R^2=0.96$.

Proposition~\ref{prop:comp} identifies two regimes comparing \ac{icl} and \ac{sft}. When the effective signal-to-noise ratio is high ($\kappa>1$), corresponding to many personalization samples, \ac{sft} attains lower error than \ac{icl}; when samples are scarce ($\kappa\to 0$), \ac{sft} incurs higher error. Figures~\ref{fig:icl} and~\ref{fig:sft} confirm these
predictions: \ac{sft} outperforms \ac{icl} at large sample sizes (e.g., $500$), but performs worse at very small sample sizes (e.g., near $1$).

\begin{figure}[t]
\centering
\subfigure[Estimation erros of \ac{sft} with different $\tilN$.]{\includegraphics[width=\myfigwidth]{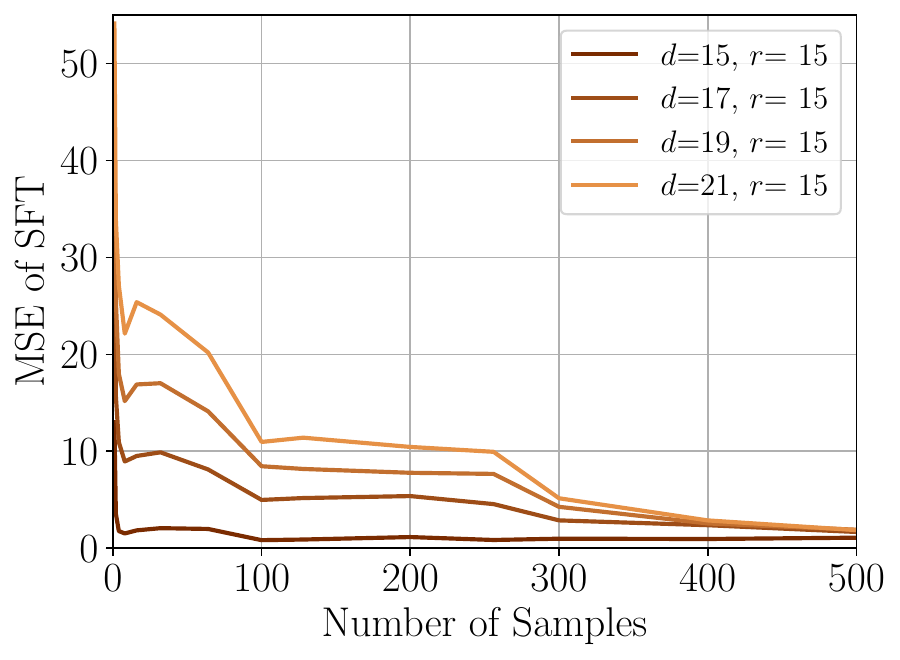}\label{fig:sft}}
\hspace{0.3em}
\subfigure[Percentage of platforms supporting \ac{sft}.]{\includegraphics[width=\myfigwidth]{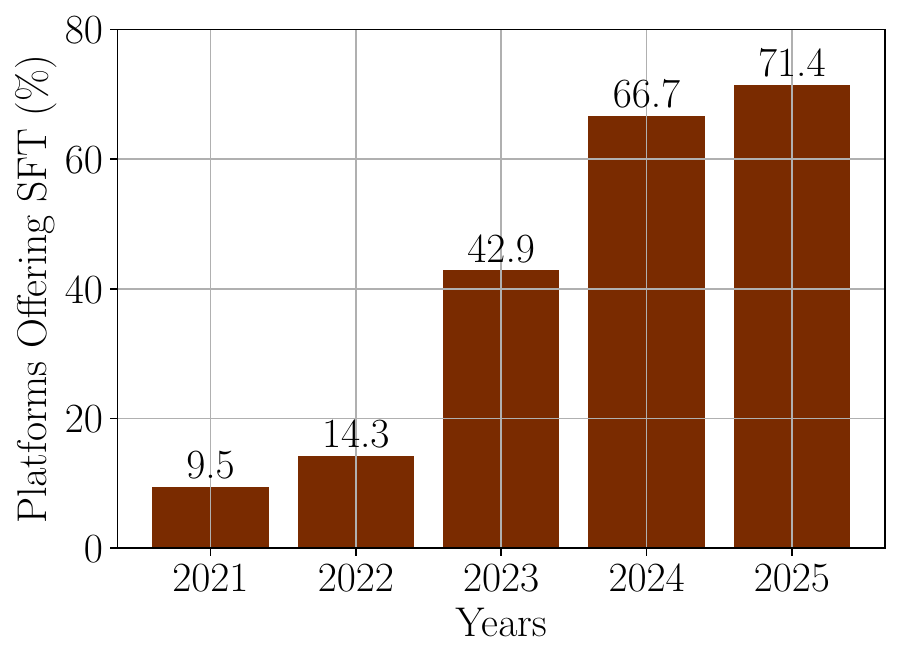}\label{fig:sft_service}}
\caption{Panel (a) shows the estimation error of \ac{sft} across different numbers of personalization samples $\tilN$, illustrating that \ac{sft} can asymptotically drive the estimation error to zero as the sample size increases. Panel (b) reports the fraction of platforms that offer \ac{sft} in addition to \ac{icl}, showing that support for \ac{sft}-based personalization has increased in recent years.}
\end{figure}

Beyond \ac{llm} behavior, we study platforms’ serving strategies—specifically, whether they offer only \ac{icl} or both \ac{sft} and \ac{icl}. We summarize the offerings of 21 companies releasing major foundation models listed in the Artificial Intelligence Index Reports 2024 and 2025~\citep{maslej2024artificial,maslej2025artificial}, with details in Appendix~\ref{app:add_exp}. Figure~\ref{fig:sft_service} shows a clear upward trend in the fraction of platforms offering both methods, consistent with Theorem~\ref{thm:sft_better}, which predicts that adding the \ac{sft} option does not reduce platform profit.

\section{Discussion and Conclusion}
The personalization of \acp{llm} is not just a technical challenge—it is an economic one. As billions of users customize AI systems for specialized tasks on shared computational infrastructure, their choices create a complex web of strategic interactions. When should a user invest in expensive \ac{sft} versus lightweight \ac{icl}? How does congestion from others' personalization decisions reshape these incentives? And what strategies should platforms pursue when computational resources are scarce and valuable?

This paper developed the first theoretical framework that jointly addresses these questions. By building linear statistical models of personalization algorithms and embedding them in a continuum game with congestion, we obtained sharp predictions about user behavior and platform strategy. Three core insights emerge. First, the comparison between \ac{sft} and \ac{icl} hinges on a precise interplay between pretraining coverage and data quality—not simply which method is "better." Second, equilibrium congestion responds to system parameters in surprisingly non-monotone ways: better pretraining can increase peak loads, harder tasks can reduce demand, and higher \ac{sft} costs can initially amplify congestion before alleviating it. Third, offering algorithmic diversity—both \ac{sft} and \ac{icl}—never hurts platform profits, rationalizing the rapid adoption of multiple personalization methods we observe in practice.

Our model admits several natural extensions. First, we assume a single platform serving all users, whereas in practice, multiple platforms compete. This can be incorporated by allowing users to choose among platforms, with congestion defined at the platform level. Second, beyond personalization algorithms, platforms increasingly offer reasoning models that deliver higher-quality responses at the cost of greater latency (e.g., GPT-o1). Modeling such choices requires expanding the set of algorithms available in our framework.


\bibliography{sample-bibliography}

@String{Springer = "Springer-Verlag" }

@article{henderson1974road,
  title={Road congestion: a reconsideration of pricing theory},
  author={Henderson, J Vernon},
  journal={Journal of Urban Economics},
  volume={1},
  number={3},
  pages={346--365},
  year={1974},
  publisher={Elsevier}
}

@book{small2024economics,
  title={The economics of urban transportation},
  author={Small, Kenneth A and Verhoef, Erik T and Lindsey, Robin},
  year={2024},
  publisher={Routledge}
}

@article{redding2015transportation,
  title={Transportation costs and the spatial organization of economic activity},
  author={Redding, Stephen J and Turner, Matthew A},
  journal={Handbook of regional and urban economics},
  volume={5},
  pages={1339--1398},
  year={2015},
  publisher={Elsevier}
}

@article{duranton2011fundamental,
  title={The fundamental law of road congestion: Evidence from US cities},
  author={Duranton, Gilles and Turner, Matthew A},
  journal={American Economic Review},
  volume={101},
  number={6},
  pages={2616--2652},
  year={2011},
  publisher={American Economic Association}
}

@book{patriksson2015traffic,
  title={The traffic assignment problem: models and methods},
  author={Patriksson, Michael},
  year={2015},
  publisher={Courier Dover Publications}
}

@article{kreindler2024peak,
  title={Peak-Hour Road Congestion Pricing: Experimental Evidence and Equilibrium Implications},
  author={Kreindler, Gabriel},
  journal={Econometrica},
  volume={92},
  number={4},
  pages={1233--1268},
  year={2024},
  publisher={Wiley Online Library}
}

@article{hu2024unveiling,
  title={Unveiling the statistical foundations of chain-of-thought prompting methods},
  author={Hu, Xinyang and Zhang, Fengzhuo and Chen, Siyu and Yang, Zhuoran},
  journal={arXiv preprint arXiv:2408.14511},
  year={2024}
}

@article{carmona2013probabilistic,
  title={Probabilistic analysis of mean-field games},
  author={Carmona, Ren{\'e} and Delarue, Fran{\c{c}}ois},
  journal={SIAM Journal on Control and Optimization},
  volume={51},
  number={4},
  pages={2705--2734},
  year={2013},
  publisher={SIAM}
}

@book{carmona2018probabilistic,
  title={Probabilistic theory of mean field games with applications I-II},
  author={Carmona, Ren{\'e} and Delarue, Fran{\c{c}}ois and others},
  volume={3},
  year={2018},
  publisher={Springer}
}

@article{meng2022locating,
  title={Locating and editing factual associations in gpt},
  author={Meng, Kevin and Bau, David and Andonian, Alex and Belinkov, Yonatan},
  journal={Advances in neural information processing systems},
  volume={35},
  pages={17359--17372},
  year={2022}
}

@article{wang2025muon,
  title={Muon outperforms Adam in tail-end associative memory learning},
  author={Wang, Shuche and Zhang, Fengzhuo and Li, Jiaxiang and Du, Cunxiao and Du, Chao and Pang, Tianyu and Yang, Zhuoran and Hong, Mingyi and Tan, Vincent YF},
  journal={arXiv preprint arXiv:2509.26030},
  year={2025}
}

@article{zhang2024learning,
  title={Learning regularized graphon mean-field games with unknown graphons},
  author={Zhang, Fengzhuo and Tan, Vincent YF and Wang, Zhaoran and Yang, Zhuoran},
  journal={Journal of Machine Learning Research},
  volume={25},
  number={372},
  pages={1--95},
  year={2024}
}

@article{xie2021explanation,
  title={An explanation of in-context learning as implicit bayesian inference},
  author={Xie, Sang Michael and Raghunathan, Aditi and Liang, Percy and Ma, Tengyu},
  journal={arXiv preprint arXiv:2111.02080},
  year={2021}
}

@article{guo2019learning,
  title={Learning mean-field games},
  author={Guo, Xin and Hu, Anran and Xu, Renyuan and Zhang, Junzi},
  journal={Advances in neural information processing systems},
  volume={32},
  year={2019}
}

@article{luo2025empirical,
  title={An empirical study of catastrophic forgetting in large language models during continual fine-tuning},
  author={Luo, Yun and Yang, Zhen and Meng, Fandong and Li, Yafu and Zhou, Jie and Zhang, Yue},
  journal={IEEE Transactions on Audio, Speech and Language Processing},
  year={2025},
  publisher={IEEE}
}

@article{akyureklearning,
  title={What learning algorithm is in-context learning? Investigations with linear models},
  author={Aky{\"u}rek, Ekin and Schuurmans, Dale and Andreas, Jacob and Ma, Tengyu and Zhou, Denny},
  journal={The Eleventh International Conference on Learning Representations},
  year={2023}
}

@techreport{chatterji2025people,
  title={How people use chatgpt},
  author={Chatterji, Aaron and Cunningham, Thomas and Deming, David J and Hitzig, Zoe and Ong, Christopher and Shan, Carl Yan and Wadman, Kevin},
  year={2025},
  institution={National Bureau of Economic Research}
}

@article{lacker2017limit,
  title={Limit theory for controlled McKean--Vlasov dynamics},
  author={Lacker, Daniel},
  journal={SIAM Journal on Control and Optimization},
  volume={55},
  number={3},
  pages={1641--1672},
  year={2017},
  publisher={SIAM}
}

@article{lu2025asymptotic,
  title={Asymptotic theory of in-context learning by linear attention},
  author={Lu, Yue M and Letey, Mary and Zavatone-Veth, Jacob A and Maiti, Anindita and Pehlevan, Cengiz},
  journal={Proceedings of the National Academy of Sciences},
  volume={122},
  number={28},
  pages={e2502599122},
  year={2025},
  publisher={National Academy of Sciences}
}

@article{su2025isotropic,
  title={Isotropic Curvature Model for Understanding Deep Learning Optimization: Is Gradient Orthogonalization Optimal?},
  author={Su, Weijie},
  journal={arXiv preprint arXiv:2511.00674},
  year={2025}
}

@article{mallinar2024minimum,
  title={Minimum-norm interpolation under covariate shift},
  author={Mallinar, Neil and Zane, Austin and Frei, Spencer and Yu, Bin},
  journal={arXiv preprint arXiv:2404.00522},
  year={2024}
}

@inproceedings{li2023transformers,
  title={Transformers as algorithms: Generalization and stability in in-context learning},
  author={Li, Yingcong and Ildiz, Muhammed Emrullah and Papailiopoulos, Dimitris and Oymak, Samet},
  booktitle={International conference on machine learning},
  pages={19565--19594},
  year={2023},
  organization={PMLR}
}

@article{weerawardhena2025llama,
  title={Llama-3.1-foundationai-securityllm-8b-instruct technical report},
  author={Weerawardhena, Sajana and Kassianik, Paul and Nelson, Blaine and Saglam, Baturay and Vellore, Anu and Priyanshu, Aman and Vijay, Supriti and Aufiero, Massimo and Goldblatt, Arthur and Burch, Fraser and others},
  journal={arXiv preprint arXiv:2508.01059},
  year={2025}
}

@article{anahtarci2023q,
  title={Q-learning in regularized mean-field games},
  author={Anahtarci, Berkay and Kariksiz, Can Deha and Saldi, Naci},
  journal={Dynamic Games and Applications},
  volume={13},
  number={1},
  pages={89--117},
  year={2023},
  publisher={Springer}
}

@inproceedings{yardim2023policy,
  title={Policy mirror ascent for efficient and independent learning in mean field games},
  author={Yardim, Batuhan and Cayci, Semih and Geist, Matthieu and He, Niao},
  booktitle={International Conference on Machine Learning},
  pages={39722--39754},
  year={2023},
  organization={PMLR}
}

@article{lauriere2022learning,
  title={Learning in mean field games: A survey},
  author={Lauri{\`e}re, Mathieu and Perrin, Sarah and P{\'e}rolat, Julien and Girgin, Sertan and Muller, Paul and {\'E}lie, Romuald and Geist, Matthieu and Pietquin, Olivier},
  journal={arXiv preprint arXiv:2205.12944},
  year={2022}
}

@article{carmona2016mean,
  title={Mean field games with common noise},
  author={Carmona, Ren{\'e} and Delarue, Fran{\c{c}}ois and Lacker, Daniel},
  year={2016}
}

@article{graber2023monotonicity,
  title={On monotonicity conditions for mean field games},
  author={Graber, P Jameson and M{\'e}sz{\'a}ros, Alp{\'a}r R},
  journal={Journal of Functional Analysis},
  volume={285},
  number={9},
  pages={110095},
  year={2023},
  publisher={Elsevier}
}

@article{lasry2007mean,
  title={Mean field games},
  author={Lasry, J. and Lions, P.},
  journal={Japanese journal of mathematics},
  volume={2},
  number={1},
  pages={229--260},
  year={2007},
  publisher={Springer}
}

@article{zhang2023learning,
  title={Learning regularized monotone graphon mean-field games},
  author={Zhang, Fengzhuo and Tan, Vincent and Wang, Zhaoran and Yang, Zhuoran},
  journal={Advances in Neural Information Processing Systems},
  volume={36},
  pages={67297--67308},
  year={2023}
}

@inproceedings{duetting2024mechanism,
  title={Mechanism design for large language models},
  author={Duetting, Paul and Mirrokni, Vahab and Paes Leme, Renato and Xu, Haifeng and Zuo, Song},
  booktitle={Proceedings of the ACM Web Conference 2024},
  pages={144--155},
  year={2024}
}

@article{hajiaghayi2024ad,
  title={Ad auctions for llms via retrieval augmented generation},
  author={Hajiaghayi, MohammadTaghi and Lahaie, S{\'e}bastien and Rezaei, Keivan and Shin, Suho},
  journal={Advances in Neural Information Processing Systems},
  volume={37},
  pages={18445--18480},
  year={2024}
}

@inproceedings{dubey2024auctions,
  title={Auctions with llm summaries},
  author={Dubey, Avinava and Feng, Zhe and Kidambi, Rahul and Mehta, Aranyak and Wang, Di},
  booktitle={Proceedings of the 30th ACM SIGKDD Conference on Knowledge Discovery and Data Mining},
  pages={713--722},
  year={2024}
}

@article{hadfield2025economy,
  title={An economy of AI agents},
  author={Hadfield, Gillian K and Koh, Andrew},
  journal={arXiv preprint arXiv:2509.01063},
  year={2025},
  publisher={arXiv}
}

@inproceedings{bergemann2025economics,
  title={The Economics of Large Language Models: Token Allocation, Fine-Tuning, and Optimal Pricing},
  author={Bergemann, Dirk and Bonatti, Alessandro and Smolin, Alex},
  booktitle={Proceedings of the 26th ACM Conference on Economics and Computation},
  pages={786--786},
  year={2025}
}

@article{fish2024algorithmic,
  title={Algorithmic collusion by large language models},
  author={Fish, Sara and Gonczarowski, Yannai A},
  year={2024}
}

@inproceedings{bergemann2025data,
  title={Data-Driven Mechanism Design: Jointly Eliciting Preferences and Information},
  author={Bergemann, Dirk and Bojko, Marek and Duetting, Paul and Paes Leme, Renato and Xu, Haifeng and Zuo, Song},
  booktitle={Proceedings of the 26th ACM Conference on Economics and Computation},
  pages={507--507},
  year={2025}
}

@inproceedings{soumalias2024truthful,
  title={Truthful Aggregation of LLMs$\backslash$$\backslash$with an Application to Online Advertising},
  author={Soumalias, Ermis and Curry, Michael and Seuken, Sven},
  year={2024},
  booktitle={Agentic Markets Workshop at ICML 2024}
}

@inproceedings{von2023transformers,
  title={Transformers learn in-context by gradient descent},
  author={Von Oswald, Johannes and Niklasson, Eyvind and Randazzo, Ettore and Sacramento, Jo{\~a}o and Mordvintsev, Alexander and Zhmoginov, Andrey and Vladymyrov, Max},
  booktitle={International Conference on Machine Learning},
  pages={35151--35174},
  year={2023},
  organization={PMLR}
}

@inproceedings{wolf2020transformers,
  title={Transformers: State-of-the-art natural language processing},
  author={Wolf, Thomas and Debut, Lysandre and Sanh, Victor and Chaumond, Julien and Delangue, Clement and Moi, Anthony and Cistac, Pierric and Rault, Tim and Louf, Remi and Funtowicz, Morgan and others},
  booktitle={Proceedings of the 2020 conference on empirical methods in natural language processing: system demonstrations},
  pages={38--45},
  year={2020}
}

@article{radford2019language,
  title={Language models are unsupervised multitask learners},
  author={Radford, Alec and Wu, Jeffrey and Child, Rewon and Luan, David and Amodei, Dario and Sutskever, Ilya and others},
  journal={OpenAI blog},
  volume={1},
  number={8},
  pages={9},
  year={2019}
}

@article{garg2022can,
  title={What can transformers learn in-context? a case study of simple function classes},
  author={Garg, Shivam and Tsipras, Dimitris and Liang, Percy S and Valiant, Gregory},
  journal={Advances in neural information processing systems},
  volume={35},
  pages={30583--30598},
  year={2022}
}

@techreport{agarwal2025designing,
  title={Designing Human-AI Collaboration: A Sufficient-Statistic Approach},
  author={Agarwal, Nikhil and Moehring, Alex and Wolitzky, Alexander},
  year={2025},
  institution={National Bureau of Economic Research}
}

@article{perrin2020fictitious,
  title={Fictitious play for mean field games: Continuous time analysis and applications},
  author={Perrin, S. and P{\'e}rolat, J. and Lauri{\`e}re, M. and Geist, M. and Elie, R. and Pietquin, O.},
  journal={Advances in Neural Information Processing Systems},
  volume={33},
  pages={13199--13213},
  year={2020}
}

@article{openai2025serviceoutage,
  title   = {OpenAI Service Incident: Major Outage across ChatGPT and API},
  author  = {{OpenAI}},
  journal = {OpenAI Status},
  year    = {2025},
  note    = {Incident report},
  url     = {https://status.openai.com/incidents/01JMYB63BJ47J3SXV6KSCT4D2A}
}

@article{openai2021service,
  title   = {GPT-3 powers the next generation of apps},
  author  = {{OpenAI}},
  journal = {OpenAI Status},
  year    = {2021},
  note    = {Incident report},
  url     = {https://openai.com/index/gpt-3-apps}
}

@article{touvron2023llama,
  title={Llama 2: Open foundation and fine-tuned chat models},
  author={Touvron, Hugo and Martin, Louis and Stone, Kevin and Albert, Peter and Almahairi, Amjad and Babaei, Yasmine and Bashlykov, Nikolay and Batra, Soumya and Bhargava, Prajjwal and Bhosale, Shruti and others},
  journal={arXiv preprint arXiv:2307.09288},
  year={2023}
}

@article{grattafiori2024llama,
  title={The llama 3 herd of models},
  author={Grattafiori, Aaron and Dubey, Abhimanyu and Jauhri, Abhinav and Pandey, Abhinav and Kadian, Abhishek and Al-Dahle, Ahmad and Letman, Aiesha and Mathur, Akhil and Schelten, Alan and Vaughan, Alex and others},
  journal={arXiv preprint arXiv:2407.21783},
  year={2024}
}

@article{brown2020language,
  title={Language models are few-shot learners},
  author={Brown, Tom and Mann, Benjamin and Ryder, Nick and Subbiah, Melanie and Kaplan, Jared D and Dhariwal, Prafulla and Neelakantan, Arvind and Shyam, Pranav and Sastry, Girish and Askell, Amanda and others},
  journal={Advances in neural information processing systems},
  volume={33},
  pages={1877--1901},
  year={2020}
}

@Article{zzzz23,
  author  = {Zhang, Yufeng and Zhang, Fengzhuo and Yang, Zhuoran and Wang, Zhaoran},
  journal = {arXiv preprint arXiv:2305.19420},
  title   = {What and how does in-context learning learn? bayesian model averaging, parameterization, and generalization},
  year    = {2023},
}

@article{maslej2024artificial,
  title={Artificial intelligence index report 2024},
  author={Maslej, Nestor and Fattorini, Loredana and Perrault, Raymond and Parli, Vanessa and Reuel, Anka and Brynjolfsson, Erik and Etchemendy, John and Ligett, Katrina and Lyons, Terah and Manyika, James and others},
  journal={arXiv preprint arXiv:2405.19522},
  year={2024}
}

@article{maslej2025artificial,
  title={Artificial intelligence index report 2025},
  author={Maslej, Nestor and Fattorini, Loredana and Perrault, Raymond and Gil, Yolanda and Parli, Vanessa and Kariuki, Njenga and Capstick, Emily and Reuel, Anka and Brynjolfsson, Erik and Etchemendy, John and others},
  journal={arXiv preprint arXiv:2504.07139},
  year={2025}
}
\bibliographystyle{ims}
\newpage
\appendix
\section{Summary of Notations}\label{app:notation}

We summarize the notation used throughout the paper in Table~\ref{tab:notation}. The notations are organized into three groups: general model setup, \ac{sft} and \ac{icl} analysis, and the game-theoretic framework.

\begin{longtable}{p{3.2cm}p{9.5cm}}
\caption{Summary of notation organized by topic.}\label{tab:notation}\\
\toprule
\textbf{Symbol} & \textbf{Description} \\
\midrule
\endfirsthead
\toprule
\textbf{Symbol} & \textbf{Description} \\
\midrule
\endhead
\midrule
\endfoot
\bottomrule
\endlastfoot

\multicolumn{2}{l}{\textbf{General}} \\
\midrule
$[N]$ & The set $\{1, 2, \ldots, N\}$ \\
$X^{\inv}$ & Moore--Penrose pseudoinverse of matrix $X$ \\
$\tr(X)$ & Trace of matrix $X$ \\
$\col(X)$ & Column space of matrix $X$ \\
$I_{N}$ & Identity matrix of size $N \times N$ \\
$\calN(\mu, \Sigma)$ & Gaussian distribution with mean $\mu$ and covariance $\Sigma$ \\
$\|x\|_{\Sigma}$ & $\Sigma$-norm: $\sqrt{x^{\top}\Sigma x}$ for positive semidefinite $\Sigma$ \\
$\Delta(\cdot)$ & Set of all probability distributions over a given set \\

\midrule
\multicolumn{2}{l}{\textbf{\acs{sft} and \acs{icl}}} \\
\midrule
$d$ & Dimension of the feature/parameter space \\
$N$ & Number of pretraining samples \\
$x_i \in \bbR^d$ & Pretraining feature vector (prefix representation) \\
$y_i$ & Pretraining response \\
$X \in \bbR^{N \times d}$ & Pretraining design matrix: $X = [x_1,\ldots,x_N]^{\top}$ \\
$Y \in \bbR^{N}$ & Pretraining response vector: $Y = [y_1,\ldots,y_N]^{\top}$ \\
$\theta_i \in \bbR^d$ & Latent task parameter for the $i$-th pretraining sample \\
$\mu_\theta, \Sigma_\theta$ & Mean and covariance of the task prior $\calN(\mu_\theta, \Sigma_\theta)$ \\
$\sigma^2$ & Observation noise variance in pretraining \\
$c^2$ & Input power: $\|x_i\|_{\Sigma_\theta}^2 = c^2$ (Assumption~\ref{assump:norm}) \\
$\bar{\sigma}^2$ & Effective pretraining variance: $\bar{\sigma}^2 = \sigma^2 + c^2$ \\
$\theta^* \in \bbR^d$ & Target task parameter, drawn from $\calN(0, \Sigma^*)$ \\
$\Sigma^*$ & Prior covariance of the target task $\theta^*$ \\
$\tilN$ & Number of personalization samples \\
$\tilx_i \in \bbR^d$ & Personalization feature vector \\
$\tily_i$ & Personalization response \\
$\tilX \in \bbR^{\tilN \times d}$ & Personalization design matrix: $\tilX = [\tilx_1,\ldots,\tilx_{\tilN}]^{\top}$ \\
$\tilY \in \bbR^{\tilN}$ & Personalization response vector: $\tilY = [\tily_1,\ldots,\tily_{\tilN}]^{\top}$ \\
$\tilsigma^2$ & Observation noise variance in personalization data \\
$\theta_{\pre}$ & Pretrained estimator mean (least-squares estimate of $\mu_\theta$) \\
$\Sigma_{\pre}$ & Covariance of the pretrained estimator \\
$\htheta_{\pre}$ & Pretrained model estimator: $\htheta_{\pre} \sim \calN(\theta_{\pre}, \Sigma_{\pre})$ \\
$\theta_{\icl}$ & \ac{icl} posterior mean \\
$\Sigma_{\icl}$ & \ac{icl} posterior covariance \\
$\htheta_{\icl}$ & \ac{icl} estimator: $\htheta_{\icl} \sim \calN(\theta_{\icl}, \Sigma_{\icl})$ \\
$\theta_{\sft}$ & \ac{sft} estimator mean \\
$\Sigma_{\sft}$ & \ac{sft} estimator covariance \\
$\htheta_{\sft}$ & \ac{sft} estimator: $\htheta_{\sft} \sim \calN(\theta_{\sft}, \Sigma_{\sft})$ \\
$\lambda$ & \ac{sft} regularization parameter \\
$\lambda^{*}$ & Optimal \ac{sft} regularization \\
$V = [v_1,\ldots,v_d]$ & Shared orthonormal eigenbasis of $X^{\top}X$, $\tilX^{\top}\tilX$, and $\Sigma^*$ \\
$r$ & Rank of the pretraining design $X$ (pretraining coverage) \\
$\pi_i$ & $i$-th eigenvalue of $X^{\top}X$ for $i \in [r]$ \\
$s_i$ & $i$-th eigenvalue of $\tilX^{\top}\tilX$ for $i \in [d]$ \\
$\tau_i$ & $i$-th eigenvalue of $\Sigma^*$ for $i \in [d]$ \\
$E_{\pre}$ & Mean-squared error of the pretrained model \\
$E_{\icl}$ & Mean-squared error of \ac{icl} \\
$E_{\sft}^{\lambda}$ & Mean-squared error of \ac{sft} with regularization $\lambda$ \\
$\alpha_i$ & \ac{icl} shrinkage coefficient: $\tilsigma^2\pi_i / (\bar{\sigma}^2 s_i + \tilsigma^2\pi_i)$ \\
$\beta_i$ & \ac{sft} shrinkage coefficient: $\bar{\sigma}^2 s_i / (\bar{\sigma}^2 s_i + \lambda\tilsigma^2\pi_i)$ \\
$\pi, s, \tau, m$ & Common values of $\pi_i$, $s_i$, $\tau_i$, and $|v_i^{\top}\mu_\theta|$ \\
$\zeta$ & Aggregate parameter: $\zeta = \tau + 2\bar{\sigma}^2/\pi + m^2$ \\
$\kappa$ & Signal-to-noise ratio in uncovered subspace: $s\tau / (2\tilsigma^2)$ \\
$R_{\crit}$ & Critical coverage ratio for \ac{sft} to dominate \ac{icl} \\

\midrule
\multicolumn{2}{l}{\textbf{LLM-Serving Congestion Game}} \\
\midrule
$t \in \calT$ & User type, summarizing model and task parameters \\
$T$ & Probability measure over the type space $\calT$ \\
$\calA = \{\icl, \sft\}$ & Set of personalization algorithms \\
$\bar{\calA} = \calA \times [0,\infty)$ & Action set: algorithm and sample size pairs \\
$\bar{a} = (a, \tilN)$ & User action: algorithm $a$ with $\tilN$ samples \\
$R_a$ & Resource consumption per sample for algorithm $a$ \\
$R_{\sft}, R_{\icl}$ & Per-sample resource cost of \ac{sft} and \ac{icl} \\
$E_a(t, \tilN)$ & Learning error of algorithm $a$ for type $t$ with $\tilN$ samples \\
$C(t, \bar{a}, R)$ & Total cost: $E_a(t,\tilN) + R_a\tilN(p + h(R))$ \\
$\alpha$ & Information growth exponent: $s = \tilN^{\alpha}$ \\
$p$ & Unit resource price set by the platform \\
$h(\cdot)$ & Congestion function: maps aggregate demand to waiting cost \\
$\barh(R, p)$ & Total per-unit cost: $p + h(R)$ \\
$R(\pi)$ & Congestion level (aggregate resource demand) under policy $\pi$ \\
$\pi_t$ & Policy (mixed action) of type-$t$ user \\
$\BR(t, R)$ & Best response operator for type $t$ at congestion $R$ \\
$(\pi^*, R^*)$ & Equilibrium: optimal policy and congestion level \\
$H$ & Effective per-unit compute cost: $H = p + h(R)$ \\
$\Phi_a(H)$ & Minimal cost for algorithm $a$ at per-unit cost $H$ \\
$N_a(H)$ & Optimal sample size for algorithm $a$ at per-unit cost $H$ \\
$H_a^*$ & Within-algorithm equilibrium fixed point for algorithm $a$ \\
$\psi(H)$ & Cost difference: $\Phi_{\sft}(H) - \Phi_{\icl}(H)$ \\
$H_{\sep}^*$ & Algorithm separation threshold: root of $\psi(H) = 0$ \\
$t_1, t_2$ & Two user types \\
$q$ & Population share of type $t_1$: $T(t_1) = q$ \\
$H_{\sep}^*(t)$ & Type-specific algorithm separation threshold \\
$e(H), f(H), g(H)$ & Aggregate demand functions for different algorithm regimes \\
$I(p)$ & Platform profit: $(p - 1) \cdot R^*(p)$ \\
$p^*$ & Profit-maximizing price \\
$(p^*, \pi^*, R^*)$ & Stackelberg equilibrium \\
$R_{\icl}^*(p)$ & Equilibrium congestion in the \ac{icl}-only game at price $p$ \\

\end{longtable}

\section{Discussions and Derivations of Statistical Personalization Models in Section~\ref{sec:model}}\label{app:model_details}

\subsection{High Probability Analysis of Covariates}\label{app:random_cov}

Throughout the paper, we assume that the covariates $X$ and $\tilX$ are deterministic. In this section, we show that our assumptions hold with high probability if rows of $X$ and $\tilX$ are i.i.d.\ Gaussian vectors.

\begin{proposition}
Let $V\in\mathbb R^{d\times d}$ be orthogonal. Fix integers $N,\tilde N\ge 1$ and nonnegative scalars
\[
\pi_1\ge \pi_2\ge \cdots \ge \pi_r>0,\text{ and }s_1,\dots,s_d\ge 0.
\]
Define the target Gram matrices as
\[
G_X = V\mathrm{diag}(\pi_1,\dots,\pi_r,0,\dots,0)V^\top,\text{ and }
G_{\tilde X} = V\mathrm{diag}(s_1,\dots,s_d)V^\top.
\]
Set the corresponding population covariances as
\[
\Sigma_X = \frac1NG_X,\qquad \Sigma_{\tilde X}=\frac1{\tilde N}G_{\tilde X}.
\]
Let $X\in\mathbb R^{N\times d}$ have i.i.d.\ rows $x_i\sim\mathcal N(0,\Sigma_X)$ and let $\tilde X\in\mathbb R^{\tilde N\times d}$ have i.i.d.\ rows $\tilde x_i\sim\mathcal N(0,\Sigma_{\tilde X})$, independent of $X$. Then we have the following results.
\begin{itemize}
\item (Unbiasedness) $\mathbb E[X^\top X]=G_X$ and $\mathbb E[\tilde X^\top \tilde X]=G_{\tilde X}$.
\item (High-probability operator-norm control) There exist universal constants $c,C>0$ such that for every $t\ge 0$,
\begin{align*}
P\bigg(\bigl\|X^\top X-G_X\bigr\|_{2}
\le C\|G_X\|_{2}\bigg(\sqrt{\frac{d+t}{N}}+\frac{d+t}{N}\bigg)\bigg)
&\ge 1-2e^{-ct},\\
P\bigg(\bigl\|\tilde X^\top \tilde X-G_{\tilde X}\bigr\|_{2}
\le C\|G_{\tilde X}\|_{2}\bigg(\sqrt{\frac{d+t}{\tilde N}}+\frac{d+t}{\tilde N}\bigg)\bigg)
&\ge 1-2e^{-ct}.
\end{align*}
The norm $\|\cdot\|_{2}$ denotes the operator norm of matrices. In particular, for fixed $d$ and $t= \Theta( \log(1/\delta))$, both deviations are $O\big(\|G\|_{2}\sqrt{d/N}\big)$ with probability at least $1-\delta$.
\item (Rank-deficient target) If $r<d$, then $\Sigma_X$ has rank $r$, and the rows of $X$ lie almost surely in the $r$-dimensional subspace $\mathrm{range}(V_r)$ spanned by the first $r$ columns $V_r$ of $V$.
\end{itemize}
\end{proposition}
Thus, if the Gram matrices $\Sigma_X$ and $\Sigma_{\tilX}$ satisfy Assumptions~\ref{assump:svd} and \ref{assump:homo}, then the corresponding sample Gram matrices constructed from $X$ and $\tilX$ satisfy the same assumptions with high probability. Specifically, when $\Sigma_{\tilX}=I_d$, we have that $s_i=\tilN$ for $i\in[d]$. We prove this proposition below.

\begin{proof}
We prove the three results separately. 

For the first result, we let $D_X=\mathrm{diag}(\pi_1,\dots,\pi_r,0,\dots,0)$ and $D_{\tilde X}=\mathrm{diag}(s_1,\dots,s_d)$ so that
$G_X=VD_XV^\top$ and $G_{\tilde X}=V D_{\tilde X}V^\top$.
Define the square-roots
\[
D_X^{1/2}=\mathrm{diag}(\sqrt{\pi_1},\dots,\sqrt{\pi_r},0,\dots,0),\qquad
D_{\tilde X}^{1/2}=\mathrm{diag}(\sqrt{s_1},\dots,\sqrt{s_d}).
\]
Let $Z\in\mathbb R^{N\times d}$ have i.i.d.\ $\mathcal N(0,1)$ entries. Then the distributional definition
$x_i\sim \mathcal N(0,\Sigma_X)$ is equivalent to the row-wise representation
\[
X = \frac{1}{\sqrt N}ZD_X^{1/2}V^\top.
\]
Similarly, with $\tilde Z\in\mathbb R^{\tilde N\times d}$ i.i.d.\ standard normal,
\[
\tilde X = \frac{1}{\sqrt{\tilde N}}\tilde ZD_{\tilde X}^{1/2}V^\top.
\]

Note that $X^\top X=\sum_{i=1}^N x_i x_i^\top$. Since the $x_i$ are i.i.d.\ with $\mathbb E[x_i x_i^\top]=\Sigma_X$,
\[
\mathbb E[X^\top X]=\sum_{i=1}^N \mathbb E[x_i x_i^\top]=N\Sigma_X=G_X.
\]
The same argument gives $\mathbb E[\tilde X^\top \tilde X]=\tilde N\Sigma_{\tilde X}=G_{\tilde X}$.

For the second result, we have that
\[
X^\top X - G_X
= VD_X^{1/2}\left(\frac1N Z^\top Z - I_d\right)D_X^{1/2}V^\top.
\]
Taking operator norms and using $\|VAV^\top\|_{2}=\|A\|_{2}$ for orthogonal $V$,
\begin{align*}
\|X^\top X-G_X\|_{2}=\left\|D_X^{1/2}\left(\frac1N Z^\top Z - I_d\right)D_X^{1/2}\right\|_{2}\le \|D_X^{1/2}\|_{2}^2\left\|\frac1N Z^\top Z - I_d\right\|_{2}= \|G_X\|_{2}\left\|\frac1N Z^\top Z - I_d\right\|_{2}.
\end{align*}

The Bernstein inequality for Gaussian random matrices states that there exist universal constants $c,C>0$ such that for all $t\ge 0$,
\[
P\bigg(\bigg\|\frac1N Z^\top Z - I_d\bigg\|_{2}
\le C\bigg(\sqrt{\frac{d+t}{N}}+\frac{d+t}{N}\bigg)\bigg)\ge 1-2e^{-ct}.
\]
Combining this with the inequality of $\|X^\top X-G_X\|_{2}$ yields
\[
P\left(\|X^\top X-G_X\|_{2}
\le C\|G_X\|_{2}\left(\sqrt{\frac{d+t}{N}}+\frac{d+t}{N}\right)\right)\ge 1-2e^{-ct}.
\]
The same argument gives the bound for $\tilde X$.

For the third result, if $r<d$, then $D_X$ has exactly $r$ positive diagonal entries, hence $\mathrm{rank}(D_X)=r$ and
$\mathrm{rank}(\Sigma_X)=\mathrm{rank}(G_X)=r$. Moreover, from the explicit representation
$x_i = \frac1{\sqrt N} V D_X^{1/2} z_i$, we have $D_X^{1/2}z_i\in\mathbb R^d$ supported on the first $r$ coordinates,
so $x_i\in \mathrm{range}(V_r)$ almost surely, where $V_r$ consists of the first $r$ columns of $V$.
This proves the last claim.
\end{proof}

\subsection{Covariance Calculation of Pretraining}\label{app:cov_pre}
We first calculate the covariance of $Y$. In fact, we have that
\begin{align*}
&\Var(y_i )
= \Var(x_i^\top \theta_i + \epsilon_i )
= \Var(x_i^\top \theta_i) + \Var(\epsilon_i)= \|x_i\|_{\Sigma_\theta}^2 + \sigma_\epsilon^2,\text{ and }
& \Cov(y_i,y_j ) = 0 \text{ for }i\neq j.
\end{align*}
Thus, we have that
\begin{align*}
\Cov(Y)
= \Omega
= \sigma_\epsilon^2 I_N + \operatorname{diag}\big(\|x_1\|_{\Sigma_\theta}^2,\dots,\|x_N\|_{\Sigma_\theta}^2\big). \end{align*}
Now the covariance of $\theta_\pre$ is
\begin{align*}
\Cov(\theta_\pre )
= \Cov\big((X^\top X)^\dagger X^\top Y \big) 
= (X^\top X)^\dagger X^\top \Cov(Y) X \big((X^\top X)^\dagger X^\top\big)^\top 
= \Sigma_\pre.
\end{align*}

\subsection{ICL as a Bayesian Update of the Prior}\label{app:icl_bayes}
\begin{proposition}\label{prop:icl}
    Given the prior $\theta\sim \calN(\mu,\Sigma)$ and the data $(\tilY,\tilX)$ generated $\theta^{*}$ from  $\calN(\mu,\Sigma)$ according to \eqref{eq:task_data}, the posterior of $\theta$ is $\calN(\theta_{\icl},\Sigma_{\icl})$, which is defined as
    \begin{align*}
        \theta_{\icl} &= \mu + \Sigma\tilX^{\top} (\tilX \Sigma \tilX^{\top}+\tilsigma^2 I_{\tilN})^{-1}(\tilY-\tilX\mu),\nonumber\\
        \Sigma_{\icl} & = \Sigma - \Sigma\tilX^{\top} (\tilX \Sigma \tilX^{\top}+\tilsigma^2 I_{\tilN})^{-1}\tilX\Sigma.
\end{align*}
\end{proposition}
\begin{proof}
We prove this via direct calculations. Define $u = \theta - \mu$ and $v = \tilY-\tilX\mu$. The projection matrix of $\col(\Sigma)$ is $\Pi=\Sigma\Sigma^{\inv}$. Then we have that
\begin{align*}
    p(u)\propto \exp\bigg(-\frac{1}{2}u^{\top}\Sigma^{\inv}u\bigg)\text{ for }u\in\col(\Sigma), \quad p(v\given u)\propto \exp\bigg(-\frac{1}{2\tilde{\sigma}^2}\| v-\tilX u\|^2\bigg).
\end{align*}
Thus, we have that
\begin{align*}
    p(u\given \tilX,\tilY)& \propto \exp\bigg(-\frac{1}{2}\big[u^{\top}(\Sigma^{\inv}+\tilde{\sigma}^{-2}\tilX^{\top}\tilX)u-2\tilde{\sigma}^{-2}v^{\top}\tilX u\big]\bigg)\cdot\bbI\{u\in\col(\Sigma)\}\\
    & = \exp\bigg(-\frac{1}{2}\big[u^{\top}(\Sigma^{\inv}+\tilde{\sigma}^{-2}\Pi\tilX^{\top}\tilX \Pi)u-2\tilde{\sigma}^{-2}v^{\top}\tilX \Pi u\big]\bigg)\cdot\bbI\{u\in\col(\Sigma)\},
\end{align*}
where we use the property that $\Pi u=u$ for $u\in\col(\Sigma)$. For the vector $u\notin\col(\Sigma)$, we define that
\begin{align*}
    p(u\given \tilX,\tilY) = p(\Pi u\given \tilX,\tilY).
\end{align*}
Then we have that 
\begin{align*}
    \Sigma_{\icl} & = (\Sigma^{\inv}+\tilde{\sigma}^{-2}\Pi\tilX^{\top}\tilX \Pi)^{\inv} \\
    \theta_{\icl} & = \mu+ \tilde{\sigma}^{-2}\Sigma_{\icl} \Pi\tilX^{\top}(\tilY-\tilX\mu).
\end{align*}
In the following, we just show that these two results are equal to the results in Proposition~\ref{prop:icl}. For $\theta_{\icl}$, we only need to show that
\begin{align*}
    \Sigma_{\icl}^{\inv}\Sigma\tilX^{\top} (\tilX \Sigma \tilX^{\top}+\tilde{\sigma}^2 I_{\tilN})^{-1}(\tilY-\tilX\mu)= \tilde{\sigma}^{-2}\Pi\tilX^{\top}(\tilY-\tilX\mu).
\end{align*}
Define $A=\tilX \Sigma \tilX^{\top}+\tilde{\sigma}^2 I_{\tilN}$. Then we have that 
\begin{align*}
    &\Sigma_{\icl}^{\inv}\Sigma\tilX^{\top} (\tilX \Sigma \tilX^{\top}+\tilde{\sigma}^2 I_{\tilN})^{-1}(\tilY-\tilX\mu)\\
    &\quad = (\Sigma^{\inv}+\tilde{\sigma}^{-2}\Pi\tilX^{\top}\tilX \Pi)\Sigma\tilX^{\top}A^{-1}v\\
    &\quad = \Pi\tilX^{\top}(A^{-1}+\tilde{\sigma}^{-2}\tilX\Sigma\tilX^{\top}A^{-1})v\\
    &\quad = \tilde{\sigma}^{-2}\Pi\tilX^{\top}(\tilY-\tilX\mu),
\end{align*}
where the second equality uses the facts that $\Pi=\Sigma\Sigma^{\inv}$ and $\Pi\Sigma=\Sigma$. To prove the results of $\Sigma_{\icl}$, we define that
\begin{align*}
    S = \Sigma - \Sigma\tilX^{\top}A^{-1}\tilX\Sigma, A=\tilX \Sigma \tilX^{\top}+\tilde{\sigma}^2 I_{\tilN}.
\end{align*}
Then we only need to prove that
\begin{align*}
    (\Sigma^{\inv}+\tilde{\sigma}^{-2}\Pi\tilX^{\top}\tilX \Pi)S=\Pi.
\end{align*}
In fact, we have that
\begin{align*}
    (\Sigma^{\inv}+\tilde{\sigma}^{-2}\Pi\tilX^{\top}\tilX \Pi)S = \Pi + \Pi\tilX^{\top}\big[\tilde{\sigma}^{-2}I_{\tilN}-A^{-1}-\tilde{\sigma}^{-2}\tilX\Sigma\tilX^{\top}A^{-1}\big]\tilX\Sigma,
\end{align*}
where we use the facts that $\Pi=\Sigma\Sigma^{\inv}$ and $\Pi\Sigma=\Sigma$. It is easy to see that
\begin{align*}
    \tilde{\sigma}^{-2}I_{\tilN}-A^{-1}-\tilde{\sigma}^{-2}\tilX\Sigma\tilX^{\top}A^{-1}=0.
\end{align*}
Thus, we conclude the proof of Proposition~\ref{prop:icl}.
\end{proof}

\subsection{Covariance Calculation of \ac{sft}}\label{app:cov_sft}

We first derive the closed form for the \ac{sft} estimator. The loss is 
\begin{align*}
J(\theta)
=
\frac{1}{\tilsigma^2}\|\tilde Y-\tilde X\theta\|_2^2
+
\lambda \|\theta-\theta_{\pre}\|^2_{\Sigma_{\pre}^\dagger}
=
\frac{1}{\tilsigma^2}\|\tilde Y-\tilde X\theta\|_2^2
+
\lambda (\theta-\theta_{\pre})^\top\Sigma_{\pre}^\dagger(\theta-\theta_{\pre}). 
\end{align*}
Setting $\nabla_\theta J(\theta)=0$ gives the normal equations
\begin{align*}
-\frac{2}{\tilsigma^2}\tilde X^\top(\tilde Y-\tilde X\theta)
+ 2\lambda \Sigma_{\pre}^\dagger(\theta-\theta_{\pre}) = 0. 
\end{align*}
Define
\begin{align*}
A = \frac{1}{\tilsigma^2}\tilde X^\top\tilde X+\lambda \Sigma_{\pre}^\dagger,
\qquad
b = \frac{1}{\tilsigma^2}\tilde X^\top\tilde Y + \lambda \Sigma_{\pre}^\dagger\theta_{\pre}.
\end{align*}
Using the Moore–Penrose pseudoinverse, the SFT solution is $\theta_{\sft} = A^\dagger b.$  Then we calculate the covariance of it as
\begin{align*}
\Cov(\theta_{\sft})
&= A^\dagger\Cov\Big(\lambda \Sigma_{\pre}^\dagger\theta_{\pre}
+ \frac{1}{\tilsigma^2}\tilde X^\top\tilde\epsilon\Big)A^\dagger 
= A^\dagger\left(\lambda^2\Sigma_{\pre}^\dagger
+ \frac{1}{\tilsigma^2}\tilde X^\top\tilde X\right)A^\dagger.
\end{align*}

\section{Proof of Theorems in Section~\ref{sec:stat_analysis}}\label{app:thm_proof}
\subsection{Proof of Theorem~\ref{thm:err_analysis}}\label{app:error_analysis}
We first state the full version of Theorem~\ref{thm:err_analysis}.

\begin{theorem}[Error Analyses of pretraining, \ac{icl} and \ac{sft}]
   Under Assumptions~\ref{assump:svd} and~\ref{assump:norm}, when \ac{sft} and \ac{icl} use data $(\tilX,\tilY)$ generated according to~\eqref{eq:task_data}, the mean-squared errors of pretraining, \ac{sft}, and \ac{icl} are given as follows.
    \begin{align*}
        E_{\pre} &= \underbrace{\sum_{i=1}^{d}\tau_i}_{\text{task var.}} + \underbrace{\sum_{i=1}^{r} (v_i^{\top}\mu_\theta)^2}_{\text{task mismatch}}+  \underbrace{\sum_{i=1}^{r}\bar{\sigma}^2\pi_i^{-1}}_{\text{algo. var.}}+ \underbrace{\sum_{i=1}^{r}\bar{\sigma}^2\pi_i^{-1}}_{\text{uncertainty}}=\sum_{i=1}^{d}E_{\pre,i} \\
        E_{\icl} & = \sum_{i=1}^{r}\bigg[\alpha_{i}^2\cdot \underbrace{\big(\tau_i+(v_i^{\top}\mu_\theta)^2+\bar{\sigma}^2\pi_i^{-1}\big)}_{\text{pretraining err. w/o uncertainty}}+  \underbrace{\frac{\tilsigma^2\bar{\sigma}^4s_i}{(\bar{\sigma}^2s_i+\tilsigma^2\pi_i)^2}}_{\text{data noise}}+\underbrace{\frac{\tilsigma^2\bar{\sigma}^2}{\bar{\sigma}^2s_i+\tilsigma^2\pi_i}}_{\text{\ac{icl} uncertainty}}\bigg] +\underbrace{\sum_{i=r+1}^{d}\tau_i}_{\text{task var.}}\\
        E_{\sft}^{\lambda} &= \sum_{i=1}^{r}\bigg[ (1-\beta_{i}^{2})\underbrace{\big(\tau_i+(v_i^{\top}\mu_\theta)^2+2\bar{\sigma}^2\pi_i^{-1}\big)}_{\text{pretraining err.}} +\beta_{i}^{2} \big(\underbrace{\tilsigma^2s_{i}^{-1}}_{\text{LS var.}}+\!\!\!\!\underbrace{\tilsigma^2s_{i}^{-1}}_{\text{\ac{sft} uncertainty}}\!\!\!\!\!\! \big)\bigg]+\sum_{i=r+1}^{d} \big(\underbrace{\tilsigma^2s_{i}^{-1}}_{\text{LS var.}}+\!\!\!\underbrace{\tilsigma^2s_{i}^{-1}}_{\text{\ac{sft} uncertainty}}\!\!\!\!\!\!\!\big),
    \end{align*}
    Here $\bar{\sigma}^2 = \sigma^2 + c^2$ denotes the effective variance in pretraining, combining the observation noise variance $\sigma^2$ and the input power $c^2$. The coefficients $\{\alpha_i\}_{i=1}^r$ for \ac{icl} and $\{\beta_i\}_{i=1}^r$ for \ac{sft} are given by
    \begin{align*}
        \alpha_i = \frac{\tilsigma^2\pi_i}{\bar{\sigma}^2s_i+\tilsigma^2\pi_i}, \qquad \beta_i = \frac{\bar{\sigma}^{2}s_i}{\bar{\sigma}^{2}s_i+\lambda \tilsigma^2\pi_i},
        \quad i\in[r].
    \end{align*}
\end{theorem}

In this section, we present the proof of this full version of Theorem~\ref{thm:err_analysis}, i.e., the error analysis of pretraining, \ac{sft}, and \ac{icl}.  We denote the \ac{svd} of $\tilX$ is $\tilX= U\Sigma_{\tilX} V^{\top}$. The singular values $\Sigma_{\tilX}$ satisfy that $\Sigma_{\tilX}^{\top}\Sigma_{\tilX} = S$.

\textbf{Error Analysis of Pretraining.} 
For pretraining, we have that $\htheta_\pre = \theta_\pre+ \omega_\pre $, where $\omega_\pre\sim\calN(0,\Sigma_{\pre})$. The covariance $\Sigma_\pre$ is calculated as
\begin{align}
    \Sigma_\pre  =  (\sigma^2+c^2)(X^\top X)^{\inv}
X^\top X (X^\top X)^{\inv} = (\sigma^2+c^2)(X^\top X)^{\inv}.\label{eq:pre_sigma}
\end{align}
In the following, we define $\bar{\sigma}^2 = \sigma^2+c^2$.
Thus, we have that
\begin{align}
    \bbE\big[\|\htheta_\pre-\theta^*\|^2\big] = \bbE\big[\|\theta_\pre-\theta^*\|^2\big]+\tr(\Sigma_\pre)=\bbE\big[\|\theta_\pre-\theta^*\|^2\big]+\bar{\sigma}^2\sum_{i=1}^{r}\pi_i^{-1}.\label{eq:pre_err}
\end{align}
For the first term in the right-hand side of \eqref{eq:pre_err}, we adopt Proposition~\ref{prop:uni_err} with $m_\pre=\theta_\pre$ and $K_\pre=0$ to derive
 \begin{align*}
    \bbE\big[\|\theta_\pre-\theta^{*}\|^2\big]= \tr\Big(\Sigma^*+\Sigma_\pre+\mu_m\mu_m^\top\Big),
\end{align*}
where we utilize the fact that $\Sigma_m=\Cov(\theta_\pre)=\Sigma_\pre$ proved in Appendix~\ref{app:cov_pre}, and the expectation of the pretraining model $\mu_m=\bbE[\theta_\pre]$ is
\begin{align}
    \bbE\big[(X^\top X)^{\inv} X^\top Y\big]=\bbE\big[(X^\top X)^{\inv} \sum_{i=1}^{N}x_i(x_i^{\top}\theta_i+\epsilon_i)\big] = (X^\top X)^{\inv} X^\top X\mu_\theta.\label{eq:mean_pre}
\end{align}
Thus, we have that
\begin{align*}
    \bbE\big[\|\htheta_\pre-\theta^*\|^2\big] = \sum_{i=1}^{d}\tau_i+ 2\bar{\sigma}^2\sum_{i=1}^{r}\pi_i^{-1} + \sum_{i=1}^{r} (v_i^{\top}\mu_\theta)^2
\end{align*}

\textbf{Error Analysis of \ac{sft}.}
For \ac{sft}, we have that $\htheta_\sft = \theta_\sft+ \omega_\sft $, where $\omega_\sft\sim\calN(0,\Sigma_{\sft})$.  Define $\Pi=\mathrm{diag}(\pi_1,\cdots,\pi_r,0,\cdots,0), S= \mathrm{diag}(s_1,\cdots,s_d)$. The covariance $\Sigma_\sft$ is calculated as
\begin{align*}
    &\bigg(\frac{1}{\tilde{\sigma}^2}\tilX^{\top}\tilX+\lambda \Sigma_{\pre}^{\inv}\bigg)^{\inv}\cdot\bigg(\lambda^2\Sigma_{\pre}^\dagger
+\frac{1}{\tilde{\sigma}^2}\tilX^{\top}\tilX\bigg)\cdot\bigg(\frac{1}{\tilde{\sigma}^2}\tilX^{\top}\tilX+\lambda \Sigma_{\pre}^{\inv} \bigg)^{\inv}\\
& = \big(\tilde{\sigma}^{-2} VSV^{\top}+\lambda\bar{\sigma}^{-2} V\Pi V^{\top}\big)^{\inv}\cdot \big(\tilde{\sigma}^{-2} VSV^{\top}+\lambda^2\bar{\sigma}^{-2} V\Pi V^{\top}\big)\cdot\big(\tilde{\sigma}^{-2} VSV^{\top}+\lambda\bar{\sigma}^{-2} V\Pi V^{\top}\big)^{\inv}\\
& = V \big(\tilde{\sigma}^{-2} S+\lambda\bar{\sigma}^{-2} \Pi\big)^{\inv}\cdot \big(\tilde{\sigma}^{-2} S+\lambda^2\bar{\sigma}^{-2} \Pi \big)\cdot \big(\tilde{\sigma}^{-2} S+\lambda\bar{\sigma}^{-2} \Pi\big)^{\inv} V^{\top},
\end{align*}
where the first equality results from Assumption~\ref{assump:svd} and \eqref{eq:pre_sigma}. Thus, we have that
\begin{align}
    \tr(\Sigma_\sft) = \sum_{i=1}^{r}\frac{\tilde{\sigma}^{-2}s_i+\lambda^2 \bar{\sigma}^{-2}\pi_i}{(\tilde{\sigma}^{-2}s_i+\lambda \bar{\sigma}^{-2}\pi_i)^2}+\sum_{i=r+1}^{d} \tilde{\sigma}^{2}s_i^{-1}.\label{eq:sft_uncertain}
\end{align}
Thus, we have that
\begin{align}
    \bbE\big[\|\htheta_\sft-\theta^*\|^2\big] = \bbE\big[\|\theta_\sft-\theta^*\|^2\big]+\sum_{i=1}^{r}\frac{\tilde{\sigma}^{-2}s_i+\lambda^2 \bar{\sigma}^{-2}\pi_i}{(\tilde{\sigma}^{-2}s_i+\lambda \bar{\sigma}^{-2}\pi_i)^2}+\sum_{i=r+1}^{d} \tilde{\sigma}^{2}s_i^{-1}.\label{eq:sft_err}
\end{align}
For the first term of the right-hand side of \eqref{eq:sft_err}, we adopt Proposition~\ref{prop:uni_err} with $m_\sft=\theta_\pre$ and $K_\sft=\tilde{\sigma}^{-2}(\tilde{\sigma}^{-2}\tilX^{\top}\tilX+\lambda \Sigma_{\pre}^{\inv} )^{\inv}\tilX^{\top}$. Then we have that
\begin{align*}
    \bbE\big[\|\theta_{\sft}-\theta^{*}\|^2\big]= \tr\Big((I-K_{\sft}\tilX)\big(\Sigma^*+\Sigma_m+\mu_m\mu_m^\top\big)(I-K_{\sft}\tilX)^\top\Big)
+\tilde{\sigma}^2\tr(K_{\sft}^\top K_{\sft}).
\end{align*}
In the following, we calculate the value of each term. For the last term, we have that
\begin{align*}
    \tr(K_\sft^{\top}K_\sft)= \tilde{\sigma}^{-4}\tr\big((\tilde{\sigma}^{-2}S+\lambda \bar{\sigma}^{-2}\Pi)^{-2}S\big)=\sum_{i=1}^{r}\frac{\tilde{\sigma}^{-4} s_i}{(\tilde{\sigma}^{-2}s_i+\lambda \bar{\sigma}^{-2}\pi_i)^2}  +\sum_{i=r+1}^{d}s_i^{-1}.
\end{align*}
For the first three terms, we have that
\begin{align*}
    K_{\sft}\tilde{X} = \tilde{\sigma}^{-2}V(\tilde{\sigma}^{-2}S+\lambda \bar{\sigma}^{-2}\Pi)^{-1}SV^{\top}.
\end{align*}
Thus, we have that
\begin{align*}
    \tr\big((I-K_{\sft}X)\Sigma^{*}(I-K_{\sft}X)\big)&=\sum_{i=1}^{r} \bigg(\frac{\lambda \bar{\sigma}^{-2}\pi_i}{\tilde{\sigma}^{-2}s_i+\lambda \bar{\sigma}^{-2}\pi_i}\bigg)^2\tau_i\\
    \tr\big((I-K_{\sft}X)\Sigma_{m}(I-K_{\sft}X)\big)&=\sum_{i=1}^{r} \bigg(\frac{\lambda \bar{\sigma}^{-2}\pi_i}{\tilde{\sigma}^{-2}s_i+\lambda \bar{\sigma}^{-2}\pi_i}\bigg)^2\bar{\sigma}^2\pi_i^{-1}\\
    \tr\big((I-K_{\sft}X)\mu_{m}\mu_{m}^{\top}(I-K_{\sft}X)\big)&=\sum_{i=1}^{r} \bigg(\frac{\lambda \bar{\sigma}^{-2}\pi_i}{\tilde{\sigma}^{-2}s_i+\lambda \bar{\sigma}^{-2}\pi_i}\bigg)^2(v_i\mu_\theta)^2.
\end{align*}
Thus, we have that
\begin{align*}
    \bbE\big[\|\htheta_\sft-\theta^*\|^2\big] = \sum_{i=1}^{r} \frac{2\tilde{\sigma}^2\bar{\sigma}^4s_i + \lambda^2 \tilde{\sigma}^4\pi_i^{2}\big(\tau_i+ 2\bar{\sigma}^2\pi_i^{-1}+(v_i\mu_\theta)^2\big)}{(\bar{\sigma}^2 s_i + \lambda \tilde{\sigma}^2\pi_i)^2}+2\sum_{i=r+1}^{d} \tilde{\sigma}^2s_{i}^{-1}
\end{align*}
Substituting the coefficients $\{\beta_i\}_{i=1}^{r}$ from Theorem~\ref{thm:err_analysis} into the expression for $E_{\sft}^{\lambda}$ yields the same expression as above.

\textbf{Error Analysis of \ac{icl}.} 
For \ac{icl}, we have that $\htheta_\icl = \theta_\icl+ \omega_\icl $, where $\omega_\icl\sim\calN(0,\Sigma_{\icl})$. The covariance $\Sigma_\icl$ is calculated as
\begin{align}
    &\Sigma_\pre - \Sigma_\pre\tilX^{\top} (\tilX \Sigma_\pre \tilX^{\top}+\tilde{\sigma}^2 I_{\tilN})^{-1}\tilX\Sigma_\pre\nonumber\\
    &\quad = \bar{\sigma}^2 V\Pi^{\inv}V^{\top} - \bar{\sigma}^2V\Pi^{\inv}V^{\top} V\Sigma_{\tilX}^{\top} U^{\top}\Big(\bar{\sigma}^2U\Sigma_{\tilX} V^{\top}V\Pi^{\inv}V^{\top} V\Sigma_{\tilX}^{\top} U^{\top}+\tilde{\sigma}^2 I_{\tilN}\Big)^{-1}U\Sigma_{\tilX} V^{\top}\bar{\sigma}^2 V\Pi^{\inv}V^{\top}\nonumber\\
    &\quad = \bar{\sigma}^2 V\Pi^{\inv}V^{\top} - \bar{\sigma}^4 V\Pi^{\inv}\Sigma_{\tilX}^{\top}\Big(\bar{\sigma}^2\Sigma_{\tilX}\Pi^{\inv}\Sigma_{\tilX}^{\top} +\tilde{\sigma}^2 I_{\tilN}\Big)^{-1}\Sigma_{\tilX}\Pi^{\inv}V^{\top},\label{eq:icl_sigma}
\end{align}
where the first equality results from Assumption~\ref{assump:svd}. Thus, we have that
\begin{align*}
    \tr(\Sigma_{\icl}) = \sum_{i=1}^{r}\frac{\tilde{\sigma}^2\bar{\sigma}^2}{\bar{\sigma}^2 s_i+\tilde{\sigma}^2\pi_i}.
\end{align*}
Thus, we have that
\begin{align}
    \bbE\big[\|\htheta_\icl-\theta^*\|^2\big] = \bbE\big[\|\theta_\icl-\theta^*\|^2\big]+\tr(\Sigma_\icl)=\bbE\big[\|\theta_\icl-\theta^*\|^2\big]+\sum_{i=1}^{r}\frac{\tilde{\sigma}^2\bar{\sigma}^2}{\bar{\sigma}^2 s_i+\tilde{\sigma}^2\pi_i}.\label{eq:icl_err}
\end{align}
For the first term of the right-hand side of \eqref{eq:icl_err}, we adopt Proposition~\ref{prop:uni_err} with $m_\icl=\theta_\pre$ and $K_\icl=\Sigma_\pre\tilX^{\top} (\tilX \Sigma_\pre \tilX^{\top}+\tilde{\sigma}^2 I_{\tilN})^{-1}$ and have that
\begin{align}
    \bbE\big[\|\theta_\icl-\theta^{*}\|^2\big]= \tr\Big((I-K_\icl\tilX)\big(\Sigma^*+\Sigma_m+\mu_m\mu_m^\top\big)(I-K_\icl\tilX)^\top\Big)
+\tilde{\sigma}^2\tr(K_\icl^\top K_\icl)\label{eq:icl_err_decomp}
\end{align}
According to \eqref{eq:mean_pre} and Appendix~\ref{app:cov_pre}, we have that
\begin{align*}
    \mu_m=\bbE[\theta_\pre ]= (X^\top X)^{\inv} X^\top X\mu_\theta, \quad \Sigma_m=\Cov(\theta_\pre)=\Sigma_\pre.
\end{align*}
To calculate the value of the four terms in the right-hand side of \eqref{eq:icl_err_decomp}, we first simplify the $K_\icl$ as
\begin{align*}
    K_\icl &= \bar{\sigma}^2V\Pi^{\inv}V^{\top} V\Sigma_{\tilX}^{\top} U^{\top}\Big(\bar{\sigma}^2U\Sigma_{\tilX} V^{\top}V\Pi^{\inv}V^{\top} V\Sigma_{\tilX}^{\top} U^{\top}+\tilde{\sigma}^2 I_{\tilN}\Big)^{-1}\\
    &= \bar{\sigma}^2V\Pi^{\inv}\Sigma_{\tilX}^{\top}\Big(\bar{\sigma}^2\Sigma_{\tilX} \Pi^{\inv}\Sigma_{\tilX}^{\top} +\tilde{\sigma}^2 I_{\tilN}\Big)^{-1}U^{\top}
\end{align*}
Then we have that
\begin{align}
    \tr(K_\icl^{\top}K_\icl) = \sum_{i=1}^{r}\frac{\bar{\sigma}^4 s_i \pi_i^{-2}}{(\tilde{\sigma}^2+\bar{\sigma}^2 s_i\pi_i^{-1})^2}=\sum_{i=1}^{r}\frac{\bar{\sigma}^4 s_i }{(\tilde{\sigma}^2\pi_i+\bar{\sigma}^2 s_i)^2}.\label{eq:icl_KK}
\end{align}
The value $K_\icl\tilX$ can be simplified as
\begin{align*}
    K_\icl\tilX &= \bar{\sigma}^2V\Pi^{\inv}\Sigma_{\tilX}^{\top}\Big(\bar{\sigma}^2\Sigma_{\tilX} \Pi^{\inv}\Sigma_{\tilX}^{\top} +\tilde{\sigma}^2 I_{\tilN}\Big)^{-1}U^{\top}U\Sigma_{\tilX} V^{\top}\nonumber\\
    & = \bar{\sigma}^2V\Pi^{\inv}\Sigma_{\tilX}^{\top}\Big(\bar{\sigma}^2\Sigma_{\tilX} \Pi^{\inv}\Sigma_{\tilX}^{\top} +\tilde{\sigma}^2 I_{\tilN}\Big)^{-1}\Sigma_{\tilX} V^{\top}.
\end{align*}
Then we have that
\begin{align}
    \tr\Big((I-K_\icl\tilX)\Sigma^*(I-K_\icl\tilX)^\top\Big)&=\sum_{i=1}^{r}\bigg(1- \frac{\bar{\sigma}^2\pi_i^{-1}s_i}{\bar{\sigma}^2\pi_i^{-1}s_i+\tilde{\sigma}^2}\bigg)^2\tau_i +\sum_{i=r+1}^{d}\tau_i \nonumber\\
    &= \sum_{i=1}^{r}\bigg(\frac{\tilde{\sigma}^2 \pi_i}{\bar{\sigma}^2s_i+\tilde{\sigma}^2\pi_i}\bigg)^2\tau_i +\sum_{i=r+1}^{d}\tau_i, \label{eq:icl_err_1} \\
    \tr\Big((I-K_\icl\tilX)\Sigma_m(I-K_\icl\tilX)^\top\Big)& =\sum_{i=1}^{r}\bigg(1- \frac{\bar{\sigma}^2\pi_i^{-1}s_i}{\bar{\sigma}^2\pi_i^{-1}s_i+\tilde{\sigma}^2}\bigg)^2\bar{\sigma}^2\pi_i^{-1} \nonumber\\
    &=  \sum_{i=1}^{r}\bigg(\frac{\tilde{\sigma}^2 \pi_i}{\bar{\sigma}^2s_i+\tilde{\sigma}^2\pi_i}\bigg)^2\bar{\sigma}^2\pi_i^{-1}\label{eq:icl_err_2}\\
    \tr\Big((I-K_\icl\tilX)\mu_m\mu_m^\top(I-K_\icl\tilX)^\top\Big)& =\sum_{i=1}^{r}\bigg(\frac{\tilde{\sigma}^2 \pi_i}{\bar{\sigma}^2s_i+\tilde{\sigma}^2\pi_i}\bigg)^2(v_i^\top \mu_\theta)^2\label{eq:icl_err_3}.
\end{align}
Combining \eqref{eq:icl_err_decomp}, \eqref{eq:icl_KK}, \eqref{eq:icl_err_1}, \eqref{eq:icl_err_2}, \eqref{eq:icl_err_3}, we have that
\begin{align*}
    \bbE\big[\|\theta_\icl-\theta^{*}\|^2\big]= \sum_{i=1}^{r} \frac{\tilde{\sigma}^4 \pi_i^2\big(\tau_i+\bar{\sigma}^2\pi_i^{-1}+(v_i^\top \mu_\theta)^2\big)+\tilde{\sigma}^2\bar{\sigma}^4s_i}{(\bar{\sigma}^2s_i+\tilde{\sigma}^2\pi_i)^2}+\sum_{i=r+1}^{d}\tau_i.
\end{align*}
Thus, we have that
\begin{align*}
    \bbE\big[\|\htheta_\icl-\theta^*\|^2\big] = \sum_{i=1}^{r} \frac{\tilde{\sigma}^4 \pi_i^2\big(\tau_i+\bar{\sigma}^2\pi_i^{-1}+(v_i^\top \mu_\theta)^2\big)+\tilde{\sigma}^2\bar{\sigma}^4s_i+\tilde{\sigma}^2\bar{\sigma}^2(\bar{\sigma}^2s_i+\tilde{\sigma}^2\pi_i)}{(\bar{\sigma}^2s_i+\tilde{\sigma}^2\pi_i)^2}+\sum_{i=r+1}^{d}\tau_i.
\end{align*}
Substituting the coefficients $\{\alpha_i\}_{i=1}^{r}$ from Theorem~\ref{thm:err_analysis} into the expression for $E_{\icl}$ yields the same expression as above.

\section{Proof of Theorems in Section~\ref{sec:mfcg}}\label{app:mfcg_proof}
\subsection{Proof of Theorem~\ref{thm:exist_unique}}\label{app:exist}
we first prove the existence of the equilibrium and then prove the uniqueness of the congestion level at equilibirum.

\textbf{Proof of the Existence of Equilibrium}

To prove the existence of equilibrium, we first show that there exists $R^{*}$ satisfying some fixed-point equation. Then we show that a equilibrium can be constructed from $R^{*}$. We define the fixed-point equation of $R^{*}$ as follows. For a given type $t$ and congestion level $R$, we first define the congestion level of the optimal actions and the corresponding convex hull as
\begin{align*}
    D(t,R)=\big\{R_a\tilN\given (a,\tilN)\in\BR(t,R)\big\}, \text{ and }\hatD(t,R)=\mathrm{co}\,D(t,R), \text{ respectively.}
\end{align*}
We call $\hatD(t,R)$ the \emph{attainable demand} for type $t$ at the congestion level $R$. Then the aggregate attainable demand correspondence is the Aumann integral
\begin{equation}\label{eq:Aumann}
\barD(R)=\int_{\calT}\hatD(t,R)\,T(\rmd t)
=\bigg\{\int_{\calT} d(t)\,T(\rmd t)\,\bigg |\, d(t)\ \text{measurable and } d(t)\in \hatD(t,R)\ \text{for a.e.\ }t\bigg\}.
\end{equation}
Our proofs proceed in five steps.
\begin{itemize}
    \item Checking the plausibility of $\barD(R)$ definition.
    \item Restricting the domain of the fixed-point of $\barD(R)$.
    \item Building the properties of $\barD(R)$.
    \item Applying the fixed-point theorem to $\barD(R)$.
    \item Constructing equilibrium from the fixed point of $\barD(R)$.
\end{itemize}

\textbf{Step 1: Checking the plausibility of $\barD(R)$ definition.}

We will first show that $\barD(R)$ is well-defined, i.e., $\barD(R)$ is not an empty set. We achieve this by proving that
\begin{itemize}
    \item The set $\hatD(t,R)$ is non-empty for any $t$ and $R$.
    \item We can find a measurable selection $d(t)\in\hatD(t,R)$, i.e., $\barD(R)$ is non-empty.
\end{itemize}

For the non-emptiness of $\hatD(t,R)$, we will show that $\BR(t,R)$ is non-empty and compact. Fix $t\in\calT$, $R\in[0,\barR]$, and any $\tilN_0>0$. For each $a\in\calA$ there exists $K_a(t,R)>0$
such that
\begin{align*}
E_a(t,\tilN)+R_a\tilN \barh(R,p)\ \ge\ E_a(t,\tilN_0)+R_a\tilN_0 \barh(R,p)+1
\quad\text{for all }\tilN\ge K_a(t,R).
\end{align*}
Here $K_a(t,R)$ can be explicitly constructed as
\begin{align*}
    K_a(t,R) = \frac{E_a(t,\tilN_0)+1}{R_a \cdot\big(h(0)+p\big)}+\tilN_0.
\end{align*}
For $a=\sft$, there also exists $\varepsilon(t,R)>0$ such that
\begin{align*}
E_{\sft}(t,\tilN)+R_{\sft}\tilN \barh(R,p)\ \ge\ E_{\sft}(t,\tilN_0)+R_{\sft}\tilN_0 \barh(R,p)+1
\quad\text{for all }\tilN\in(0,\varepsilon(t,R)].
\end{align*}
Here $\varepsilon(t,R)$ can be explicitly constructed as
\begin{align*}
    \varepsilon(t,R) = \bigg(\frac{2(d-r)\tilde{\sigma}^2}{E_{\sft}(t,\tilN_0)+R_{\sft}\tilN_0 \barh(R,p)+1}\bigg)^{1/\alpha}.
\end{align*}
Therefore, minimizing over $\bar{\calA}$ is equivalent to minimizing over the compact set
\begin{align*}
\bar{\calA}(t,R)=\Big\{(\sft,\tilN)\given\tilN\in\big[\varepsilon(t,R),K_{\sft}(t,R)\big]\Big\}\cup
\Big\{(\icl,\tilN)\given\tilN\in\big[0,K_{\icl}(t,R)\big]\Big\}.
\end{align*}

Note that the map $(a,\tilN)\mapsto C(t,(a,\tilN),R)$
is continuous on $\bar{\calA}(t,R)$ (it is continuous in $\tilN$ on each branch $a$, and $a$ ranges over a finite set), and that a continuous function on a compact set attains its minimum.
Thus, $\BR(t,R)$ is nonempty. Moreover, since $\BR(t,R)$ is the set of minimizers of a continuous function over a compact set, it is compact. 

Then we show that we can find a measurable selection $d(t)\in\hatD(t,R)$. We will prove that the maps $t\mapsto \inf \hatD(t,R)$, and $t\mapsto \sup \hatD(t,R)$ are measurable for each fixed $R$. Since $\hatD(t,R)=\mathrm{co}\,D(t,R)$ and that
\begin{align*}
\inf(\mathrm{co}\,D(t,R))=\inf D(t,R),
\qquad
\sup(\mathrm{co}\,D(t,R))=\sup D(t,R),
\end{align*}
it suffices to show $t\mapsto \inf D(t,R)$ is measurable. Since the map $(t,\tilN)\mapsto E_a(t,\tilN)$ is jointly measurable for each $a$,
$(t,a,\tilN)\mapsto C(t,(a,\tilN),R)$ is jointly measurable for any $R$. Define the minimum value 
\begin{align*}
    m(t,R) = \inf_{(a,\tilN)\in\calA}C(t,(a,\tilN),R).
\end{align*}
Because $C$ is jointly measurable and $\bar{\calA}$ is a Borel set, the map $t\mapsto m(t,R)$ is measurable for any $R$. Thus, the graph $\mathrm{Gr}(R)=\{(t,a,\tilN): C(t,(a,\tilN),R)=m(t,R)\}$ for any $R$. This results in the measurability of $\inf D(t,R)$ by noting that
\begin{align*}
\inf D(t,R)=\inf\{R_a\tilN \given (t,a,\tilN)\in \mathrm{Gr}(R)\}.
\end{align*}
The proof for $t\mapsto \sup D(t,R)$ is identical. Therefore, the extremal selections $t\mapsto \inf \hatD(t,R)$ and $t\mapsto \sup \hatD(t,R)$ are measurable. Thus, $\barD(R)$ is non-empty.

\textbf{Step 2: Restricting the domain of the fixed-point of $\barD(R)$.}

We first prove the following lemma.
\begin{lemma}[Uniform bound on best-response demand]\label{lem:integrable}
There exists a number $\barR<\infty$
such that for all $t\in\calT$, all $R\in[0,\infty)$, and all $(a,\tilN)\in \BR(t,R)$,
\begin{align*}
R_a\tilN\le \barR.
\end{align*}
\end{lemma}
\begin{proof}
    We define
    \begin{align*}
        M(t)= \sup_{R\in[0,\infty)} S(t,R),\text{ where }S(t,R) =\sup_{(a,\tilN)\in \BR(t,R)}R_a\tilN.
    \end{align*}
    According to Proposition~\ref{prop:cvx_func}, it is easy to verify that $E_{a}(t,\tilN)$ is strictly convex and decreasing in $\tilN$ for any $a\in\calA$ and $t\in\calT$. Thus, for any $a\in\calA$ and $t\in\calT$, the minimizer $\tilN^{*}(a,t)$ of $C(t,(a,\tilN),R)$ is unique and decreasing with increasing $R$. Thus, we have that
    \begin{align*}
        M(t)= \sup_{R\in[0,\infty)} S(t,R)\leq S(t,0).
    \end{align*}
    Since the minimizers are unique, and $C(t,(a,\tilN),R)$ are continuous in $t$, $S(t,0)$ has its finite maximal value $\barR$ on $\calT$, due to the compactness of $\calT$. Thus, $M(t)\leq \barR$.
\end{proof}

Lemma~\ref{lem:integrable} implies that for any $R\in[0,\infty)$ and any measurable selection $d(t)\in\hatD(t,R)$,
we have $0\le d(t)\le \barR$ a.e., hence
\begin{align*}
0\le \int_{\calT} d(t)\,T(\rmd t)\le \barR<\infty.
\end{align*}
Then for all $R\in[0,\infty)$,
\begin{equation}\label{eq:range_inclusion}
\barD(R)\subseteq [0,\barR].
\end{equation}
Thus, it suffices to find a fixed point in the compact convex set $K=[0,\barR]$.

\textbf{Step 3: Building the properties of $\barD(R)$.}

In this step, we will build three properties of $\barD(R)$: convexity, compactness, and upper hemicontinuity.

We start with convexity. Convexity of $\barD(R)$ follows directly: if $x_i=\int d_i(t)\,T(\rmd t)\in\barD(R)$ for $i=1,2$, and $\lambda\in[0,1]$,
then $\lambda d_1(t)+(1-\lambda)d_2(t)\in \hatD(t,R)$ because $\hatD(t,R)$ is convex for each $t$.
Hence
\begin{align*}
\lambda x_1+(1-\lambda)x_2=\int\big(\lambda d_1(t)+(1-\lambda)d_2(t)\big)\,T(\rmd t)\in \barD(R).
\end{align*}

For compactness, we show that $\barD(R)$ is closed and bounded in $\mathbb R$. Boundedness follows from \eqref{eq:range_inclusion}.
For closedness, let $\{x_n\}\subseteq \barD(R)$ with $x_n\to x$.
For each $n$ pick a measurable selection $d_n(t)\in \hatD(t,R)$ such that $x_n=\int_{\calT} d_n(t)\,T(\rmd t)$.
By Lemma~\ref{lem:integrable}, $0\le d_n(t)\le \barR$ a.e.\ for all $n$, so $\{d_n\}$ is uniformly integrable in $L^1(T)$. By the Komlos' theorem, there exists a subsequence (not relabeled) and a function $d\in L^1(T)$ such that $\bar d_N=\frac{1}{N}\sum_{n=1}^N d_n$ converges a.e.\ to $d$ and in $L^1$.
Because each $\hatD(t,R)$ is convex and closed, $\bar d_N(t)\in \hatD(t,R)$ for each $N$ and a.e.\ $t$,
and taking limits preserves membership: $d(t)\in \hatD(t,R)$ a.e.
Then by $L^1$ convergence,
\begin{align*}
\int_{\calT} d(t)\,T(\rmd t)=\lim_{N\to\infty}\int_{\calT} \bar d_N(t)\,T(\rmd t)=\lim_{N\to\infty}\frac{1}{N}\sum_{n=1}^N x_n=x.
\end{align*}
Hence $x\in \barD(R)$, proving closedness.

For the upper hemicontinuity, it suffices (in $\mathbb R$) to show that the graph of $R\mapsto \barD(R)$ is closed:
if $R_n\to R$ and $x_n\in \barD(R_n)$ with $x_n\to x$, then $x\in \barD(R)$. Fix such sequences. For each $n$, choose a measurable $d_n(t)\in \hatD(t,R_n)$ such that $x_n=\int_{\calT} d_n(t)\,T(\rmd t)$.
Again $0\le d_n\le M$ a.e., so by the same Komlos argument,
there exists a subsequence and a limit $d\in L^1(T)$ such that $\bar d_N$ converges a.e.\ and in $L^1$ to $d$.
In the following, we only need to show that $d(t)\in \hatD(t,R)$ a.e. and $\int_{\calT} d(t)\,T(\rmd t)=x$.

To prove this claim, note that $\hatD(t,R)$ is the convex hull of demands induced by minimizers of $C(t,\cdot,R)$.
Because $C(t,\cdot,R)$ is continuous in $R$, the objective $C(t,(a,\tilN);R_n)$ converges pointwise to $C(t,(a,\tilN),R)$.
By standard maximum theorem arguments on each compactified feasible set, the set of minimizers varies upper hemicontinuously, implying that any limit point of minimizer-induced demands at $R_n$
belongs to the minimizer-induced demand set at $R$. Convexification and closedness of $\hatD(t,R)$ then imply that any limit point of convex combinations belongs to $\hatD(t,R)$. Since $\bar d_N(t)$ is a convex combination of elements in $\hatD(t,R_n)$
with varying $n$, and $\hatD(t,R)$ is upper hemicontinuous in $R$, the a.e.\ limit $d(t)$ lies in $\hatD(t,R)$. Finally, $L^1$ convergence yields
\begin{align*}
\int_{\calT} d(t)\,T(\rmd t)=\lim_{N\to\infty}\int_{\calT} \bar d_N(t)\,T(\rmd t)=\lim_{N\to\infty}\frac{1}{N}\sum_{n=1}^N x_n=x.
\end{align*}
so $x\in \barD(R)$, proving closedness of the graph and hence upper hemicontinuity.

\textbf{Step 4:  Applying the fixed-point theorem to $\barD(R)$.}

We have shown:
\begin{itemize}
\item $K=[0,\barR]$ is nonempty, compact, convex;
\item for each $R\in K$, $\barD(R)$ is nonempty, compact, convex (Steps 1,3);
\item $\barD(R)\subseteq K$ (Step 2);
\item $\barD(\cdot)$ is upper hemicontinuous (Step 3).
\end{itemize}
Therefore, by Kakutani's fixed point theorem, there exists $R^{*}\in K$ such that
\begin{align*}
R^{*}\in \barD(R^{*}).
\end{align*}
By definition of the Aumann integral \eqref{eq:Aumann}, there exists a measurable selection
$d^{*}(t)\in \hatD(t,R^{*})$ for a.e. $t$ such that
\begin{align*}
R^{*}=\int_{\calT} d^{*}(t)\,T(\rmd t).
\end{align*}

\textbf{Step 5: Constructing equilibrium from the fixed point of $\barD(R)$}

Fix $t$. Since $d^{*}(t)\in \hatD(t,R^{*})=\mathrm{co}\,D(t,R^{*})$, there exist
(best-response) actions $(a,\tilN_a)\in \BR(t,R^{*})$ and weights $\lambda_a\ge 0$ with $\sum_{a\in\calA} \lambda_a=1$
such that
\begin{align*}
d^{*}(t)=\sum_{a\in\calA} \lambda_a R_{a}\tilN_a.
\end{align*}
Define $\pi^{*}_t$ as the probability distribution assigning mass $\lambda_a$ to action $(a,\tilN_a)$.
Then $\pi^{*}_t$ is supported on $\BR(t,R^{*})$. By construction, the expected demand of type $t$ under $\pi^{*}_t$ equals $d^{*}(t)$, hence
\begin{align*}
R(\pi^{*})=\int_{\calT}\bigg(\sum_{a\in\calA}\int_{0}^{\infty} R_a\tilN\,\pi^{*}_t(a,\rmd\tilN))\bigg)\,T(\rmd t)
=\int_{\calT} d^{*}(t)\,T(\rmd t)=R^{*}.
\end{align*}
Thus $(\pi^{*},R^{*})$ satisfies optimality and consistency, proving existence. Thus, we conclude the proof of this theorem.


\textbf{Proof of the Congestion-Level Uniqueness of Equilibrium}

We first prove the following lemma.
\begin{lemma}[Monotonicity of best-response demand]\label{lem:strong_mono}
For each $t\in\calT$ and any $0\le R_1<R_2$, if we define $D(t,R)=\{R_a\tilN\given (a,\tilN)\in\BR(t,R)\},$ then $\sup D(t,R_2) \le \inf D(t,R_1).$
\end{lemma}
\begin{proof}
Fix a type $t$. Define the resource demand variable $r = R_a\tilde N.$
For each $r\ge 0$, we define the minimal statistical error with resource $r$ as
\begin{align*}
g_t(r)
=\min_{a\in\{\mathrm{sft},\mathrm{icl}\}} E_{a}\!\left(t,\frac{r}{R_a}\right).
\end{align*}
Then the type-$t$ optimization problem is equivalent to
\begin{align*}
\min_{(a,\tilN)\in\bar{\calA}}E_{a}(t,\tilde N)+ R_a\tilde N\, \barh(R,p) = \min_{r\ge 0}\ \big\{ g_t(r)+\barh(R,p)\, r \big\}.
\end{align*}
Hence, we have that
\begin{align*}
D(t,R)=\arg\min_{r\ge 0}\big\{ g_t(r)+\barh(R,p)r\big\}.
\end{align*}

Let $R_1<R_2$. Since $h$ is increasing, we have $\barh(R_1,p)<\barh(R_2,p).$
Pick arbitrary
\begin{align*}
r_1\in D(t,R_1), \qquad r_2\in D(t,R_2).
\end{align*}
Optimality of $r_1$ and $r_2$ implies
\begin{align}
g_t(r_1)+\barh(R_1,p)r_1 &\le g_t(r_2)+\barh(R_1,p)r_2, \label{eq:ineq1}\\
g_t(r_2)+\barh(R_2,p)r_2 &\le g_t(r_1)+\barh(R_2,p)r_1. \label{eq:ineq2}
\end{align}
Adding \eqref{eq:ineq1} and \eqref{eq:ineq2} and canceling terms yields
\begin{align*}
\bigl(\barh(R_2,p)-\barh(R_1,p)\bigr)\,(r_2-r_1)\le 0.
\end{align*}
Since $\barh(R_2,p)>\barh(R_1,p)$, it follows that $r_2\le r_1.$

Because the choice of $r_1\in D(t,R_1)$ and $r_2\in D(t,R_2)$ was arbitrary, we have $\sup D(t,R_2)\le \inf D(t,R_1).$ Thus, we conclude the proof of this lemma.
\end{proof}
Then we prove the uniqueness as follows. Assume for contradiction that there exist two equilibrium congestion levels $R_1<R_2$. By the  definition of equilibrium with mixed policies, we must have
\begin{align*}
R_i\in \barD(R_i)=\int_{\calT}\hatD(t,R_i)\,T(\rmd t)\text{ for }i=1,2,
\end{align*}
where $\hatD(t,R)=\mathrm{co}\,D(t,R)$, and the integral is the Aumann integral. The existence proof has proved that $\barD(\cdot)$ is well-defined. According to the definition of the Aumann integral, there exist measurable selections $d_i(t)\in \hatD(t,R_i)$ such that
\begin{align*}
R_i=\int_{\calT} d_i(t)\,T(\rmd t)\text{ for }i=1,2.
\end{align*}

Fix $t$. Because $d_1(t)\in \hatD(t,R_1)$ and $d_2(t)\in \hatD(t,R_2)$, we have
\begin{align*}
d_1(t)\ \ge\ \inf \hatD(t,R_1),\qquad d_2(t)\ \le\ \sup \hatD(t,R_2).
\end{align*}
In addition, note that $\inf \hatD(t,R_1)=\inf D(t,R_1),$ and $\sup \hatD(t,R_2)=\sup D(t,R_2).$ Therefore, $d_2(t)\ \le\ \sup D(t,R_2)\ \le\ \inf D(t,R_1)\ \le\ d_1(t),$ where the middle inequality uses Lemma~\ref{lem:strong_mono}. Thus $d_2(t)\le d_1(t)$ for $T$-a.e.\ $t$. Integrating yields
\begin{align*}
R_2=\int d_2(t)\,T(\rmd t)\ \le\ \int d_1(t)\,T(\rmd t)=R_1,
\end{align*}
contradicting $R_1<R_2$. Hence, there cannot exist two distinct equilibrium congestion levels. We conclude the proof of Theorem~\ref{thm:exist_unique}.

\subsection{Proof of Theorem~\ref{thm:homo}}\label{app:homo}
We will prove our claims one by one.

For the first result, we calculate the closed-form expression of $\Phi_a$ and $N_a$ in the following proposition.
\begin{proposition}\label{prop:opt_alg}
    For \ac{icl} and \ac{sft} with errors in \eqref{eq:simplified}, we have that
    \begin{align*}
        N_{\sft}(H) = \sqrt{\frac{2\tilsigma^2d}{R_{\sft}H}},\text{ and } N_{\icl}(H) = \max\bigg\{0,\sqrt{\frac{2\tilsigma^2 r}{R_{\icl}H}}-\frac{2\tilsigma^2}{\zeta}\bigg\}.
    \end{align*}
    The corresponding minimal costs are
    \begin{align*}
        \Phi_{\sft}(H) & =2\sqrt{2\tilsigma^{2}dR_{\sft}H}, \text{ and }
    \end{align*}
    \begin{equation*}
        \Phi_{\icl}(H) = (d-r)\tau+\left\{
        \begin{array}{ll}
        r\zeta, &  \text{if }R_{\icl}H\geq \frac{r\zeta^2}{2\tilsigma^2}\\
        2\sqrt{2r\tilsigma^2R_{\icl}H}-R_{\icl}\cdot\frac{2\tilsigma^2}{\zeta}\cdot H, & \text{if }R_{\icl}H<\frac{r\zeta^2}{2\tilsigma^2}.
        \end{array}
        \right.
    \end{equation*}

\end{proposition}
The proof is provided in Appendix~\ref{app:opt_alg}. Thus, we have that
\begin{align*}
    R_{\sft}N_{\sft}(H) =  \sqrt{\frac{2R_{\sft}\tilsigma^2d}{H}}>\sqrt{\frac{2R_{\icl}\tilsigma^2r}{H}}>\max\bigg\{0,\sqrt{\frac{2R_{\icl}\tilsigma^2 r}{H}}-\frac{2\tilsigma^2}{\zeta}R_{\icl}\bigg\}=R_{\icl}N_{\icl}(H),
\end{align*}
where the first inequality results from Assumption~\ref{assump:large} and that $d\geq r$.

For the second result, we define the function
\begin{align*}
    f_a(H) = \barh\big(R_a N_a(H),p\big)-H =p +(R_a N_a(H))^2 -H .
\end{align*}
We note that  $f_a(H)$ are strictly decrcreasing in $H$ and $\lim_{H\rightarrow\infty}f_a(H)=-\infty$ for any $a\in\calA$. Since $\lim_{H\downarrow 0}f_a(H)=\infty$ for any $a\in\calA$, both $H_{\icl}^{*}$ and $H_{\sft}^{*}$ exist and are unique. Since $f_a(p)\geq 0$ for all $a$, $H_a^{*}\geq p$. The relationship $H_{\sft}^{*}>H_{\icl}^{*}$ directly follows from our first claim $R_{\sft}N_{\sft}(H)>R_{\icl}N_{\icl}(H)$ for any $H\geq 0$, since they are the solutions of $H_a^* =(R_a\cdot N_a(H_a^*))^2+p$. 

For the third claim, Proposition~\ref{prop:opt_alg} shows that
\begin{align*}
    \psi(H) =2\sqrt{2\tilsigma^{2}dR_{\sft}H} - 2\sqrt{2r\tilsigma^2R_{\icl}H}+R_{\icl}\cdot\frac{2\tilsigma^2}{\zeta}\cdot H - (d-r)\tau \text{ if }H<\frac{r\zeta^2}{2R_{\icl}\tilsigma^2}.
\end{align*}
We have that 
\begin{align*}
    \psi^{\prime}(H) = \frac{\sqrt{2\tilsigma^{2}dR_{\sft}} - \sqrt{2r\tilsigma^2R_{\icl}}}{\sqrt{H}}+R_{\icl}\cdot\frac{2\tilsigma^2}{\zeta}>0.
\end{align*}
With the similar procedure, we can show $\psi(H)$ is also strictly increasing when $R_{\icl}H\geq r\zeta^2/(2\tilsigma^2)$. To show that the root of $\psi(H)$ exists, we note that
\begin{align*}
    \psi(0)<0,\text{ and }\lim_{H\rightarrow\infty}\psi(H)=\infty.
\end{align*}
Thus, $H_{\sep}^{*}>0$ exists. The uniqueness results from the strict monotonicity of $\psi$.

In the following, we prove the last result. For the equilibrium $(\pi^{*},R^{*})$, we note that $\pi^{*}$ is a distribution on $\Delta(\bar{\calA})$, since there is only one user type. We adopt $H^{*}$ to denote the corresponding cost per unit computational resource, i.e., $H^{*}=\barh(R^{*},p)$. We consider three cases as follows.
\begin{itemize}
    \item If $H^{*}<H_{\sep}^{*}$, then $\psi(H^{*})<0$. Thus, all the users will adopt \ac{sft}. In other words, $H^{*}$ must be $H_{\sft}^{*}$, since $H_{\sft}^{*}$ is the fixed point of $H_a^* = \barh\big(R_a\cdot N_a(H_a^*),p\big)$. Then $\pi^*$ adopts $(\sft,N_{\sft}(H^{*}))$ with probability $1$.
    \item If $H^{*}>H_{\sep}^{*}$, the arguments are similar to the previous ones. We have that $H^{*}=H_{\icl}^{*}$ and that $\pi^*$ adopts $(\icl,N_{\icl}(H^{*}))$ with probability $1$.
    \item If $H^{*}=H_{\sep}^{*}$, then the users can adopt either $(\sft,N_{\sft}(H^{*}))$ and $(\icl,N_{\icl}(H^{*}))$. Assume that we adopt $\sft$ with probability $\lambda$. Then we must have
    \begin{align*}
        \sqrt{H^{*}-p} = \lambda R_{\sft}N_{\sft}(H^{*}) + (1-\lambda) R_{\icl}N_{\icl}(H^{*}).
    \end{align*}
    Thus, we have that
    \begin{align*}
        \lambda = \frac{\sqrt{H^{*}-p}-R_{\icl}N_{\icl}(H^{*})}{ R_{\sft}N_{\sft}(H^{*})-R_{\icl}N_{\icl}(H^{*})}.
    \end{align*}
\end{itemize}

Thus, we conclude the proof of Theorem~\ref{thm:homo}.

\subsection{Full Statement and Proof of Theorem~\ref{thm:hetero}}\label{app:hetero}
\begin{theorem}\label{thm:hetero_full}
    Consider the game \eqref{eq:simplified} with $\alpha=1$, $T(t_1)=q$, $T(t_2)=1-q$, $h(x)=x^2$, and $R_{\sft}>R_{\icl}$. When $H_{\sep}^{\ast}(t_2)< H_{\sep}^{\ast}(t_1)$, the equilibrium has a threshold structure:
    \begin{itemize}
        \item[1.] If $H_{\sep}^{*}(t_2)> e(H_{\sep}^{*}(t_2))^2+p$, then $H^*$ at equilibrium is the root of $H=e(H)^2+p$.
        \item[2.] If $f(H_{\sep}^{*}(t_2))^2+p\leq H_{\sep}^{*}(t_2)\leq e(H_{\sep}^{*}(t_2))^2+p $, $H^*=H_{\sep}^{*}(t_2)$.
        \item[3.] If $H_{\sep}^{*}(t_2)<  f(H_{\sep}^{*}(t_2))^2+p$ and $f(H_{\sep}^{*}(t_1))^2+p< H_{\sep}^{*}(t_1)$, then $H^*$ at equilibrium is the root of $H=f(H)^2+p$.
        \item[4.] If $g(H_{\sep}^{*}(t_1))^2+p\leq H_{\sep}^{*}(t_1)\leq f(H_{\sep}^{*}(t_1))^2+p $, $H^*=H_{\sep}^{*}(t_1)$.
        \item[5.] If $H_{\sep}^{*}(t_1)< g(H_{\sep}^{*}(t_1))^2+p$, then $H^*$ at equilibrium is the root of $H=g(H)^2+p$.
    \end{itemize}
    The equilibrium congestion level is $R^*=\sqrt{H^*-p}$, where $H^*$ is determined by cases (1)–(5).
    
    When $H_{\sep}^*(t_2)= H_{\sep}^*(t_1)$, the following holds.
    \begin{itemize}
        \item[1.] If $H_{\sep}^{*}(t_2)> e(H_{\sep}^{*}(t_2))^2+p$, then $H^*$ at equilibrium is the root of $H=e(H)^2+p$.
        \item[2.] If $g(H_{\sep}^{*}(t_2))^2+p\leq H_{\sep}^{*}(t_2)\leq e(H_{\sep}^{*}(t_2))^2+p $, $H^*=H_{\sep}^{*}(t_2)$.
        \item[3.] If $H_{\sep}^{*}(t_2)< g(H_{\sep}^{*}(t_2))^2+p$, then $H^*$ at equilibrium is the root of $H=g(H)^2+p$.
    \end{itemize}
    The equilibrium congestion level is $R^*=\sqrt{H^*-p}$, where $H^*$ is determined by cases (1)–(3).
\end{theorem}
The explicit expressions of $e,f,g$ are
\begin{align*}
    e(H) & = q\sqrt{\frac{2\tilsigma_{1}^2 d_1 R_\sft}{H}} + (1-q)\sqrt{\frac{2\tilsigma_{2}^2 d_2 R_\sft}{H}},\\
    f(H) &= q\sqrt{\frac{2\tilsigma_{1}^2 d_1 R_\sft}{H}} + (1-q)\max\bigg\{0,\sqrt{\frac{2\tilsigma_{2}^2 r_2 R_\icl}{H}}- \frac{2\tilsigma_2^2 R_\icl}{\zeta_2}\bigg\},\\
    g(H) & = q\max\bigg\{0,\sqrt{\frac{2\tilsigma_{1}^2 r_1 R_\icl}{H}}- \frac{2\tilsigma_1^2 R_\icl}{\zeta_1}\bigg\} + (1-q)\max\bigg\{0,\sqrt{\frac{2\tilsigma_{2}^2 r_2 R_\icl}{H}}- \frac{2\tilsigma_2^2R_\icl}{\zeta_2}\bigg\}.
\end{align*}

Before the proof, we first define several quantities. 
\begin{align*}
    J(t,H) &= \big\{R_a\tilN_a(t,H) \given \Phi_a(t,H)\leq \Phi_{a^{\prime}}(t,H)\text{ for all }a^{\prime}\in\calA\big\},\\
    \hatJ(t,H) & = \mathrm{co}\, J(t,H),\\
    \barJ(H) &= \big\{ q\cdot R_1+(1-q)\cdot R_2\given R_{i}\in \hatJ(t_i,H)\text{ for }i=1,2\big\}.
\end{align*}
Intuitively, the set $J(t,H)$ includes all the possible amount of consumed resource of the pure actions given $H$ for type-$t$ users, $\hatJ(t,H)$ is the convex hull of $J(t,H)$, reflecting the randomized policy, and $\barJ(H)$ includes all the possible amount of consumed resource for all the users. Thus, $H^*=h(R^*)+p$ at equilibrium is the root of the following equation.
\begin{align}
    h^{-1}(H^*-p) \in \barJ(H^*),\label{eq:H_fixed_equation_hetero}
\end{align}
where $h^{-1}$ is the inverse function of $h$. In the following, we explicitly calculate $\barJ(H)$. First, we assume that $H_{\sep}^{*}(t_2)< H_{\sep}^{*}(t_1)$. Then we consider $5$ cases.

\textbf{Case 1: $H< H_{\sep}^{*}(t_2)$.}

Since $H_{\sep}^{*}(t_2)< H_{\sep}^{*}(t_1)$, we have that $\psi(t_1,H)<0$ and $\psi(t_2,H)<0$. Thus, all users choose $\ac{sft}$. Define the function
\begin{align*}
    e(H) = q\sqrt{\frac{2\tilsigma_{1}^2 d_1 R_\sft}{H}} + (1-q)\sqrt{\frac{2\tilsigma_{2}^2 d_2 R_\sft}{H}}.
\end{align*}
Under this case, we have that $\barJ(H) = e(H)$. Obviously, $e(H)$ is strictly decreasing in $H$. If $\sqrt{H_{\sep}^{*}(t_2)-p}> e(H_{\sep}^{*}(t_2))$, then the root $H^*$ of \eqref{eq:H_fixed_equation_hetero} is not larger than $H_{\sep}^{*}(t_2)$, which is the unique root of
\begin{align*}
    \sqrt{H-p}=e(H).
\end{align*}

\textbf{Case 2: $H= H_{\sep}^{*}(t_2)$.}

Since $H_{\sep}^{*}(t_2)< H_{\sep}^{*}(t_1)$, we have that $\psi(t_1,H)<0$ and $\psi(t_2,H)=0$. The users of type $t_1$ will choose \ac{sft}, and the users of type $t_2$ can choose both \ac{sft} and \ac{icl}. Due to Assumption~\ref{assump:large}, we have that 
\begin{align*}
    R_{\sft}N_{\sft}(t,H)\geq R_{\icl}N_{\icl}(t,H).
\end{align*}
To specify $\barJ(H)$, we define the following function
\begin{align*}
    f(H) = q\sqrt{\frac{2\tilsigma_{1}^2 d_1 R_\sft}{H}} + (1-q)\max\bigg\{0,\sqrt{\frac{2\tilsigma_{2}^2 r_2 R_\icl}{H}}- \frac{2\tilsigma_2^2R_\icl}{\zeta_2}\bigg\}.
\end{align*}
Under this case, we have that $\barJ(H) = [f(H),e(H)].$ If $f(H_{\sep}^{*}(t_2))\leq \sqrt{H_{\sep}^{*}(t_2)-p}\leq e(H_{\sep}^{*}(t_2))$, then we have that the root $H^*$ of \eqref{eq:H_fixed_equation_hetero} is $H_{\sep}^{*}(t_2)$.

\textbf{Case 3: $ H_{\sep}^{*}(t_2)<H<H_{\sep}^{*}(t_1)$.}

Under this setting, the users of type $t_2$ will choose \ac{icl}, while the users of type $t_1$ will choose \ac{sft}. Thus, we have that $\barJ(H)= f(H)$. We note that $f(H)$ is strictly decreasing in $H$. If $\sqrt{H_{\sep}^{*}(t_2)-p}<f(H_{\sep}^{*}(t_2))$ and $\sqrt{H_{\sep}^{*}(t_1)-p}>f(H_{\sep}^{*}(t_1))$, then the root $H^*$ of \eqref{eq:H_fixed_equation_hetero} is the root of the following equation
\begin{align*}
    \sqrt{H-p}=f(H).
\end{align*}

\textbf{Case 4: $ H=H_{\sep}^{*}(t_1)$.}

Under this case, the users of type $t_2$ will choose \ac{icl}, and the users of type $t_1$ can choose both \ac{sft} and \ac{icl}. Thus, we define the following function
\begin{align*}
    g(H) = q\max\bigg\{0,\sqrt{\frac{2\tilsigma_{1}^2 r_1 R_\icl}{H}}- \frac{2\tilsigma_1^2R_\icl}{\zeta_1}\bigg\} + (1-q)\max\bigg\{0,\sqrt{\frac{2\tilsigma_{2}^2 r_2 R_\icl}{H}}- \frac{2\tilsigma_2^2R_\icl}{\zeta_2}\bigg\}.
\end{align*}
We have that $\barJ(H) = [g(H),f(H)]$. Thus, if $g(H_{\sep}^{*}(t_1))\leq \sqrt{H_{\sep}^{*}(t_1)-p}\leq f(H_{\sep}^{*}(t_1))$, then the root $H^*$ of \eqref{eq:H_fixed_equation_hetero} is equal to $H_{\sep}^{*}(t_1)$.

\textbf{Case 5: $ H>H_{\sep}^{*}(t_1)$.}

Under this case, all the users adopt \ac{icl}. Thus, we have that $\barJ(H) = g(H)$. The function $g(H)$ is decreasing in $H$. If $\sqrt{H_{\sep}^{*}(t_1)-p}< g(H_{\sep}^{*}(t_1))$, the root $H^*$ of \eqref{eq:H_fixed_equation_hetero} is the root of the following equation
\begin{align*}
    \sqrt{H-p}=g(H).
\end{align*}

Then, we handle the case $H_{\sep}^{*}(t_2)= H_{\sep}^{*}(t_1)$.

\textbf{Case 1: $H< H_{\sep}^{*}(t_2)=H_{\sep}^{*}(t_1)$.}

All the users will choose \ac{sft} in this case. Thus, we have that $\barJ(H) = e(H)$. If $\sqrt{H_{\sep}^{*}(t_2)-p}> e(H_{\sep}^{*}(t_2))$, then the root $H^*$ of \eqref{eq:H_fixed_equation_hetero} is not larger than $H_{\sep}^{*}(t_2)$, which is the unique root of $\sqrt{H-p}=e(H)$.

\textbf{Case 2: $H= H_{\sep}^{*}(t_2)=H_{\sep}^{*}(t_1)$.}

In this case, all the users can choose both \ac{sft} and \ac{icl}. Thus, we have that $\barJ(H) =[g(H),e(H)] $. If $g( H_{\sep}^{*}(t_2))\leq \sqrt{ H_{\sep}^{*}(t_2)-p}\leq e( H_{\sep}^{*}(t_2))$, then $H^*= H_{\sep}^{*}(t_2)$.

\textbf{Case 3: $H> H_{\sep}^{*}(t_2)=H_{\sep}^{*}(t_1)$.}

All the users will choose \ac{icl} in this case. Thus, we have that $\barJ(H) = g(H)$. If $\sqrt{H_{\sep}^{*}(t_2)-p}< g(H_{\sep}^{*}(t_2))$, the root $H^*$ of \eqref{eq:H_fixed_equation_hetero} is the root of the following equation $\sqrt{H-p}=g(H).$

\section{Proof of Theorems in Section~\ref{sec:smfcg}}\label{app:smfcg_proof}
\subsection{Proof of Theorem~\ref{thm:r_decreasing_p}}\label{app:r_decreasing_p}
The proof proceeds in three steps.

\textbf{Step 1: Individual demand is decreasing in $p$.}

Fix a type $t\in \calT$ and a congestion level $R\ge 0$.
Recall that the user cost is
\begin{align*}
C\bigl(t,(a,N),R,p\bigr)
=E_a(t,N)+R_a N \bigl(p+h(R)\bigr),
\end{align*}
where we explicitly write out $p$ as a variable of $C$.
Define the total resource consumption $r=R_aN$ and the induced statistical error function
\begin{align*}
g_t(r)=\min_{a\in\{\sft,\icl\}} E_a\!\left(t,\frac{r}{R_a}\right).
\end{align*}
Then the type-$t$ optimization problem is equivalent to
\begin{align*}
\min_{r\ge 0}\; \bigl\{ g_t(r)+\bar h(R,p)\, r \bigr\},\text{ where }
\barh(R,p)=p+h(R).
\end{align*}
Let $D(t,R,p)$ denote the set of minimizers. Take any $p_1<p_2$, and let $r_1\in D(t,R,p_1)$, $r_2\in D(t,R,p_2)$.
Optimality implies
\begin{align*}
g_t(r_1)+\bar h(R,p_1) r_1
\le g_t(r_2)+\bar h(R,p_1) r_2,
\end{align*}
\begin{align*}
g_t(r_2)+\bar h(R,p_2) r_2
\le g_t(r_1)+\bar h(R,p_2) r_1.
\end{align*}
Adding the two inequalities and cancelling $g_t(\cdot)$ yields
\begin{align*}
\bigl(\bar h(R,p_2)-\bar h(R,p_1)\bigr)(r_2-r_1)\le 0.
\end{align*}
Since $\bar h(R,p_2)>\bar h(R,p_1)$, it follows that $r_2\le r_1$.
Because the choice of $r_1,r_2$ was arbitrary,
\begin{align*}
\sup D(t,R,p_2)\le \inf D(t,R,p_1).
\tag{A.1}
\label{eq:D_monotone_p}
\end{align*}

\textbf{Step 2: Aggregate best response is decreasing in $p$.}

Let $\hatD(t,R,p)=\operatorname{co} D(t,R,p)$ be the convexified demand correspondence,
and define the aggregate demand correspondence via the Aumann integral
\begin{align*}
\bar D(R,p)=\int_T \hatD(t,R,p)\,T(dt).
\end{align*}
Convexification preserves the ordering of extrema, and \eqref{eq:D_monotone_p} implies
\begin{align*}
\sup \bar D(R,p_2)\le \inf \bar D(R,p_1)
\qquad \forall R\ge 0.
\tag{A.2}
\label{eq:aggregate_monotone_p}
\end{align*}

\textbf{Step 3: Comparison of fixed points.}

By the equilibrium definition, we have that
\begin{align*}
R^*(p)\in \bar D(R^*(p),p).
\end{align*}
Assume for contradiction that $p_1<p_2$ but $R^*(p_2)>R^*(p_1)$.
Using the monotonicity of $\bar D(\cdot,p_2)$ in $R$ (Lemma~\ref{lem:strong_mono}) and
\eqref{eq:aggregate_monotone_p}, we obtain
\begin{align*}
R^*(p_2)
\leq  \sup\bar D(R^*(p_2),p_2)
\leq\inf \bar D(R^*(p_1),p_2)
\leq \inf \bar D(R^*(p_1),p_1)
\leq R^*(p_1),
\end{align*}
which contradicts $R^*(p_2)>R^*(p_1)$.
Therefore,
\begin{align*}
R^*(p_2)\le R^*(p_1).
\end{align*}

We conclude the proof of Theorem~\ref{thm:r_decreasing_p}.

\subsection{Proof of Theorem~\ref{thm:sft_better}}\label{app:sft_better}
We first define several quantities. For $H>0$, define the optimal number of sample for each algorithm as
\begin{align*}
    N_a(t,H)=\argmin_{N\geq 0}\big\{E_a(t,N)+R_a H N\big\}.
\end{align*}
For each $(t,R)$ define the minimal best-response demand
\[
d_{\full}(t,R)=\inf\{R_a\tilde N\given (a,\tilde N)\in BR(t,R)\}.
\]
Lemma~\ref{lem:strong_mono} shows that the map $R\mapsto d_{\full}(t,R)$ is decreasing for each $t$. We also define the analogy for \ac{icl} as
\begin{align*}
    d_{\icl}(t,R) = R_\icl\cdot N_\icl\big(t,p+h(R)\big).
\end{align*}
In the following, since we fix the value $p$ through the whole proof, we will omit the dependency on $p$ in the notations.

\textbf{Step 1: Comparison of $d_{\full}$ and $d_{\icl}$.}

To compare $d_{\full}$ and $d_{\icl}$, we present the following proposition.
\begin{proposition}\label{prop:sft_icl_mono}
    Under Assumption~\ref{assump:large_large}, for any $t\in\calT,H>0$, we have that $R_{\sft}N_\sft(t,H)> R_{\icl}N_\icl(t,H)$.
\end{proposition}
The proof is in Appendix~\ref{app:sft_icl_mono}. Thus, we have that
\begin{align*}
    d_{\full}(t,R) = \inf_{(a,\tilde N)\in BR(t,R)} R_a\tilde N \geq R_{\icl}N_\icl(t,p+h(R))=d_{\icl}(t,R) \text{ for all }t\in\calT,R\geq 0.
\end{align*}

\textbf{Step 2: Aggregate maps and their monotonicity.}

Define the aggregate minimal-demand map in the full game and the aggregate ICL-demand map:
\[
F_{\full}(R):=\int_{\calT} d_{\full}(t,R)\,T(\rmd t),
\qquad
F_{\icl}(R):=\int_{\calT} d_{\mathrm{icl}}(t,R)\,T(\rmd t).
\]
These integrals are well-defined by measurability and nonnegativity. Since $d_{\full}(t,R)\geq d_{\icl}(t,R)$ pointwisely, we have that 
\begin{equation}
\label{eq:Fmin_ge_Ficl}
F_{\full}(R)\ \ge\ F_{\icl}(R),
\qquad \forall R\ge 0.
\end{equation}

\textbf{Step 3: Equilibria Comparison.}

By definition of $R_{\mathrm{icl}}^*$ as the equilibrium congestion level in the ICL-only game, it is the fixed point of the following equation
\begin{equation}
\label{eq:icl_fixedpoint}
R_{\mathrm{icl}}^*\ =\ F_{\icl}(R_{\mathrm{icl}}^*).
\end{equation}
The existence of such equilibrium can be similarly proved as Theorem~\ref{thm:exist_unique}. For conciseness, we omit it here. Combining \eqref{eq:Fmin_ge_Ficl} and \eqref{eq:icl_fixedpoint} yields
\begin{equation}
\label{eq:Fmin_at_Ricl}
F_{\full}(R_{\mathrm{icl}}^*)\ \ge\ R_{\mathrm{icl}}^*.
\end{equation}

Let $(\pi^*,R^*)$ be an equilibrium in the full game.
For $T$-a.e.\ $t$, $\pi_t^*$ is supported on $BR(t,R^*)$.
Therefore, for $T$-a.e.\ $t$,
\[
\mathbb E_{(a,\tilN)\sim\pi_t^*}[R_a\tilde N]\ \ge\ \inf_{(a,\tilde N)\in BR(t,R^*)} R_a\tilde N
\ =\ d_{\full}(t,R^*).
\]
Integrating over $t$ and using consistency ($R^*$ equals aggregate expected demand) gives
\begin{equation}
\label{eq:R_ge_Fmin}
R^*\ =\ \int_{\calT} \mathbb E_{\pi_t^*}[R_a\tilde N]\,T(\rmd t)\ \ge\ \int_{\calT} d_{\full}(t,R^*)\,T(\rmd t)\ =\ F_{\full}(R^*).
\end{equation}

Suppose for contradiction that $R^*<R_{\mathrm{icl}}^*$.
Since $F_{\full}$ is decreasing,
\[
F_{\full}(R^*)\ \ge\ F_{\full}(R_{\mathrm{icl}}^*)\ \ge\ R_{\mathrm{icl}}^* \ >\ R^*,
\]
where the second inequality is \eqref{eq:Fmin_at_Ricl}.
This contradicts \eqref{eq:R_ge_Fmin}, which states $R^*\ge F_{\full}(R^*)$.
Hence $R^*\ge R_{\mathrm{icl}}^*$. Thus, we conclude the proof of Theorem~\ref{thm:sft_better}.

\section{Proofs of Propositions and Corollaries}\label{app:prop_proof}

\subsection{Proof of Corollary~\ref{coro:pre_icl}}\label{app:coro_pre_icl}
Define $\bar{\sigma}^2 = \sigma^{2}+c^2$. We directly calculate the difference between the errors of pretrained models and \ac{icl} as follows.
\begin{align*}
    &E_{\pre}-E_{\icl} \\
    &\quad = \sum_{i=1}^{r} (\bar{\sigma}^2 s_i+\tilde{\sigma}^2 \pi_i)^{-2} \Big\{\big[\tau_i + 2\bar{\sigma}^2 \pi_i^{-1}+ (v_i^{\top}\mu_{\theta})^2\big]\cdot (\bar{\sigma}^2 s_i+ \tilde{\sigma}^2 \pi_i)^2 \nonumber\\ 
    &\quad\qquad\qquad\qquad\qquad\qquad\qquad\qquad\quad - \tilde{\sigma}^4\pi_i^{2}\tau_i - 2\tilde{\sigma}^4\bar{\sigma}^2\pi_i - \tilde{\sigma}^4\pi_i^2(v_i^{\top}\mu_{\theta})^2 - 2 \tilde{\sigma}^2\bar{\sigma}^4 s_i  \Big\}\nonumber\\
    &\quad =  \sum_{i=1}^{r} (\bar{\sigma}^2 s_i+\tilde{\sigma}^2 \pi_i)^{-2} \Big\{\big[\tau_i + (v_i^{\top}\mu_{\theta})^2\big]\cdot (\bar{\sigma}^4 s_i^2 + 2\bar{\sigma}^2\tilde{\sigma}^2 s_i\pi_i)+ 2\bar{\sigma}^6 s_i^{2} \pi_i^{-1} + 2 \tilde{\sigma}^2\bar{\sigma}^4 s_i \Big\}.
\end{align*}
The right-hand side of this equality is non-negative. Thus, we have that $E_\pre\geq E_\icl$. In the non-degenerate case, the right-hand side of this equality is positive.

\subsection{Proof of Proposition~\ref{prop:comp}}\label{app:prop_comp}

We first calculate the optimal value $\lambda^*$ for \ac{sft} and the corresponding error $E_{\sft}^{\lambda^*}$. With Assumption~\ref{assump:homo}, Theorem~\ref{thm:err_analysis} shows that the \ac{sft} error with a fixed regularization $\lambda>0$ is
\begin{align}
\label{eq:SFT-master}
E_{\sft}^{\lambda}
=r\cdot
\frac{2\,\tilde{\sigma}^2\,\bar{\sigma}^4\,s
+\lambda^2\,\tilde{\sigma}^4\,\pi^2\big(\tau+2\bar{\sigma}^2/\pi+m^2\big)}
{\big(\bar{\sigma}^2 s+\lambda\,\tilde{\sigma}^2\pi\big)^2}
+2(d-r)\cdot\frac{\tilde{\sigma}^2}{s}.
\end{align}
To calculate the minimal value of $E_{\sft}^{\lambda}$, we define following quantities.
\begin{align*}
A = 2r\,\tilde{\sigma}^2\,\bar{\sigma}^4\,s,\quad 
B = \tilde{\sigma}^4\,\pi^2\,r\big(\tau+2\bar{\sigma}^2/\pi+m^2\big),\quad  C = \bar{\sigma}^2 s,\quad D=\tilde{\sigma}^2\pi.
\end{align*}
Then the first term of the error in \eqref{eq:SFT-master} equals
\begin{align}
\label{eq:f-lambda}
F(\lambda)=\frac{A+B\lambda^2}{(C+D\lambda)^2}.
\end{align}
Thus, we have that
\begin{align*}
E_{\sft}^{\lambda}=F(\lambda)+2(d-r)\,\frac{\tilde{\sigma}^2}{s}.
\end{align*}
Differentiating \eqref{eq:f-lambda}, we get
\begin{align*}
F'(\lambda)
&=\frac{(2B\lambda)(C+D\lambda)^2-(A+B\lambda^2)\cdot 2D(C+D\lambda)}{(C+D\lambda)^4}\\
&=\frac{2(C+D\lambda)\big[B\lambda(C+D\lambda)-D(A+B\lambda^2)\big]}{(C+D\lambda)^4}.
\end{align*}
Since $C+D\lambda>0$, critical points satisfy
\[
B\lambda(C+D\lambda)-D(A+B\lambda^2)=BC\lambda-DA=0,
\]
hence the unique minimizer on $(0,\infty)$ is
\begin{align*}
\lambda^{*}=\frac{A\cdot D}{B\cdot C}
=\frac{\tilde{\sigma}^2\pi\cdot2r\,\tilde{\sigma}^2\bar{\sigma}^4 s}
{\tilde{\sigma}^4\pi^2 r\big(\tau+2\bar{\sigma}^2/\pi+m^2\big)\cdot\bar{\sigma}^2 s}
=\frac{2\,\bar{\sigma}^2}{\pi\big(\tau+m^2\big)+2\,\bar{\sigma}^2}.
\end{align*}

\textbf{Minimal row-space value.} For $F(\lambda)=(A+B\lambda^2)/(C+D\lambda)^2$ with $A,B,C,D>0$, the minimal value at $\lambda^{*}$ is
\begin{align}
\label{eq:Fmin-template}
F(\lambda^{*})=\frac{A\,B}{A D^2 + B C^2}.
\end{align}
Substituting $A,B,C,D$ and simplifying, we have that
\begin{align*}
F(\lambda^{*})
&=\frac{2r\,\tilde{\sigma}^2\,\Big(\tau+2\bar{\sigma}^2\pi^{-1}+m^2\Big)}
{s\,\Big(\tau+2\bar{\sigma}^2\pi^{-1}+m^2\Big)+2\,\tilde{\sigma}^2}.
\end{align*}

Thus, the optimal \ac{sft} error is 
\begin{align}
E_{\mathrm{sft}}^{\lambda^{*}} =F(\lambda^{*})+2(d-r)\frac{\tilde{\sigma}^2}{s}
=\frac{2r\,\tilde{\sigma}^2\,\Big(\tau+2\bar{\sigma}^2\pi^{-1}+m^2\Big)}
{s\,\Big(\tau+2\bar{\sigma}^2\pi^{-1}+m^2\Big)+2\,\tilde{\sigma}^2}
+2(d-r)\frac{\tilde{\sigma}^2}{s}.\label{eq:sft_opt}
\end{align}
We first derive the error of \ac{icl} under Assumption~\ref{assump:homo}. According to Theorem~\ref{thm:err_analysis}, \ac{icl}’s row-space per-direction error in the first $r$ subspaces simplifies to
\begin{align*}
\frac{\tilde{\sigma}^4 \pi_i^2\big(\tau_i+\bar{\sigma}^2\pi_i^{-1}+(v_i^\top \mu_\theta)^2\big)+\tilde{\sigma}^2\bar{\sigma}^4s_i+\tilde{\sigma}^2\bar{\sigma}^2(\bar{\sigma}^2s_i+\tilde{\sigma}^2\pi_i)}{(\bar{\sigma}^2s_i+\tilde{\sigma}^2\pi_i)^2}=\frac{\tilsigma^4\pi^2\big(\tau+\bar\sigma^2/\pi+m^2\big)+2\tilsigma^2\bar\sigma^4 s+\bar\sigma^2\tilsigma^4\pi}
{(\bar\sigma^2 s+\tilsigma^2\pi)^2}.
\end{align*}
To further simplify the expression, we define
\begin{align*}
    \zeta=\tau+\frac{2\bar\sigma^2}{\pi}+m^2.
\end{align*}
Observe that $\tilsigma^4\pi^2(\tau+\bar\sigma^2/\pi+m^2)
=\tilsigma^2\big(\tilsigma^2\pi^2 \zeta- \bar\sigma^2\tilsigma^2\pi\big)$, so the \ac{icl} error becomes
\begin{align*}
 E_{\icl}=r\cdot \frac{\tilsigma^2\big(2\bar\sigma^4 s+\tilsigma^2\pi^2 \zeta\big)}{(\bar\sigma^2 s+\tilsigma^2\pi)^2}+(d-r)\tau.
\end{align*}
According to \eqref{eq:sft_opt}, the $\ac{sft}$ error can be simplified as
\begin{align*}
     E_{\sft}^{*} = r\cdot \frac{2\tilsigma^2 \zeta}{s\zeta+2\tilsigma^2} + (d-r)\frac{2\tilsigma^2}{s}.
\end{align*}

The difference between the \ac{sft} and \ac{icl} errors is
\begin{align*}
\Delta=E_{\mathrm{sft}}^{\lambda^*}-E_{\mathrm{icl}}
= r\Big(\frac{2\tilsigma^2 \zeta}{s\zeta+2\tilsigma^2}-\frac{\tilsigma^2\big(2\bar\sigma^4 s+\tilsigma^2\pi^2 \zeta\big)}{(\bar\sigma^2 s+\tilsigma^2\pi)^2}\Big)
+(d-r)\Big(\frac{2\tilsigma^2}{s}-\tau\Big) = r\cdot \delta +(d-r)\Big(\frac{2\tilsigma^2}{s}-\tau\Big).
\end{align*}
The term $\delta$ can be simplified as
\begin{align*}
\delta&=\frac{2\tilsigma^2 \zeta\,(\bar\sigma^2 s+\tilsigma^2\pi)^2-(s\zeta+2\tilsigma^2)\tilsigma^2(2\bar\sigma^4 s+\tilsigma^2\pi^2 \zeta)}{(s\zeta+2\tilsigma^2)\,(\bar\sigma^2 s+\tilsigma^2\pi)^2}\\
&=-\frac{\pi^{3}\,\tilsigma^4\,(\tau+m^2)^{2}\,s}
{\big(\bar\sigma^2 s+\pi\tilsigma^2\big)^{2}\,\big(2\bar\sigma^2 s+\pi s(\tau+m^2)+2\pi\tilsigma^2\big)}.
\end{align*}
The fact that $\delta<0$ establishes that \emph{SFT is strictly better than ICL within the pretraining subspace}. To further study the comparison between \ac{sft} and \ac{icl}, we define
\begin{align*}
    \kappa=\dfrac{s\tau}{2\tilsigma^2}, \text{ and }R=\frac{r}{d-r}
\end{align*}
Then the error difference is expressed as
\begin{align*}
\Delta= (d-r)\Big(R\cdot \delta+\frac{2\tilsigma^2}{s}-\tau\Big).
\end{align*}
Hence $\Delta\lesseqqgtr 0$ if and only if
\begin{align*}
R\ \gtreqqless\ \frac{\tau-2\tilsigma^2s^{-1}}{\delta}
=-\frac{\tau-2\tilsigma^2s^{-1}}{\,|\delta|\,}.
\end{align*}
Substitute $s=2\tilsigma^2\tau^{-1}\kappa$ into $\delta$ and simplify to obtain
\begin{align*}
R_{\crit}
=\frac{(1-\kappa)\big(2\bar\sigma^2\kappa+\pi\tau\big)^2\big(2\bar\sigma^2\kappa+\pi\kappa(\tau+m^2)+\pi\tau\big)}
{\pi^{3}\,\kappa^2\,\tau\,(\tau+m^2)^{2}}.
\end{align*}
Thus, we reach the following conclusions
\begin{itemize}
\item If $\kappa\ge 1$ then \ac{sft} has less error than \ac{icl}.
\item If $0<\kappa<1$, then the sign of $\Delta$ is decided by $R$ relative to $R_{\crit}$.
\end{itemize}
We conclude the proof of Proposition~\ref{prop:comp}.

\subsection{Proof of Proposition~\ref{prop:pretrain_influence}}\label{app:pretrain_influence}
\begin{proposition}
    Under models \eqref{eq:simplified} with $\alpha=1$, $\barh(R,p)=p+R^2$, and $R_{\sft}>R_{\icl}$, the following holds.
    \begin{itemize}
        \item[1.] The congestion-level $R^{\ast}$ decreasing in $\pi$.
        \item[2.] Assume that $2\tau > \zeta$. If the following condition holds 
        \begin{align*}
            p< \frac{r^*\zeta^2}{2R_{\icl}\tilsigma^2},\text{ and } r^*> 2 \tilsigma^2 R_{\icl}\tau^{-4}\zeta^{-2}\big[4\tilsigma^4R_{\icl}^2(\zeta-\tau)^2+\tau^2W^2p\big],
        \end{align*}
        where
        \begin{align*}
            r^*=\bigg(\frac{\sqrt D-\sqrt{D-(1-E)F}}{1-E}\bigg)^2\text{ for } D= R_{\sft}d/R_{\icl}, E=\tau/\zeta, \text{ and }F=d,
        \end{align*}
        then $R^{\ast}$ is increasing in $r$. Otherwise, there exist thresholds $0< r^{*} < r^{**} < d$ (with $r^{*}$ as defined above) such that the equilibrium congestion $R^{*}(r)$ is increasing in $r$ on $[0,r^{*}] \cup [r^{**},d]$ and decreasing in $r$ on $[r^{*},r^{**}]$.
    \end{itemize}
\end{proposition}
Note that the congestion level at the equilibrium is $R^{*} = \sqrt{H^{*}-p}$. We only need to prove the monotonicity of $H^{*}$, which is a threholded version of $H_{\sep}^{*}$ as
\begin{align*}
    H^{*} = \max\big\{H_{\icl}^{*},\min\{H_{\sep}^{*},H_{\sft}^{*}\}\big\}.
\end{align*}
In the following, we will prove the monotonicity of $H_{\icl}^{*}$, $H_{\sft}^{*}$, and $H_{\sep}^{*}$ in $r$ and $\pi$.

We first derive the definitions or expressions of them in the following proposition.
\begin{proposition}\label{prop:H_fixed_points}
    Under models \eqref{eq:simplified} with $\alpha=1$ and $\barh(R,p)=p+R^2$, we have the following results.
    \begin{itemize}
        \item For \ac{sft}, the root $H_{\sft}^{*}$ is 
        \begin{align*}
            H_{\sft}^*=\frac{p+\sqrt{p^2+8\tilsigma^2 d\,R_{\sft}}}{2}.
        \end{align*}
        \item For \ac{icl}, $H_{\icl}^{*}$ is the root of the following equation
        \begin{align}
            \max\left\{0,\sqrt{\frac{2\tilsigma^2 rR_{\icl}}{H}}-\frac{2\tilsigma^2 R_{\icl}}{\zeta}\right\} = \sqrt{H-p}.\label{eq:H_icl}
        \end{align}
        Specifically, when 
        \begin{align*}
           p\geq \barH =  \frac{r\zeta^2}{2R_{\icl}\tilsigma^2},
        \end{align*}
        the root is $H_{\icl}^{*}=p$.
        \item For $H_{\sep}^{*}$, we have that
        \begin{align}\label{eq:H_sep}
            H_{\sep}^*
            &= 
            \begin{cases}
            \displaystyle \frac{\left((d-r)\tau+r\zeta\right)^2}{8\tilsigma^2 d\,R_{\sft}}, &\text{if }\frac{R_\sft}{R_\icl}\leq \frac{\big((d-r)\tau+r\zeta\big)^2}{4dr\zeta^2} ,\\[6pt]
            \displaystyle \left(\frac{-B+\sqrt{B^2+4C(d-r)\tau}}{2C}\right)^2, & \text{otherwise},
            \end{cases}
        \end{align}
        where
        \begin{align*}
        B =2\sqrt{2}\tilsigma\Big(\sqrt{dR_{\sft}}-\sqrt{rR_{\icl}}\Big),
        \qquad
        C =\frac{2\tilsigma^2 R_{\icl}}{\zeta}.
        \end{align*}
    \end{itemize}
\end{proposition}
The proof is in Appendix~\ref{app:H_fixed_points}.

Then we prove the monotonicity of them.  We start from the parameter $\pi$. Since $H_{\sft}^{*}$ is not influenced by $\pi$, which is a trivially monotonic function of $\pi$. For $H_{\icl}^{*}$, we first note that the parameter $\zeta$ depends on $\pi$ as $\zeta = \tau + 2\bar{\sigma}^2\pi^{-1} + m^2$. Thus, $\zeta^{-1}$ is an increasing function in $\pi$. Proposition~\ref{prop:H_fixed_points} shows that $H_{\icl}^{*}$ is the root of \eqref{eq:H_icl}, whose left-hand side is decreasing in $\pi$. In addition, the left-hand side of it is decreasing in $H$, while the right-hand side is increasing in $H$. Thus, when $\pi$ increases, the root $H_{\icl}^{*}$ accordingly decreases to balance both sides. For $H_{\sep}^{*}$, we note that when $R_\sft/R_\icl\leq \big((d-r)\tau+r\zeta\big)^2/(4dr\zeta^2)$, $H_{\sep}^{*}$ is decreasing in $\pi$. When this condition does not hold, Proposition~\ref{prop:mono1} in Appendix~\ref{app:mono} shows that the function
\[
f(x)=\frac{\sqrt{A+x}-\sqrt{A}}{x}
\]
is decreasing in $x$ when $A>0$ and $x>0$. Thus, $H_{\sep}^{*}$ is also decreasing in $\pi$. In summary, all of $H_{\icl}^{*}$, $H_{\sft}^{*}$, and $H_{\sep}^{*}$ is decreasing in $\pi$.
We then analyze the parameter $r$. We first note that the value of $H_{\sft}^{*}$ is independent of $r$. For $H_{\icl}^{*}$, since the left-hand side of \eqref{eq:H_icl} is increasing in $r$. Thus, $H_{\icl}^{*}$ increases with increasing $r$. For $H_{\sep}^{*}$, we first analyze when the following relationship is satisfied
\begin{align}
    \frac{R_\sft}{R_\icl}\leq \frac{\big((d-r)\tau+r\zeta\big)^2}{4dr\zeta^2}.\label{ieq:cond1}
\end{align}
Proposition~\ref{prop:mono2} in Appendix~\ref{app:mono} shows that the function
\begin{align*}
f(x)\;=\;\frac{(A-x+Bx)^2}{AB^2x},\qquad 0<x<A,\ \ A>0,\ \ B> 0.
\end{align*}
is strictly decreasing if $0<B<2$. In addition, it achieves its minimal value $f(A)=1$ at $x=A$. Set $A=d$ and $B=\zeta/\tau$. Since $\zeta\leq 2\tau$, the right-hand side of \eqref{ieq:cond1} is decreasing and achieves its minimal value $1/4$ when $d=r$. Since $R_{\sft}>R_{\icl}$, the expression of $H_{\sep}^{*}$ has one shift from the first expression in \eqref{eq:H_sep} to the second one. When \eqref{ieq:cond1} holds, $H_{\sep}^{*}$  is a increasing function of $r$ due to $\zeta>\tau$. When \eqref{ieq:cond1} does not hold, we note that Proposition~\ref{prop:monotone_g} shows that the function
\begin{align*}
g(r)=\sqrt{\bigl(\sqrt D-\sqrt r\bigr)^2+E(F-r)}-\bigl(\sqrt D-\sqrt r\bigr)
\end{align*}
 is strictly increasing on $(0,r^{*})$ and strictly decreasing on $(r^{*},F)$ when $D>F>0$, $E<1$, and $0<x<F$. The value of $r^{*}$ is given by
\begin{align*}
t^{*}=\frac{\sqrt D-\sqrt{D-(1-E)F}}{1-E}, \text{ and }r^{*}=(t^{*})^2.
\end{align*}
We set $D= R_{\sft}d/R_{\icl}$, $E=\tau/\zeta$, and $F=d$. To study the relationship between $r^*$ and the phase shift in \eqref{ieq:cond1}, it is easy to directly verify that
\begin{align*}
    \frac{R_\sft}{R_\icl}> \frac{\big((d-r^{*})\tau+r^{*}\zeta\big)^2}{4dr^{*}\zeta^2}.
\end{align*}
Thus, $H_{\sep}^{*}$ is increasing in $r$ for $r\in[0,r^*]$ and is decreasing in $r$ for $r\in[r^*,d]$. When $r=d$, $H_{\sep}^{*}=0$. Thus, at the point $r^*$, if $H_{\icl}^*\geq H_{\sep}^*$, then $H^*$ is increasing in $r$. If $H_{\icl}^*<H_{\sep}^*$, then $H^*$ is increasing in $r$ for $r\in[0,r^*]$ and $[r^{**},d]$ and deacreasing in $r$ for $[r^*,r^{**}]$. In the following, we only need to compare $H_{\icl}^*$ and $H_{\sep}^*$ at $r^*$. It is easy to verify that $H_{\sep}^*> H_{\icl}^*$ if and only if
\begin{align}
    p< \frac{r^*\zeta^2}{2R_{\icl}\tilsigma^2},\text{ and } \sqrt{\frac{2\tilsigma^2 r R_{\icl}}{H_{\sep}^*}}-\frac{2\tilsigma^2R_{\icl}}{\zeta}<\sqrt{H_{\sep}^*-p}.\label{eq:conds}
\end{align}
We note that the latter inequality in  \eqref{eq:conds} can be achieved if and only if 
\begin{align*}
    r^*> 2 \tilsigma^2 R_{\icl}\tau^{-4}\zeta^{-2}\big[4\tilsigma^4R_{\icl}^2(\zeta-\tau)^2+\tau^2W^2p\big].
\end{align*}
Thus, we conclude the proof of Proposition~\ref{prop:pretrain_influence}.

\subsection{Proof of Proposition~\ref{prop:task_sigma}}\label{app:task_sigma}
Proposition~\ref{prop:H_fixed_points} in the proof of Proposition~\ref{prop:pretrain_influence} shows that $H_{\sft}^{*}$ is increasing in $\tilsigma$, $H_{\sft}^{*}=p$ if $\tilsigma=0$, and $H_{\sft}^{*}\rightarrow\infty$ as $\tilsigma\rightarrow\infty$. In addition, we have that $H_{\sep}^{*}$ is decreasing in $\tilsigma$, $H_{\sep}^{*}\rightarrow \infty$ as $\tilsigma\downarrow 0$, and $H_{\sep}^{*}\rightarrow 0$ as $\tilsigma\rightarrow \infty$. Thus, we can take $\tilsigma_1$ as their intersection. When $\tilsigma\in[0,\tilsigma_1]$, $H^{*}=H_{\sft}^{*}$ is increasing in $\tilsigma$. Thus, we prove the first claim.

For the second claim, we note that when $\tilsigma\in[\tilsigma_1,\infty)$, $H^{*}=\max\{H_{\icl}^{*},H_{\sep}^{*}\}$. Thus, we only need to prove that $H_{\icl}^{*}$ is decreasing in $\tilsigma$ when $\tilsigma\in[\tilsigma_2,\infty)$. 

Proposition~\ref{prop:icl_root_property} in Appendix~\ref{app:mono} analyzes the root of 
\begin{equation*}
\frac{e x}{\sqrt{H}} - f x^2 - \sqrt{H-g}=0
\end{equation*}
for $e,f,g>0$. Denote the unique root $H_*(x)$ as a function of $x$, it shows that when $0 < x \le x_{\max} =e/(f\sqrt{g})$, there exists $x_* < x_{\max}$ such that $H_*(x)$ is increasing on $(0,x_*]$ and decreasing on $(x_*, x_{\max}]$. In our setting, $x=\tilsigma$, $e = \sqrt{2rR_{\icl}}$, $f = 2 R_\icl/\zeta$, $g=p$. If $e/(f\sqrt{g})>\tilsigma_1$, then we can set $\tilsigma_2 = x_*$ in Proposition~\ref{prop:icl_root_property}. If $e/(f\sqrt{g})<\tilsigma_1$, we can set $\tilsigma_2=\tilsigma_1$. Thus, we conclude the proof of Proposition~\ref{prop:task_sigma}.

\subsection{Proof of Proposition~\ref{prop:r_sft}}\label{app:r_sft}
Proposition~\ref{prop:H_fixed_points} in the proof of Proposition~\ref{prop:pretrain_influence} shows that once $R_{\icl}$ is fixed, the value of $H_{\icl}^{*}$ is fixed. Thus, we only need to analyze $H_{\sft}^*$ and $H_{\sep}^*$. Obviously,  $H_{\sft}^*$ is increasing in $R_{\sft}$, and its minimal value is 
\begin{align*}
    \frac{p+\sqrt{p^2+8\tilde\sigma^2 d\,R_{\icl}}}{2}.
\end{align*}

For $H_{\sep}^{*}$, if 
\begin{align}
    \frac{R_\sft}{R_\icl}\leq \frac{\big((d-r)\tau+r\zeta\big)^2}{4dr\zeta^2},\label{ieq:cond3}
\end{align}
then Proposition~\ref{prop:H_fixed_points} shows that it is a decreasing function of $R_{\sft}$. If \eqref{ieq:cond3} does not hold, then Proposition~\ref{prop:mono4} in Appendix~\ref{app:mono} establishes the same monotonicity result. Thus, $H_{\sep}^{*}$ is a decreasing function of $R_{\sft}$. Under both cases,  $H_{\sep}^{*}$ decreases to $0$ as $R_{\sft}\rightarrow\infty$.

In the following, we will compare the value of $H_{\sep}^{*}$ and $H_{\sft}^{*}$ when $R_{\sft}=R_{\icl}$. If  $H_{\sep}^{*}>H_{\sft}^{*}$, $H^*$ first increases and then decreases with increasing $R_{\sft}$. If $H_{\sep}^{*}\leq H_{\sft}^{*}$, $H^*$ decreases with increasing $R_{\sft}$. To compare their values when $R_{\sft}=R_{\icl}$, we need to study whether \eqref{ieq:cond3} holds if $R_{\sft}=R_{\icl}$.

If $(\big((d-r)\tau+r\zeta\big)^2)/(4dr\zeta^2)\geq 1$, then \eqref{ieq:cond3} holds when $R_{\sft}=R_{\icl}$. Thus, $H_{\sep}^{*}\leq H_{\sft}^{*}$ is equivalent to
\begin{align*}
    \displaystyle \frac{\left((d-r)\tau+r\zeta\right)^2}{8\tilde\sigma^2 d\,R_{\icl}}\leq \frac{p+\sqrt{p^2+8\tilde\sigma^2 d\,R_{\icl}}}{2}.
\end{align*}
The left-hand side of this inequality is decreasing in $R_{\icl}$, while the right-hand side is increasing in $R_{\icl}$. Thus, there exists $R_{\icl}^{*}$ such that this inequality holds if and only if $R_{\icl}\geq R_{\icl}^{*}$.

If $(\big((d-r)\tau+r\zeta\big)^2)/(4dr\zeta^2)< 1$, then \eqref{ieq:cond3} does not hold when $R_{\sft}=R_{\icl}$. Thus, $H_{\sep}^{*}\leq H_{\sft}^{*}$ is equivalent to
\begin{align*}
    \left(\frac{-\sqrt{2}\Big(\sqrt{d}-\sqrt{r}\Big)+\sqrt{2\Big(\sqrt{d}-\sqrt{r}\Big)^2+2 \zeta^{-1}(d-r)\tau}}{2\tilsigma \sqrt{R_{\icl}}\zeta^{-1}}\right)^2\leq \frac{p+\sqrt{p^2+8\tilde\sigma^2 d\,R_{\icl}}}{2}.
\end{align*}
The left-hand side of this inequality is decreasing in $R_{\icl}$, while the right-hand side is increasing in $R_{\icl}$. Thus, there exists $R_{\icl}^{*}$ such that this inequality holds if and only if $R_{\icl}\geq R_{\icl}^{*}$. Thus, we conclude the proof of Proposition~\ref{prop:r_sft}.
\subsection{Proof of Proposition~\ref{prop:only_one_type}}\label{app:only_one_type}
We first state the precise statement of this proposition.
\begin{proposition}
Fix $\tilde{\sigma}_2>0$ and assume $q\in(0,1)$, $\tilde{\sigma}_1<\tilde{\sigma}_2$. Then $\partial_{\tilde{\sigma}_1}R^{\ast}(\tilde{\sigma}_1,\tilde{\sigma}_2,q)=
\partial_q R^{\ast}(\tilde{\sigma}_1,\tilde{\sigma}_2,q)=0$
holds if and only if one of the following two regimes holds:
\begin{itemize}
\item The equilibrium is pinned at the separator of type $2$, i.e., the following interval condition holds
\begin{equation}
f(H^{\ast}_{\sep}(\tilde{\sigma}_2)) < \sqrt{H^{\ast}_{\sep}(\tilde{\sigma}_2)-p} < e(H^{\ast}_{\sep}(\tilde{\sigma}_2)).
\label{eq:case_sep2_interval_strict}
\end{equation}
\item The congestion level is zero $R^{\ast}=0$, i.e., the following condition holds
\begin{equation}
p>\max\{H^{\ast}_{\sep}(\tilde{\sigma}_1), \barH(\tilde{\sigma}_1)\},
\quad
\barH(\tilde{\sigma})=r\zeta^2 (2R_{\icl}\tilde{\sigma}^2)^{-1}.
\label{eq:corner_cond_strict}
\end{equation}
\end{itemize}
Sufficient conditions for \eqref{eq:case_sep2_interval_strict} are
\begin{align*}
     &\big((d-r)\tau+r\zeta\big)^2 (24dR_{\sft}p)^{-1}\leq \tilsigma_2^2 \leq \big((d-r)\tau+r\zeta\big)^2 (12dR_{\sft}p)^{-1}, \quad \tilsigma_1\leq \tilsigma_2/2,\\
     &\frac{1}{\zeta}\geq \frac{1}{(d-r)\tau}\sqrt{\frac{4drR_{\sft}}{R_{\icl}}},\quad (d-r)\tau\geq \max\big\{84 R_{\sft}^{3/2}p^{3/2},9p^{3/2}\big\},\text{ and } q\leq\min\bigg\{\frac{1}{2},\frac{p}{2\tilsigma_1\sqrt{R_\sft d}}\bigg\}.
\end{align*}
\end{proposition}

Before the formal proof, we note that Proposition~\ref{prop:H_fixed_points} in the proof of Proposition~\ref{prop:pretrain_influence} derives the closed-form expression of $H_\sep^*$, which is decreasing in $\tilsigma$. Thus, we have that $H_\sep^*(\tilsigma_2)<H_\sep^*(\tilsigma_1)$. In the following, we prove the sufficiency and the necessity separately. Since $H^*=h(R^*)+p$, this proposition focuses on the case where
\begin{equation}
\partial_{\tilde{\sigma}_1}H^*(\tilde{\sigma}_1,\tilde{\sigma}_2,q)=0
\quad\text{and}\quad
\partial_qH^*(\tilde{\sigma}_1,\tilde{\sigma}_2,q)=0
\label{eq:stationary_goal}
\end{equation}

We start with the sufficiency of the condition. 
\begin{itemize}
    \item If $f(H^*_{\sep}(\tilde{\sigma}_2)) < \sqrt{H^*_{\sep}(\tilde{\sigma}_2)-p} < e(H^*_{\sep}(\tilde{\sigma}_2))$ holds, then Theorem~\ref{thm:hetero} shows that $H^*=H^*_{\sep}(\tilde{\sigma}_2)$, which does not depend on $\tilde{\sigma}_1$ nor $q$. Hence, both partial derivatives in \eqref{eq:stationary_goal} are zero.
    \item If \eqref{eq:corner_cond_strict} holds, then for each type $\tilde{\sigma}\in\{\tilde{\sigma}_1,\tilde{\sigma}_2\}$: (i) $p>H^*_{\sep}(\tilde{\sigma})$ implies $\psi(\tilde{\sigma},p)>0$, so ICL is the unique best-response algorithm; (ii) $p>\bar H(\tilde{\sigma})$ implies $\tilde N_{\icl}(\tilde{\sigma},p)=0$ and hence $D_{\icl}(\tilde{\sigma},p)=0$. Therefore, the congestion level is $0$, i.e., $R^*=\sqrt{p-p}=0$. Moreover, strictness implies local stability of this selection, $H^*$ is locally constant, and both partial derivatives
are zero.
\end{itemize}

Then we prove the necessity. Assume \eqref{eq:stationary_goal} holds at a point where $H^*$ is differentiable in $(\tilde{\sigma}_1,q)$. By Theorem~\ref{thm:hetero}, the equilibrium must lie in one of the five cases. We will rule out all cases except the second and the last ones.

\textbf {Case 1: $H^*$ is a root of $\sqrt{H-p}=e(H)$.} According to the definition, we have that
\[
q\sqrt{\frac{2\tilsigma_{1}^2 d R_\sft}{H^*}} + (1-q)\sqrt{\frac{2\tilsigma_{2}^2 d R_\sft}{H^*}}=\sqrt{H^*-p}.
\]
Since both terms on the left-hand side of this equation are positive, changing the value of either $\tilsigma_1$ or $q$ will shift the value of $H^*$. Specifically,
\begin{align*}
    \partial_{\tilde{\sigma}_1}H^*(\tilde{\sigma}_1,\tilde{\sigma}_2,q)>0
\quad\text{and}\quad
\partial_qH^*(\tilde{\sigma}_1,\tilde{\sigma}_2,q)<0.
\end{align*}

\textbf{Case 2: $H^*=H^*_{\sep}(\tilde{\sigma}_2)$.} This is exactly the ``separator pinned at type 2'' regime. Differentiability and the equalities
\eqref{eq:stationary_goal} requires that we are not on the boundary where the equilibrium switches to a neighboring case, which is precisely the strict interval condition \eqref{eq:case_sep2_interval_strict}.

\textbf{Case 3: $H^*$ is a root of $\sqrt{H-p}=f(H)$.} According to the definition, we have that
\begin{align*}
    q\sqrt{\frac{2\tilsigma_{1}^2 d R_\sft}{H^*}} + (1-q)\max\bigg\{0,\sqrt{\frac{2\tilsigma_{2}^2 r R_\icl}{H^*}}- \frac{2\tilsigma_2^2 R_\icl}{\zeta}\bigg\} = \sqrt{H^* - p}
\end{align*}
Since the first term in the left-hand side of this equation is positive and is increasing in $\tilsigma_1$, thus
\begin{align*}
    \partial_{\tilde{\sigma}_1}H^*(\tilde{\sigma}_1,\tilde{\sigma}_2,q)>0.
\end{align*}

\textbf{Case 4: $H^*=H^*_{\sep}(\tilde{\sigma}_1)$.} Proposition~\ref{prop:H_fixed_points} shows that that $H^*_{\sep}(\tilde{\sigma})$ is a decreasing function of $\tilde{\sigma}$. Thus, we have that
\begin{align*}
    \partial_{\tilde{\sigma}_1}H^*(\tilde{\sigma}_1,\tilde{\sigma}_2,q)<0.
\end{align*}

\textbf{Case 5: $H^*$ is a root of $\sqrt{H-p}=g(H)$.}

According to the definition, we have that
\begin{align*}
    \sqrt{H^*-p}=q\max\bigg\{0,\sqrt{\frac{2\tilsigma_{1}^2 r R_\icl}{H^*}}- \frac{2\tilsigma_1^2 R_\icl}{\zeta}\bigg\} + (1-q)\max\bigg\{0,\sqrt{\frac{2\tilsigma_{2}^2 r R_\icl}{H^*}}- \frac{2\tilsigma_2^2 R_\icl}{\zeta}\bigg\} 
\end{align*}
Write $x_i(H)= R_\icl N_{\icl}(\tilde{\sigma}_i,H)$ so that $g(H)=q x_1(H)+(1-q)x_2(H)$.

If $x_1(H^*)=0$ and $x_2(H^*)>0$. Then $g(H^*)=(1-q)x_2(H^*)$ and $\partial_q g(H^*)=x_1(H^*)-x_2(H^*)=-x_2(H^*)< 0$, contradicting \eqref{eq:stationary_goal}.

If $x_2(H^*)=0$ and $x_1(H^*)>0$. Then $g(H^*)=q x_1(H^*)$ and $\partial_q g(H^*)=x_1(H^*)\neq 0$, again a contradiction.

If $x_1(H^*)>0$ and $x_2(H^*)>0$. In this interior region, the truncation is inactive and
\[
x_i(H)=\frac{A_{\icl}\tilde{\sigma}_i}{\sqrt{H}}-B\tilde{\sigma}_i^2,
\]
where $A_\icl = \sqrt{2rR_\icl}$, and $B=2 R_\icl /\zeta$. Then $\partial_q g(H)=x_1(H)-x_2(H)$, so $\partial_q g(H^*)=0$ implies $x_1(H^*)=x_2(H^*)$, i.e.
\begin{equation}
\frac{A_{\icl}}{\sqrt{H^*}} = B(\tilde{\sigma}_1+\tilde{\sigma}_2).
\label{eq:eq_x1_x2}
\end{equation}
Also, $\partial_{\tilde{\sigma}_1} g(H)=q\,\partial_{\tilde{\sigma}_1}x_1(H)$, and
\[
\partial_{\tilde{\sigma}_1}x_1(H)=\frac{A_{\icl}}{\sqrt{H}}-2B\tilde{\sigma}_1.
\]
Thus $\partial_{\tilde{\sigma}_1}g(H^*)=0$ implies
\begin{equation}
\frac{A_{\icl}}{\sqrt{H^*}} = 2B\tilde{\sigma}_1.
\label{eq:eq_dx1_zero}
\end{equation}
Combining \eqref{eq:eq_x1_x2} and \eqref{eq:eq_dx1_zero} yields
$2\tilde{\sigma}_1=\tilde{\sigma}_1+\tilde{\sigma}_2$, i.e. $\tilde{\sigma}_1=\tilde{\sigma}_2$,
contradicting $\tilde{\sigma}_1<\tilde{\sigma}_2$. Hence, \eqref{eq:stationary_goal} is impossible.

If $x_1(H^*)=0$ and $x_2(H^*)=0$. Then the fixed-point equation $\sqrt{H^*-p}=g(H^*)$ forces $\sqrt{H^*-p}=0$, so $H^*=p$. Therefore, in this case, \eqref{eq:stationary_goal} implies $H^*=p$, i.e., the zero-congestion corner.

In the following, we prove the sufficient condition for \eqref{eq:case_sep2_interval_strict}. According to the definition, it is equivalent to
\begin{align}
    q\sqrt{\frac{2\tilsigma_{1}^2 d  R_\sft}{H^*_{\sep}(\tilde{\sigma}_2)}} + (1-q)\max\bigg\{0,\sqrt{\frac{2\tilsigma_{2}^2 r  R_\icl}{H^*_{\sep}(\tilde{\sigma}_2)}}- \frac{2\tilsigma_2^2 R_\icl}{\zeta}\bigg\} &<\sqrt{H^*_{\sep}(\tilde{\sigma}_2)-p}\label{ieq:one_type_cond1} \\
    q\sqrt{\frac{2\tilsigma_{1}^2 d  R_\sft}{H^*_{\sep}(\tilde{\sigma}_2)}}+(1-q)\sqrt{\frac{2\tilsigma_{2}^2 d R_\sft}{H^*_{\sep}(\tilde{\sigma}_2)}}&>\sqrt{H^*_{\sep}(\tilde{\sigma}_2)-p}.\label{ieq:one_type_cond2}
\end{align}
Since the sufficient conditions guarantee that 
\begin{align}
    \frac{\tau}{\zeta}\geq \frac{1}{(d-r)}\sqrt{\frac{4drR_{\sft}}{R_{\icl}}},\label{ieq:pi_large}
\end{align}
and Proposition~\ref{prop:H_fixed_points} shows that
\begin{align}
    H^*_{\sep}(\tilde{\sigma}_2) = \frac{\big((d-r)\tau+r\zeta\big)^2}{8dR_{\sft}\tilsigma_2^2}.\label{eq:H_sep_def},
\end{align}
for the second term in the left-hand side of both \eqref{ieq:one_type_cond1}, we have that
\begin{align*}
    \sqrt{\frac{2\tilsigma_{2}^2 r  R_\icl}{H^*_{\sep}(\tilde{\sigma}_2)}}- \frac{2\tilsigma_2^2 R_\icl}{\zeta}= \bigg(\sqrt{\frac{16drR_{\sft}}{\big((d-r)\tau+r\zeta\big)^2}}-\frac{2 \sqrt{R_\icl}}{\zeta}\bigg)\sqrt{R_{\icl}}\tilsigma_2^2\leq 0,
\end{align*}
where the inequality results from \eqref{ieq:pi_large}. Since we require that
\begin{align}
    \frac{\big((d-r)\tau+r\zeta\big)^2}{24dR_{\sft}p}\leq \tilsigma_2^2 \leq \frac{\big((d-r)\tau+r\zeta\big)^2}{12dR_{\sft}p},\label{ieq:sigma_med}
\end{align}
we have that
\begin{align*}
    \sqrt{H^*_{\sep}(\tilde{\sigma}_2)-p}/\sqrt{\frac{2\tilsigma_{1}^2 d  R_\sft}{H^*_{\sep}(\tilde{\sigma}_2)}}>\frac{p}{2\tilsigma_1\sqrt{R_\sft d}}\geq q.
\end{align*}
Thus, \eqref{ieq:one_type_cond1} is satisfied by our sufficient conditions. For \eqref{ieq:one_type_cond2}, by utilizing \eqref{eq:H_sep_def}, we first note that
\begin{align*}
    \sqrt{\frac{2\tilsigma_{2}^2 d R_\sft}{H^*_{\sep}(\tilde{\sigma}_2)}}>\sqrt{H^*_{\sep}(\tilde{\sigma}_2)-p}
\end{align*}
is equivalent to
\begin{align*}
     \frac{\big((d-r)\tau+r\zeta\big)^2}{8dR_{\sft}\tilsigma_2^2}-p<\frac{16d^2R_\sft^2 \tilsigma_2^4}{ \big((d-r)\tau+r\zeta\big)^2}.
\end{align*}
When $((d-r)\tau)^2 >72p^3$, with \eqref{ieq:sigma_med}, the direct calculations can prove this inequality. Thus, to guarantee \eqref{ieq:one_type_cond2}, we only require that
\begin{align*}
    q\leq \frac{\sqrt{\frac{2\tilsigma_{2}^2 d R_\sft}{H^*_{\sep}(\tilde{\sigma}_2)}}-\sqrt{H^*_{\sep}(\tilde{\sigma}_2)-p}}{\sqrt{\frac{2\tilsigma_{2}^2 d R_\sft}{H^*_{\sep}(\tilde{\sigma}_2)}}-\sqrt{\frac{2\tilsigma_{1}^2 d R_\sft}{H^*_{\sep}(\tilde{\sigma}_2)}}}= 1- \frac{\sqrt{H^*_{\sep}(\tilde{\sigma}_2)-p}-\sqrt{\frac{2\tilsigma_{1}^2 d R_\sft}{H^*_{\sep}(\tilde{\sigma}_2)}}}{\sqrt{\frac{2\tilsigma_{2}^2 d R_\sft}{H^*_{\sep}(\tilde{\sigma}_2)}}-\sqrt{\frac{2\tilsigma_{1}^2 d R_\sft}{H^*_{\sep}(\tilde{\sigma}_2)}}}.
\end{align*}
We only need to prove that
\begin{align*}
    \frac{\sqrt{H^*_{\sep}(\tilde{\sigma}_2)-p}-\sqrt{\frac{2\tilsigma_{1}^2 d R_\sft}{H^*_{\sep}(\tilde{\sigma}_2)}}}{\sqrt{\frac{2\tilsigma_{2}^2 d R_\sft}{H^*_{\sep}(\tilde{\sigma}_2)}}-\sqrt{\frac{2\tilsigma_{1}^2 d R_\sft}{H^*_{\sep}(\tilde{\sigma}_2)}}}\leq \frac{1}{2}.
\end{align*}
In fact, we have that
\begin{align*}
    \frac{\sqrt{H^*_{\sep}(\tilde{\sigma}_2)-p}-\sqrt{\frac{2\tilsigma_{1}^2 d R_\sft}{H^*_{\sep}(\tilde{\sigma}_2)}}}{\sqrt{\frac{2\tilsigma_{2}^2 d R_\sft}{H^*_{\sep}(\tilde{\sigma}_2)}}-\sqrt{\frac{2\tilsigma_{1}^2 d R_\sft}{H^*_{\sep}(\tilde{\sigma}_2)}}}\leq \frac{2\sqrt{H^*_{\sep}(\tilde{\sigma}_2)-p}-2\sqrt{\frac{2\tilsigma_{1}^2 d R_\sft}{H^*_{\sep}(\tilde{\sigma}_2)}}}{\sqrt{\frac{2\tilsigma_{2}^2 d R_\sft}{H^*_{\sep}(\tilde{\sigma}_2)}}}\leq \frac{12\sqrt{2}R_{\sft}^{3/2}p^{3/2}}{(d-r)\tau+r\zeta}\leq \frac{1}{2}
\end{align*}
where the first inequality results from $\tilsigma_1\leq \tilsigma_2/2$, the second inequality results from \eqref{ieq:sigma_med}, and the last inequality results from that $(d-r)\tau\geq 84 R_{\sft}^{3/2}p^{3/2} $. Thus, we conclude the proof of Proposition~\ref{prop:only_one_type}.

\subsection{Proof of Proposition~\ref{prop:s_exist}}\label{app:s_exist}
Given a price $p$, the existence of equilibrium $(\pi^*,R^*)$ is already proved in Theorem~\ref{thm:exist_unique}. Thus, in the following, we only need to show that the maximizer of $I(p)$ is finite. We achieve this in three steps.

\textbf{Step 1: Build the continuity of $I(p)$ on $[0,\infty)$.}

To build the continuity of $I(p)$ on $[0,\infty)$, we need to: 1. prove the continuity of $R^*$ on $(0,\infty)$; 2. extend the domain of $I(p)$ to $[0,\infty)$ by $I(0)=\lim_{p\downarrow 0}I(p)$. For both results, we first prove that $R^*(p)$ is bounded for any $p\geq 0$.

\begin{proposition}\label{prop:R_bd}
    Under regularity conditions of $h$, there exists a constant $\barR$ that only depends on $\calT$ and function $h$ such that for $R^*$ induced by any $p\geq 0$, we have that $R^*\leq \barR$.
\end{proposition}
The proof is in Appendix~\ref{app:R_bd}. This proposition helps us to retrict the range of $R^*$ and will serve as a basis to restrict the domain of $R$. Then we prove the contiuity of $R^*$ in the following proposition.
\begin{proposition}\label{prop:conti}
    Under regularity conditions of $h$, $R^{*}(p)$ is continuous in $p$.
\end{proposition}
The proof is in Appendix~\ref{app:conti}. Thus, the function $I(p)$ is continuous in $p$. In addition, since $R^*(p)\leq \barR$, we have that $\lim_{p\downarrow 0}pR^*(p)=0$. As a result, the limit $I(0)=\lim_{p\downarrow 0}I(p)$ is well-defined. With this extension, $I(p)$ is continuous in $p$ for $p\in[0,\infty)$.

\textbf{Step 2: Restrict the domain of $p$.}

To prove that the maximizer of $I(p)$ is not $\infty$, we need to prove that the over-large $p$ will lead to small value of $I(p)$. We prove this in two steps: 1. there exists $p_1<\infty$ such that for any $p\in[p_1,\infty)$, the type-$t$ user will choose \ac{icl} instead of \ac{sft} at equilibrium for $T$-a.e.; 2. there exists $p_2<\infty$ such that for any $p\in[p_2,\infty)$, the optimal \ac{icl} sample number is $0$ for type-$t$ users $T$-a.e. 

We start with the first claim. Define 
\begin{align*}
    \Phi_a(t,H)=\inf_{N\geq 0}\big\{E_a(t,N)+R_a H N\big\}, \text{ and }N_a(t,H)=\argmin_{N\geq 0}\big\{E_a(t,N)+R_a H N\big\}.
\end{align*}
Then we have that
\begin{align*}
    \Phi_\icl(t,H)&\leq E_\icl(t,0)= (d-r)\tau + r\zeta,\\
    \Phi_\sft(t,H)&\geq \inf_{N\geq 0}\bigg\{\frac{2(d-r)\tilsigma^2}{N}+R_\sft H N\bigg\} = 2\sqrt{2(d-r)\tilsigma^2R_\sft H}.
\end{align*}
Thus, when $H>\barH$, $\Phi_\sft(t,H)>\Phi_\icl(t,H)$ for all $t$, where
\begin{align*}
    \barH = \sup_{t\in\calT} \frac{((d-r)\tau + r\zeta)^{2}}{8(d-r)\tilsigma^2R_\sft}.
\end{align*}
Since the congestion level $R^*$ at equilibrium is non-negative, we can set $p_1 = \barH$. Thus, we prove the first claim.

For the second claim, we note that $E_{\icl}$ is convex in $\tilN$. Thus, when $R_{\icl}H$ is larger than the  partial derivative of $E_{\icl}$ with respect to $\tilN$ at $\tilN=0$, $N_\icl(t,H)=0$. The direct calculation shows that this is equivalent to 
\begin{align*}
    H\geq \sup_{t\in\calT}\frac{2\bar{\sigma}^2r(\pi\zeta-\bar{\sigma}^2)}{\pi^2\tilsigma^2 R_\icl}.
\end{align*}
Thus, we can set $p_2$ as the right-hand side of this equation.

Summarizing these two claims, once $p>\max\{p_1,p_2\}$, the $R^*$ at equilibrium is equal to $0$, and leads to $0$ profit.  

\textbf{Step 3: Prove the finitness of $p^*$.}

According to Steps 1 and 2, the optimization of $p$ is search the maximial value of a continuous function $I(p)$ on a compact set $[0,\max\{p_1,p_2\}]$, which has the finite maximizer and profit. Thus, we conclude the proof of Proposition~\ref{prop:s_exist}.

\section{Supporting Propositions}\label{app:supp_prop}
\subsection{Error Analysis of Affine Estimators}
\begin{proposition}\label{prop:uni_err}
    Let $\theta^{*}\in\mathbb{R}^d$ be the parameter sampled from $\calN(0,\Sigma^{*})$, $m\in\mathbb{R}^d$ a reference mean, which is a random vector with expectation $\mu_m$ and covariance $\Sigma_m$,
    $X\in\mathbb{R}^{n\times d}$ a fixed design matrix, and $Y\in\mathbb{R}^n$ obey the linear model
    \begin{align*}
        Y = X\theta^{*} + \varepsilon, \qquad \varepsilon \sim \mathcal{N}(0,\sigma^2 I_n), 
    \end{align*}
    where $\varepsilon$, $\theta^{*}$ and $m$ are independent. Consider the affine estimator
    $
    \hat\theta = m + K(Y - Xm), \quad K\in\mathbb{R}^{d\times n},
    $
    where $K$ depends on $X$ and is a fixed matrix. The error of it is
    \begin{align*}
        \bbE\big[\|\hat{\theta}-\theta^{*}\|^2\big]= \tr\Big((I-KX)\big(\Sigma^*+\Sigma_m+\mu_m\mu_m^\top\big)(I-KX)^\top\Big)
+\sigma^2\tr(K^\top K).
    \end{align*}
\end{proposition}
\begin{proof}[Proof of Proposition~\ref{prop:uni_err}]
Let $A=I_d-KX$ and $\Delta=m-\theta^*$. Then we have that
\begin{align*}
& \bbE\big[\|\hat\theta-\theta^*\|^2\big]
= \bbE\big[\|A\Delta+K\varepsilon\|^2\big]
= \bbE\big[\Delta^\top A^\top A\Delta\big] + \bbE\big[\varepsilon^\top K^\top K\varepsilon\big] + \bbE\big[2\Delta^\top A^\top K\varepsilon\big].
\end{align*}
It is easy to see the last term is $0$, since $\varepsilon$ is zero-mean. Write  $S=\mathbb{E}[\Delta\Delta^\top]$, then we have that
\begin{align*}
&\mathbb{E}\left[\Delta^\top A^\top A\Delta\right]
= \tr\left(A^\top A\mathbb{E}[\Delta\Delta^\top]\right)
= \tr(A S A^\top),\text{ and } \mathbb{E}\left[\varepsilon^\top K^\top K\varepsilon\right]
= \tr\left(K^\top K\mathbb{E}[\varepsilon\varepsilon^\top]\right)
= \sigma^2\tr(K^\top K),
\end{align*}
where $S=\mathbb{E}[(m-\theta^*)(m-\theta^*)^\top]$. In fact, we have that
\begin{align*}
S=\bbE[mm^\top]-\bbE[m\theta^{*\top}]-\bbE[\theta^*m^\top]+\bbE[\theta^*\theta^{*\top}]
=\Sigma_m+\mu_m\mu_m^\top + \Sigma^*.
\end{align*}
Thus, we have that
\begin{align*}
\bbE\left[\|\hat\theta-\theta^*\|^2\right]
= \tr\Big(A\big(\Sigma^*+\Sigma_m+\mu_m\mu_m^\top\big)A^\top\Big)
+\sigma^2\tr(K^\top K).
\end{align*}
We conclude the proof of Proposition~\ref{prop:uni_err}.

\end{proof}

\subsection{Proof of Proposition~\ref{prop:opt_alg}}\label{app:opt_alg}
For \ac{sft}, since the error $E_\sft$ is strictly convex and $\lim_{N\downarrow 0 }E_{\sft}(N)=\infty$. Thus, the minimizer of the cost can be found by the first-order condition. By direct calculations, we have
\begin{align*}
    N_{\sft}(H) = \sqrt{\frac{2\tilsigma^2d}{R_{\sft}H}}, \text{ and } \Phi_{\sft}(H) = 2\sqrt{2\tilsigma^{2}dR_{\sft}H}.
\end{align*}
For \ac{icl}, the error $E_\icl$ is strictly convex. Thus, the minimizer is achieved at the point satisfying the first-order condition or the boundary point. By the direct calculation, we have that
\begin{align*}
    N_{\icl}(H) = \max\bigg\{0,\sqrt{\frac{2\tilsigma^2 r}{R_{\icl}H}}-\frac{2\tilsigma^2}{\zeta}\bigg\}.
\end{align*}
When $R_{\icl}H\geq r\zeta^2/(2\tilsigma^2)$, the value of $\Phi_{\icl}(H)$ is
\begin{align*}
    \Phi_{\icl}(H) = (d-r)\tau+ r\zeta.
\end{align*}
When $R_{\icl}H<r\zeta^2/(2\tilsigma^2)$, the value of $\Phi_{\icl}(R)$ is
\begin{align*}
    \Phi_{\icl}(H) = (d-r)\tau+2\sqrt{2r\tilsigma^2R_{\icl}H}-R_{\icl}\cdot\frac{2\tilsigma^2}{\zeta}\cdot H.
\end{align*}
Thus, we conclude the proof of Proposition~\ref{prop:opt_alg}.

\subsection{Proof of Proposition~\ref{prop:sft_icl_mono}}\label{app:sft_icl_mono}

We prove this by deriving and comparing the bounds of $N_{\icl}$ and $N_{\sft}$. To derive the bounds, we note that for all $a\in\{\icl,\sft\}$, the following holds.
\[
-\,\frac{\partial}{\partial N}E_a(t,N)=\alpha N^{\alpha-1} J_a(t,N^\alpha),
\]
where
\begin{align*}
J_{\sft}(t,x)=\frac{2r\tilde\sigma^2\zeta^2}{(\zeta x+2\tilde\sigma^2)^2}
  +\frac{2(d-r)\tilde\sigma^2}{x^2},\quad 
J_{\icl}(t,x)=
\frac{2r\,\bar\sigma^2\,\tilde\sigma^2\big(\bar\sigma^4 x+\tilde\sigma^2\pi(\pi\zeta-\bar\sigma^2)\big)}
     {(\tilde\sigma^2\pi+\bar\sigma^2 x)^3}.
\end{align*}
To derive the bounds of $N_\sft$ and $N_\icl$, we lower and upper bound $J_{\sft}$ and $J_{\icl}$. For \ac{sft}, we have that 
\begin{align*}
    J_{\sft}(t,x)> K_{\sft}^{\rmL}(t)/x^2 \text{, where } K_{\sft}^{\rmL}(t)=2(d-r)\tilde\sigma^2.
\end{align*}
Since any interior optimizer $N_a(t,H)$ satisfies the first-order condition
\[
R_aH=\alpha N_a(t,H)^{\alpha-1}J_a\big(t,N_a(t,H)^\alpha\big),
\]
$N_\sft(t,H)$ can be lower bounded as
\begin{align*}
    N_{\sft}(t,H)>
\Big(\frac{\alpha K_{\sft}^{\rmL}(t)}{R_{\sft}H}\Big)^{\frac{1}{\alpha+1}}.
\end{align*}
For \ac{icl}, we have that 
\begin{align*}
    J_{\icl}(x)\le K_{\icl}^{\rmU}(t)/x^2 \text{, where } K_{\icl}^{\rmU}(t)=\sup_{x\geq 0}x^2 J_{\icl}(t,x).
\end{align*}
Thus, we can upper bound $N_\icl(t,H)$ as
\begin{align}
    N_{\icl}(t,H)\leq\Big(\frac{\alpha K_{\icl}^{\rmU}(t)}{R_{\icl}H}\Big)^{\frac{1}{\alpha+1}}.\label{ieq:compare}
\end{align}

To ensure that $R_{\sft}N_{\sft}(H)> R_{\icl}N_{\icl}(H)$ for all $H>0$, we only need 
\begin{align*}
R_{\sft}^{\alpha}K_{\sft}^{\rmL}(t)> R_{\icl}^{\alpha}K_{\icl}^{\rmU}(t).
\end{align*}
In the following, we will upper bound $K_{\icl}^{\rmU}(t)$. Let $\beta(t)=(\pi\zeta-\bar\sigma^2)/\bar\sigma^2=1+\pi(\tau+m^2)/\bar\sigma^2$. Then a change of variables shows
\[
K_{\icl}^{\rmU}
=
2r\tilde\sigma^2\sup_{x\geq 0}\frac{x^2(x+\beta(t))}{(1+x)^3}.
\]
For the last term on the right-hand side of this equality, we have
\begin{align*}
    \sup_{x\geq 0}\frac{x^2(x+\beta(t))}{(1+x)^3}< \sup_{x\geq 0}\frac{x+\beta(t)}{x+1} = \beta(t).
\end{align*}
Thus, we only need the following condition for \eqref{ieq:compare}, which is exactly Assumption~\ref{assump:large_large}.
\[
\Big(\frac{R_{\sft}}{R_{\icl}}\Big)^{\alpha}
\ge
\frac{r}{d-r}\beta(t).
\]
Thus, we conclude the proof of Proposition~\ref{prop:sft_icl_mono}.

\subsection{Proof of Proposition~\ref{prop:H_fixed_points}}\label{app:H_fixed_points}
We start with \ac{sft}. We note that within-algorithm equilibrium is defined via 
\begin{align}
    H_a^* = \barh\big(R_a\cdot N_a(H_a^*),p\big) =\big(R_a\cdot N_a(H_a^*)\big)^2+p .\label{eq:H_fixed}
\end{align}
Pluging $N_{\sft}(H)$ in Proposition~\ref{prop:opt_alg} into the fixed point equation \eqref{eq:H_fixed} derives that
\begin{align*}
H = p + \left(\sqrt{\frac{2\tilde\sigma^2 d\,R_{\sft}}{H}}\right)^2
= p + \frac{2\tilde\sigma^2 d\,R_{\sft}}{H}.
\end{align*}
Solving this equation gives us the result that
\begin{align*}
    H_{\sft}^*(p)=\frac{p+\sqrt{p^2+8\tilde\sigma^2 d\,R_{\sft}}}{2}.
\end{align*}

For \ac{icl}, Proposition~\ref{prop:opt_alg} shows that
\begin{align*}
    N_{\icl}(H)
&=\max\left\{0,\sqrt{\frac{2\tilde\sigma^2 r}{R_{\icl}H}}-\frac{2\tilde\sigma^2}{\zeta}\right\}.
\end{align*}
Thus, the fixed point equation \eqref{eq:H_fixed} becomes
\begin{align*}
    \max\left\{0,\sqrt{\frac{2\tilde\sigma^2 rR_{\icl}}{H}}-\frac{2\tilde\sigma^2 R_{\icl}}{\zeta}\right\} = \sqrt{H-p}.
\end{align*}
Since the left-hand side of this equation is monotonically decreasing in $H$, while the right-hand side is increasing in $H$, the equation admits at most one root. When $p\geq \barH$, $H=p$ makes both sides as zero. Thus, the root under this setting is $p$.

For $H_{\sep}^{*}$, we consider two cases: $H\geq\barH$, and $H<\barH$, where $\barH=r\zeta^2/(2R_{\icl}\tilsigma^2)$. When $H\geq\barH$, Proposition~\ref{prop:opt_alg} shows that
\begin{align*}
\Phi_{\icl}(H)=(d-r)\tau+r\zeta,
\qquad
\Phi_{\sft}(H)=2\sqrt{2\tilde\sigma^2 dR_{\sft}H}.
\end{align*}
Thus $\psi(H)=0$ becomes
\begin{align*}
8\tilde\sigma^2 dR_{\sft}H=\left((d-r)\tau+r\zeta\right)^2,
\end{align*}
which is the desired result. The condition $H\geq\barH$ requires that
\begin{align*}
    \frac{R_\sft}{R_\icl}\leq \frac{\big((d-r)\tau+r\zeta\big)^2}{4dr\zeta^2}.
\end{align*}

When $H<\barH$, Proposition~\ref{prop:opt_alg} shows that
\begin{align*}
\Phi_{\icl}(H)=(d-r)\tau + 2\sqrt{2r\tilde\sigma^2R_{\icl}H}-\frac{2\tilde\sigma^2 R_{\icl}}{\zeta}H.
\end{align*}
Therefore $\psi(H)=0$ is equivalent to
\begin{equation}
\label{eq:psi_low}
2\sqrt{2\tilde\sigma^2 dR_{\sft}H}
-2\sqrt{2r\tilde\sigma^2R_{\icl}H}
+\frac{2\tilde\sigma^2 R_{\icl}}{\zeta}H
-(d-r)\tau
=0.
\end{equation}
Let $u=\sqrt{H}>0$. Then $\sqrt{H}=u$ and $H=u^2$, and $\psi(H)=0$ becomes
\begin{align*}
\Bigg(2\sqrt{2}\tilde\sigma(\sqrt{dR_{\sft}}-\sqrt{rR_{\icl}})\Bigg)u
+\Bigg(\frac{2\tilde\sigma^2 R_{\icl}}{\zeta}\Bigg)u^2
-(d-r)\tau=0.
\end{align*}
That is,
\begin{align*}
Cu^2+Bu-(d-r)\tau=0,
\end{align*}
with $B,C$ as
\begin{align*}
B =2\sqrt{2}\tilde\sigma\Big(\sqrt{dR_{\sft}}-\sqrt{rR_{\icl}}\Big),
\qquad
C =\frac{2\tilde\sigma^2 R_{\icl}}{\zeta}.
\end{align*}
Solving for the positive root gives
\begin{align*}
u=\frac{-B+\sqrt{B^2+4C(d-r)\tau}}{2C},
\end{align*}
and hence $H_{\sep}^*=u^2$, which is exactly the results. Thus, we conclude the proof of Proposition~\ref{prop:H_fixed_points}.

\subsection{Proof of Proposition~\ref{prop:R_bd}}\label{app:R_bd}
Before the formal proof, we first define several quantities. For each type $t$, we define the resource-normalized error envelope
\begin{equation*}
g_t(r)=\min_{a\in\{\sft,\icl\}} E_a\!\left(t,\frac{r}{R_a}\right),\qquad r\ge 0.
\end{equation*}
Then we have that
\begin{align*}
    g_t(0)=(d-r)\tau+r\zeta,
\end{align*}
which is a continuous function of the type $t$. Since $\calT$ is compact, we define the maximal value of $g_t(0)$ as
\begin{align*}
    B=\sup_{t\in\calT} (d-r)\tau+r\zeta.
\end{align*}
For each congestion level $R\ge 0$, we define the (type-$t$) optimal resource demand correspondence
\begin{equation}\label{eq:def_D_t}
D(t,R)=\arg\min_{r\ge 0}\Big\{g_t(r)+\barh(R,p)\,r\Big\}.
\end{equation}
Let the aggregate best-response correspondence be
\begin{equation}\label{eq:def_D_bar}
\bar D(R)=\left\{\int_{\calT} d(t)\,T(dt): d(\cdot)\ \text{measurable and }d(t)\in D(t,R)\ \forall t\right\}.
\end{equation}
An equilibrium congestion level is a fixed point $R^*\in\bar D(R^*)$. In the following, we show that $R^{*}$ is bounded..

First, we derive a pointwise upper bound of $\bar D(R)$. Fix any $R>0$ and any $t\in\calT$. Let $r_t\in D(t,R)$ be an optimizer of~\eqref{eq:def_D_t}.
Comparing the optimal value at $r_t$ with the feasible point $r=0$ yields
\[
g_t(r_t)+h(R)\,r_t\ \le g_t(r_t)+\barh(R,p)\,r_t \le\ g_t(0).
\]
Since $g_t(r)\geq 0$, we obtain
\[
h(R)\,r_t\ \le B.
\]
Because $h(R)>0$ for $R>0$ (strictly increasing with $h(0)\geq 0$), it follows that every optimizer satisfies
\begin{equation*}
0\le r_t\ \le\ \frac{B}{h(R)}.
\end{equation*}
Consequently, we have that
\begin{equation}\label{eq:Dbar_upper}
\bar D(R)\ \subseteq\ \Big[0,\ \frac{B}{h(R)}\Big],\qquad \forall R>0.
\end{equation}

Then we show that the fixed point $R^{*}$ is bounded. Define the threshold
\begin{equation}\label{eq:def_Rbar}
\bar R=\inf\Big\{R>0:\ R\,h(R)\ \ge\ B\Big\}.
\end{equation}
Because $h$ is strictly increasing and $h(R)\to\infty$ as $R\to\infty$, we have $R\,h(R)\to\infty$ as $R\to\infty$ and therefore the set in~\eqref{eq:def_Rbar} is nonempty; hence $\bar R<\infty$. In fact, no fixed point can lie above $\bar R$: if $R^*\in\bar D(R^*)$ and $R^*>\bar R$, then~\eqref{eq:Dbar_upper} implies
\[
R^*\ \le\ \frac{B}{h(R^*)},
\]
which rearranges to $R^*\,h(R^*)\le B$, contradicting $R^*>\bar R$ and the definition~\eqref{eq:def_Rbar}. Thus, we conclude the proof of this proposition.

\subsection{Proof of Proposition~\ref{prop:conti}}\label{app:conti}
The main idea of the proof is similar to Step 3 in the proof of Theorem~\ref{thm:exist_unique}. Here, we first extend the notations there to include the influence of $p$.
For each $(R,p)$ and type $t$, define the (resource) best-response correspondence
\[
D(t,R,p)=\arg\min_{r\ge 0}\ \bigl\{ g_t(r)+\bar h(R,p)\,r\bigr\}\text{, where }g_t(r)=\min_{a\in\{\sft,\icl\}} E_a\!\left(t,\frac{r}{R_a}\right).
\]
Let the aggregate best-response correspondence be the Aumann integral
\[
\barD(R,p)=\int_T \hatD(t,R,p)\,T(dt)\subset\mathbb R_+, \text{ where }\hatD(t,R,p) = \mathrm{co}D(t,R,p).
\]
Proposition~\ref{prop:R_bd} shows that we can restrict the range of $R$ to $[0,\barR]$. The equilibrium congestion level $R^*(p)$ is fixed point
\[
R^*(p)\in \bar D(R^*(p),p).
\]

In the Step 3 of Theorem~\ref{thm:exist_unique} proof, we show that $\barD(R)$ there is upper hemicontinuious. Following the exact procedures, we can also prove that $\barD(R,p)$ is upper hemicontinuous in $(R,p)$. Define the fixed-point correspondence
\[
\Gamma(p)=\{R\in[0,\barR]\given R\in \barD(R,p)\},
\]
where $B$ is the bound of $\barD(R,p)$. Then the upper hemicontinuity of $\barD(R,p)$ implies that
\[
\mathrm{Gr}(\Gamma)=\{(p,R)\given p\geq 0,\ R\in[0,B],\ R\in\bar D(R,p)\}
\]
is closed in $\bbR_{+}\times[0,B]$. In addition, Theorem~\ref{thm:exist_unique} shows that 
\begin{align*}
    \mathrm{Gr}(\Gamma)=\{(p,R^*(p))\given p\geq 0\}.
\end{align*}
Thus, $R^*(p)$ is continuous in $p$. We conclude the proof of Proposition~\ref{prop:conti}.

\subsection{Monotonicity and Convexity of Several Functions}\label{app:mono}
\begin{proposition}\label{prop:cvx_func}
Fix $a,b,c,d>0$ and $0<\alpha\le 1$. Define
\begin{align*}
f(x)=\frac{a x^{\alpha}+b}{\big(c x^{\alpha}+d\big)^2},\qquad x\ge 0,
\end{align*}
and write $t=x^\alpha\in[0,\infty)$. Then, for every $x>0$,
\begin{equation}
\label{eq:fpp_factor}
f''(x)=\frac{\alpha\,x^{\alpha-2}}{\big(c t+d\big)^4}\,Q(t),
\end{equation}
where $Q$ is the quadratic polynomial
\begin{equation}
\label{eq:Qt_def}
Q(t)=a c^2(\alpha+1)t^2+2c\Big((2\alpha+1)bc-2a\alpha d\Big)t+(\alpha-1)d(ad-2bc).
\end{equation}
A sufficient condition for the strict convexity of $f$ is 
\begin{align*}
    \frac{ad}{bc}\leq \min\bigg\{2,\frac{2\alpha+1}{2\alpha}\bigg\}.
\end{align*}
\end{proposition}

\begin{proof}
Let $t=x^\alpha$. Define $g(t)=(a t+b)/(c t+d)^2$ so that $f(x)=g(t)$.
We first compute $g'(t)$ and $g''(t)$ as follows.
\begin{align*}
g'(t)=\frac{ad-ac\,t-2bc}{(c t+d)^3}, \text{ and }
g''(t)=\frac{2ac^2 t+2c(3bc-2ad)}{(c t+d)^4}.
\end{align*}

Now use the chain rule for $f(x)=g(t)$ with $t=x^\alpha$ to derive
\begin{equation}
\label{eq:fpp_chain}
f''(x)=g''(t)\,(t')^2+g'(t)\,t'',
\end{equation}
where $t'=\alpha x^{\alpha-1}$, and $t''=\alpha(\alpha-1)x^{\alpha-2}$. Plugging the expressions for $g'(t),g''(t),t',t''$ into \eqref{eq:fpp_chain}, we get
\begin{align*}
f''(x)&=\frac{\alpha x^{\alpha-2}}{(c t+d)^4}
\Big( 2\alpha x^\alpha\big( ac^2 t+c(3bc-2ad)\big)
+(\alpha-1)(ad-ac\,t-2bc)(c t+d)\Big).
\end{align*}
Since $x^\alpha=t$, expand the bracket:
\begin{align*}
&2\alpha t\big( ac^2 t+c(3bc-2ad)\big)
+(\alpha-1)(ad-ac\,t-2bc)(c t+d)\\
&\quad =a c^2(\alpha+1)t^2+2c\big((2\alpha+1)bc-2a\alpha d\big)t+(\alpha-1)d(ad-2bc).
\end{align*}
This yields \eqref{eq:fpp_factor}--\eqref{eq:Qt_def}. For the convexity of the function, we only need to study the sign of $Q(t)$. Write
\begin{align*}
Q(t)=A t^2+B t+C,
\quad
A=ac^2(\alpha+1)>0,
\quad
B=2c\big((2\alpha+1)bc-2a\alpha d\big),
\quad
C=(\alpha-1)d(ad-2bc).
\end{align*}
Since $A>0$, $Q$ is convex in $t$ and attains its minimum over $[0,\infty)$ either at
\begin{align*}
t_*=-\frac{B}{2A}\quad\text{if }t_*\ge 0,
\qquad\text{or at }t=0\quad\text{if }t_*<0.
\end{align*}
Introduce the ratio $u=ad/bc$. Substituting $ad=u\,bc$ into $(A,B,C)$ gives
\begin{align*}
A=\frac{bc^3}{d}u(\alpha+1),\qquad
B=2bc^2\big((2\alpha+1)-2\alpha u\big),\qquad
C=bcd(\alpha-1)(u-2).
\end{align*}
Hence, we have that 
\begin{align*}
t_*=-\frac{B}{2A}
=\frac{d\,(2\alpha u-2\alpha-1)}{c\,u(\alpha+1)}.
\end{align*}
Therefore, $t_*\leq 0$ is equivalent to 
\begin{align*}
    u\leq \frac{2\alpha+1}{2\alpha}.
\end{align*}
Thus, a sufficient condition for the strict convexity is that $u\leq (2\alpha+1)/(2\alpha)$ and $Q(0)\geq 0$, i.e.,
\begin{align*}
    \frac{ad}{bc}\leq \frac{2\alpha+1}{2\alpha} \text{, and }ad-2bc\leq 0.
\end{align*}
Thus, we conclude the proof of this proposition.
\end{proof}

\begin{proposition}\label{prop:mono1}
For $A>0$ and $x>0$, the function
\begin{align*}
f(x)=\frac{\sqrt{A+x}-\sqrt{A}}{x}
\end{align*}
is strictly decreasing on $(0,\infty)$. Moreover,
\begin{align*}
\lim_{x\downarrow 0} f(x)=\frac{1}{2\sqrt{A}},\qquad \lim_{x\to\infty} f(x)=0,
\end{align*}
so $f$ maps $(0,\infty)$ onto $\bigl(0,\frac{1}{2\sqrt{A}}\bigr)$.
\end{proposition}
\begin{proof}
We first prove the first result. Rewrite $f$ by rationalizing:
\begin{align*}
f(x)=\frac{\sqrt{A+x}-\sqrt{A}}{x}\cdot \frac{\sqrt{A+x}+\sqrt{A}}{\sqrt{A+x}+\sqrt{A}}
=\frac{1}{\sqrt{A+x}+\sqrt{A}}.
\end{align*}
Differentiate for $x>0$:
\begin{align*}
f'(x)= -\frac{1}{\left(\sqrt{A+x}+\sqrt{A}\right)^2}\cdot \frac{1}{2\sqrt{A+x}}
= -\frac{1}{2\sqrt{A+x}\left(\sqrt{A+x}+\sqrt{A}\right)^2}.
\end{align*}
Since the denominator is positive for all $x>0$, we have $f'(x)<0$ for all $x>0$.
Hence $f$ is strictly decreasing on $(0,\infty)$. Using the simplified form $f(x)=1/(\sqrt{A+x}+\sqrt{A})$,
\begin{align*}
\lim_{x\downarrow 0} f(x)=\frac{1}{\sqrt{A}+\sqrt{A}}=\frac{1}{2\sqrt{A}},
\qquad
\lim_{x\to\infty} f(x)=\lim_{x\to\infty}\frac{1}{\sqrt{A+x}+\sqrt{A}}=0.
\end{align*}
\end{proof}
\begin{proposition}\label{prop:mono2}
    Consider the function
    \begin{align*}
    f(x)=\frac{(A-x+Bx)^2}{AB^2x},\qquad 0<x<A,\ \ A>0,\ \ B> 0.
    \end{align*}
    If $0<B<2$, then $f$ is strictly decreasing and achieves its minimal value $f(A)=1$ at $x=A$.
    
\end{proposition}
\begin{proof}
Rewrite
\begin{align*}
A-x+Bx = A+(B-1)x,
\end{align*}
so
\begin{align*}
f(x)=\frac{(A+(B-1)x)^2}{AB^2x}
=\frac{1}{AB^2}\,g(x),
\qquad 
g(x)=\frac{(A+cx)^2}{x},\ \ c=B-1.
\end{align*}
Differentiating it, we get:
\begin{align*}
g'(x)=\frac{2c(A+cx)\,x-(A+cx)^2}{x^2}
=\frac{(A+cx)\bigl(2cx-(A+cx)\bigr)}{x^2}
=\frac{(A+cx)(cx-A)}{x^2}.
\end{align*}
Since $x^2>0$ on $(0,A)$, the sign of $g'(x)$ (and hence $f'(x)$) is the sign of
$(A+cx)(cx-A)$.

If $0<B\le 2$ (i.e., $-1<c\le 1$), we have $cx-A<0$ on $(0,A)$ (because $cx\le cA\le A$), and also
\begin{align*}
A+cx\ge A+cA=A(1+c)=AB>0.
\end{align*}
Therefore $g'(x)<0$ and $f'(x)<0$ on $(0,A)$:$f$ is strictly decreasing on $(0,A)$.

Since, $f$ is strictly decreasing on $(0,A)$, so the infimum is
attained at the right endpoint:
\begin{align*}
\inf_{0\leq x\leq A} f(x)=f(A)=\frac{(A+(B-1)A)^2}{AB^2A}=1.
\end{align*}
    
\end{proof}

\begin{proposition}[Unimodality and monotonicity of $g$]\label{prop:monotone_g}
Fix parameters $A,C>0$, $0<B<1$, and assume $A>C$. Define, for $x\in(0,C)$,
\begin{align*}
g(x)=\sqrt{\bigl(\sqrt A-\sqrt x\bigr)^2+B(C-x)}-\bigl(\sqrt A-\sqrt x\bigr).
\end{align*}
Then $g$ is strictly increasing on $(0,x_\ast)$ and strictly decreasing on $(x_\ast,C)$, where the unique critical point
\begin{align*}
x_\ast =\left(\frac{\sqrt A-\sqrt{A-(1-B)C}}{1-B}\right)^2 \in (0,C)
\end{align*}
is the unique maximizer of $g$ on $(0,C)$. Moreover,
\begin{align*}
g'(x)>0 \ \text{for }x\in(0,x_\ast),\qquad g'(x_\ast)=0,\qquad g'(x)<0 \ \text{for }x\in(x_\ast,C),
\end{align*}
and the maximal value satisfies
\begin{align*}
g(x_\ast)=B\sqrt{x_\ast}
=\frac{B\bigl(\sqrt A-\sqrt{A-(1-B)C}\bigr)}{1-B}.
\end{align*}
\end{proposition}

\begin{proof}
Throughout, fix $A>C>0$ and $0<B<1$. For $x\in(0,C)$ define
\begin{align*}
u(x)=\sqrt A-\sqrt x,\qquad s(x)=\sqrt{u(x)^2+B(C-x)}.
\end{align*}
Since $x\in(0,C)$ and $A>C$, we have $\sqrt x<\sqrt C<\sqrt A$, hence $u(x)>0$. Note that
\begin{align*}
g(x)=s(x)-u(x).
\end{align*}

We first calculate the derivative with respect to $x$.
Note that $u'(x)=-\frac{1}{2\sqrt x}$. Next, we have that
\begin{align*}
s'(x)
=\frac{1}{2s(x)}\frac{d}{dx}\bigl(u(x)^2+B(C-x)\bigr)
=\frac{1}{2s(x)}\bigl(2u(x)u'(x)-B\bigr).
\end{align*}
Therefore
\begin{align*}
g'(x)
&=s'(x)-u'(x)
=\frac{2u u'-B}{2s}-u'\\
&=\frac{1}{2s}\bigl(2u u'-B-2su'\bigr).
\end{align*}
Substituting $u'=-\frac{1}{2\sqrt x}$ yields
\begin{align}
g'(x)
&=\frac{1}{2s(x)}\left(\frac{s(x)-u(x)}{\sqrt x}-B\right)
=\frac{1}{2s(x)}\left(\frac{g(x)}{\sqrt x}-B\right).\label{eq:gprime_sign}
\end{align}
Since $s(x)>0$ for all $x\in(0,C)$, the sign of $g'(x)$ is exactly the sign of
\begin{align*}
\frac{g(x)}{\sqrt x}-B,
\quad\text{i.e.,}\quad
\sign(g'(x))=\sign\bigl(g(x)-B\sqrt x\bigr).
\end{align*}
In particular, critical points satisfy
\begin{equation}\label{eq:crit_equation}
g(x)=B\sqrt x.
\end{equation}

To find the critial point, we note that \eqref{eq:crit_equation} is equivalent to
\begin{align*}
s(x)=u(x)+B\sqrt x.
\end{align*}
Both sides are nonnegative, so we may square:
\begin{align*}
u(x)^2+B(C-x)=(u(x)+B\sqrt x)^2=u(x)^2+2u(x)B\sqrt x+B^2x.
\end{align*}
Cancel $u(x)^2$ and divide by $B>0$ to obtain
\begin{align*}
C-x=2u(x)\sqrt x+Bx.
\end{align*}
Using $u(x)=\sqrt A-\sqrt x$ gives
\begin{align*}
C-x=2(\sqrt A-\sqrt x)\sqrt x+Bx
=2\sqrt A\,\sqrt x-2x+Bx,
\end{align*}
hence
\begin{equation}\label{eq:quadratic_in_sqrtx}
C+(1-B)x=2\sqrt A\,\sqrt x.
\end{equation}
Let $y=\sqrt x\in(0,\sqrt C)$, so that $x=y^2$. Then \eqref{eq:quadratic_in_sqrtx} becomes
\begin{equation}\label{eq:q_def}
q(y)=(1-B)y^2-2\sqrt A\,y+C=0.
\end{equation}
The discriminant is
\begin{align*}
\Delta = 4\bigl(A-(1-B)C\bigr)>0,
\end{align*}
because $(1-B)C<C<A$. Thus \eqref{eq:q_def} has two real roots
\begin{align*}
\frac{\sqrt A\pm\sqrt{A-(1-B)C}}{1-B}.
\end{align*}
Both are positive, but the larger one is not in the interval $[0,\sqrt{C}]$ due to $A>C$ and $0<B<1$. Thus, we  define
\begin{align*}
y_\ast=\frac{\sqrt A-\sqrt{A-(1-B)C}}{1-B},
\qquad
x_\ast=y_\ast^2.
\end{align*}

In the following, we show  that $q(y)$ has exactly one zero in $(0,\sqrt C)$, hence \eqref{eq:crit_equation} has exactly one solution
$x_\ast\in(0,C)$. First, $q(0)=C>0$. Next,
\begin{align*}
q' (y)=2(1-B)y-2\sqrt A.
\end{align*}
For $y\in(0,\sqrt C)$,
\begin{align*}
q'(y)\le 2(1-B)\sqrt C-2\sqrt A<0,
\end{align*}
since $(1-B)\sqrt C<\sqrt C<\sqrt A$ (because $A>C$). Hence $q$ is strictly decreasing on $(0,\sqrt C)$.

Moreover,
\begin{align*}
q(\sqrt C)=(1-B)C-2\sqrt A\,\sqrt C+C=C(2-B)-2\sqrt A\,\sqrt C
<2C-2C=0,
\end{align*}
where the strict inequality uses $\sqrt A\,\sqrt C>A$. By continuity and strict monotonicity of $q$ on $(0,\sqrt C)$, there exists a unique
$y_\ast\in(0,\sqrt C)$ with $q(y_\ast)=0$. Therefore $x_\ast=y_\ast^2\in(0,C)$ is the unique solution of
\eqref{eq:crit_equation}, i.e., the unique critical point of $g$ on $(0,C)$.

Then we build the monotonicity analysis of $g$. From \eqref{eq:gprime_sign}, we have that
\begin{align*}
\sign(g'(x))=\sign\bigl(g(x)-B\sqrt x\bigr).
\end{align*}
Define $F(x)=g(x)-B\sqrt x$. The argument above shows $F$ has a unique zero at $x_\ast$.
To determine the sign change, note that $F(0)=g(0)>0$. Hence $F$ is strictly positive to the left
of $x_\ast$ and strictly negative to the right. Consequently,
\begin{align*}
g'(x)>0\ \text{for }x\in(0,x_\ast),\qquad g'(x_\ast)=0,\qquad g'(x)<0\ \text{for }x\in(x_\ast,C).
\end{align*}
Thus $g$ is strictly increasing on $(0,x_\ast)$ and strictly decreasing on $(x_\ast,C)$, so $x_\ast$
is the unique maximizer of $g$ on $(0,C)$.

Finally, plugging the critical-point condition $g(x_\ast)=B\sqrt{x_\ast}$ and the explicit expression
for $\sqrt{x_\ast}=y_\ast$ yields
\begin{align*}
g(x_\ast)=B y_\ast
=\frac{B\bigl(\sqrt A-\sqrt{A-(1-B)C}\bigr)}{1-B},
\end{align*}
as claimed. Thus, we conclude the proof of Proposition~\ref{prop:monotone_g}.
\end{proof}

\begin{proposition}\label{prop:icl_root_property}
    Let $e,f,g>0$ be fixed constants. Consider the equation
\begin{equation}\label{eq:icl_fixed}
\frac{e x}{\sqrt{H}} - f x^2 = \sqrt{H-g},
\end{equation}
where $H \ge g$.
To ensure feasibility, we restrict $x$ to satisfy
$0 < x \le x_{\max} =e/(f\sqrt{g})$. For each such $x$, the solution $H_*(x)$ of \eqref{eq:icl_fixed} exists and is unique. Moreover, there exists $x_* < x_{\max}$ such that $H_*(x)$ is increasing on $(0,x_*]$ and decreasing on $(x_*, x_{\max}]$.
\end{proposition}
\begin{proof}
    We first prove the existence and uniqueness of the root of \eqref{eq:icl_fixed}. Define
\begin{align*}
F(H,x)=\frac{ex}{\sqrt{H}}-f x^2-\sqrt{H-g}.
\end{align*}
For fixed $x>0$, $F(\cdot,x)$ is strictly decreasing on $[g,\infty)$, with
\begin{align*}
F(g,x)=\frac{ex}{\sqrt g}-f x^2\ge 0,
\qquad
\lim_{H\to\infty}F(H,x)=-\infty.
\end{align*}
Hence, for each $x\in(0,x_{\max}]$, there exists a unique solution
\begin{align*}
H_*(x)\ge g.
\end{align*}

Then we analyze the derivative of $H_*$ with respect to $x$. Since $F_H(H_*(x),x)<0$, the implicit function theorem yields
\begin{align*}
H_{*}^{\prime}(x)=-\frac{\partial F}{\partial x}/\frac{\partial F}{\partial H}.
\end{align*}
A direct calculation shows that
\begin{equation}\label{eq:xdprime}
H_{*}^{\prime}(x)=
\frac{\frac{e}{\sqrt{H_*(x)}}-2fx}
{\frac{ex}{2H_*(x)^{3/2}}+\frac{1}{2\sqrt{H_*(x)-g}}}.
\end{equation}
Since the denominator in \eqref{eq:xdprime} is strictly positive, the sign of $H_{*}^{\prime}(x)$ is determined by the numerator. Substituting \eqref{eq:icl_fixed} into the numerator of \eqref{eq:xdprime} yields
\begin{align*}
\frac{e}{\sqrt{H_*(x)}}-2fx
=\frac{1}{x}\sqrt{H_*(x)-g}-fx.
\end{align*}
Hence, the sign of $H_{*}^{\prime}$ is
\begin{equation}\label{eq:sign}
\operatorname{sign}\,H_{*}^{\prime}(x)
=\operatorname{sign}\!\big(\sqrt{H_*(x)-g}-f x^2\big).
\end{equation}

As $x\downarrow 0$, \eqref{eq:icl_fixed} reduces to $\sqrt{H-g}=0$, so
\begin{align*}
\lim_{x\downarrow 0}H_*(x)=g.
\end{align*}
At the upper boundary $x=x_{\max}=e/(f\sqrt g)$,
\begin{align*}
\frac{ex}{\sqrt g}-f x^2=0,
\end{align*}
so $H_*(x_{\max})=g$ by uniqueness. Thus, if there is only one stationary point of $H_{*}$, then $H_{*}$ first increases and then decreases, which proves our claim. In the following, we will show that $H_{*}$ has a unique stationary point.

A stationary point $H_{*}^{\prime}(x_*)=0$ occurs if and only if
\begin{equation}\label{eq:stationary}
\sqrt{H_*(x_*)-g}=f x_*^2.
\end{equation}
Combining \eqref{eq:stationary} with \eqref{eq:icl_fixed} gives
\begin{align*}
H_*(x_*)=\frac{e^2}{4x_*^2 f^2},
\end{align*}
while \eqref{eq:stationary} also implies $H_*(_x*)=g+x_*^4 f^2$.
Equating the two expressions yields
\begin{align*}
\frac{e^2}{4x_*^2 f^2}=g+x_*^4 f^2.
\end{align*}
Letting $z_*=x_*^2$, this becomes the cubic
\begin{equation}\label{eq:cubic}
f^2 z_*^3+g z_*-\frac{e^2}{4f^2}=0.
\end{equation}
Since the derivative $3f^2 z^2+g>0$ for all $z>0$, equation
\eqref{eq:cubic} admits a unique positive root $z_\star$. Thus, we conclude the proof of Proposition~\ref{prop:icl_root_property}

\end{proof}

\begin{proposition}\label{prop:mono4}
Let $A\in\mathbb{R}$ and $B>0$. Define
\begin{align*}
f(x)=\sqrt{(x-A)^2+B}-(x-A), \qquad x>A.
\end{align*}
Then $f$ is strictly decreasing on $(A,\infty)$.
\end{proposition}
\begin{proof}
Set $y=x-A$, so that $y>0$ and
\begin{align*}
f(x)=g(y)=\sqrt{y^2+B}-y.
\end{align*}
We compute the derivative of $g$:
\begin{align*}
g'(y)=\frac{d}{dy}\bigl(\sqrt{y^2+B}-y\bigr)
=\frac{y}{\sqrt{y^2+B}}-1
=\frac{y-\sqrt{y^2+B}}{\sqrt{y^2+B}}<0.
\end{align*}
Therefore, $f$ is strictly decreasing on $(A,\infty)$.
\end{proof}

\section{Experimental Details}\label{app:exp_details}
In this section, we provide the experimental details of our work.
\subsection{Experiments in Sections~\ref{sec:homo_mfcg} and \ref{sec:hetero_mfcg}}

For the experiments in Section~\ref{sec:homo_mfcg} (Figures~\ref{fig:homo_1}–\ref{fig:homo_2}), we compute the \ac{icl} and \ac{sft} errors using \eqref{eq:simplified} and set the congestion function to $h(R)=R^2$. Unless otherwise noted, the parameters are
\begin{align*}
d &= 300,\, r=100,\, \tau=3,\, \tilsigma=1,\, \pi=1,\, \zeta=8,\, \bar{\sigma}=1,\, R_{\sft}=27,\, R_{\icl}=4.5,\\
p&=10,\, m=\sqrt{\zeta-\tau-2\bar{\sigma}^2/\pi}.
\end{align*}
We then sweep the following parameters to produce Figures~\ref{fig:homo_1}–\ref{fig:homo_2}:
\begin{itemize}
    \item Figure~\ref{fig:homo_pi}: $\pi\in [0, 20]$, $p\in[0,35]$.
    \item Figure~\ref{fig:homo_r}: $r\in[0,100]$, $p\in[0,14]$.
    \item Figure~\ref{fig:tilsigma}:$\tilsigma\in[0,12]$, $\pi\in[0.1,1.5]$.
    \item Figure~\ref{fig:r_sft}: $R_{\icl}\in[1,8]$, $R_{\sft}\in [R_{\icl},R_{\icl}+20]$ for each $R_\icl$.
\end{itemize}

For the experiments in Section~\ref{sec:hetero_mfcg}, i.e., Figure~\ref{fig:hetero}, the \ac{icl} and \ac{sft} errors are computed according to \eqref{eq:simplified}. The hyperparameters are set as follows.
\begin{align*}
d &= 300,\, r=100,\, \tau=3,\, \tilsigma_1=0.5,\, \tilsigma_2 = 0.6,\, \pi=1,\, m=1.73,\, \bar{\sigma}=1,\, R_{\sft}=27,\, R_{\icl}=4.5,\\
p&=10,\, \zeta=\tau+2\bar{\sigma}^2/\pi+m^2.
\end{align*}
When deriving Figure~\ref{fig:p1_grad}, we sweep $\tilsigma_1\in[0.15,0.6]$ and $q\in[0,1]$.

\subsection{Experiments in Section~\ref{sec:exp}}

\paragraph{Model Architecture.}
In our experiments, we employ decoder-only Transformer models from the GPT-2 family~\citep{radford2019language}, as implemented in the HuggingFace library~\citep{wolf2020transformers}. Specifically, we experiment with embedding dimension $256$, $12$ layers, $8$ attention heads, covariate dimension $30$, and maximal length $2048$ (about $22$M parameters).

\paragraph{Prompt and Data Representation.}
Tasks are cast as sequence prediction: each prompt contains $k$ in-context input--output pairs and a query. Specifically, for inputs $x_i\in\mathbb{R}^d$ and outputs $f(x_i)=x_i^\top\theta^*+\varepsilon_i$, the prompt is
\[
(x_1,f(x_1),\,x_2,f(x_2),\,\ldots,\,x_k,f(x_k),\,x_{\text{q}}),
\]
where $x_i$ and $x_{\text{q}}$ are sampled i.i.d.\ from the input distribution and $\varepsilon_i\stackrel{\text{i.i.d.}}{\sim}\mathcal{N}(0,0.5)$. Scalar inputs and outputs are embedded by zero-padding to the covariate dimension and applying a learned linear projection; a linear head maps predicted embeddings back to scalars for evaluation. We train a decoder-only Transformer to autoregressively predict each $f(x_i)$ from preceding tokens and finally predict $f(x_{\text{q}})$.

\paragraph{Training Objective and Optimization.}
Pretraining starts from scratch with squared error loss. At each step, we sample a batch of random prompts where each $x_i$ is $r$-sparse ($|\supp(x_i)|=r<d$), and update the parameters to minimize the expected squared prediction error over all prefix positions. We set dropout to $0$ (inputs are always fresh and we observe no overfitting), use batch size $480$, and train for $100$k steps with Adam at learning rate $10^{-4}$. For some tasks, we use a curriculum that gradually increases the effective input dimensionality early in training, which substantially accelerates convergence---especially at larger $d$, where training without curriculum can exhibit a long initial plateau.

\ac{sft} initializes from the pretrained parameters and fine-tunes using full-dimensional inputs ($r=d$) with no curriculum. We first generate an \ac{sft} dataset of the prescribed size, then fine-tune for $100$ epochs using Adam (learning rate $10^{-4}$, batch size $480$). We save the end-of-epoch checkpoint and report the best-performing epoch (on the validation set), corresponding to the optimal regularization $\lambda^*$ in Proposition~\ref{prop:comp}. Finally, we enforce a minimum of $100$ gradient updates to ensure the fine-tuned model differs from the pretrained initialization.

\paragraph{Evaluation Settings.} To evaluate the pretrained model, we prompt it with the sequence $(x_i,f(x_i))$ for $i\in[\tilN]$ followed by the query input $x_{\mathrm{q}}$. Note that at evaluation time, the covariates may have larger support $|\supp(x_i)|$ than in pretraining. We report the squared error $(x_{\mathrm{q}}^\top\theta^*-\hat y)^2$, where $\hat y$ is the model's prediction. We evaluate the \ac{sft} model in the same manner. For a fair comparison between \ac{sft} and \ac{icl}, in Figure~\ref{fig:sft} we prompt the \ac{sft} model with $500$ in-context examples; Figure~\ref{fig:icl} shows that the \ac{icl} error has essentially plateaued by $500$ examples.

\section{Additional Experimental Results}\label{app:add_exp}
\subsection{More Results of Two-type LLM-serving Congestion Games}

In this section, we provide more experimental results about Proposition~\ref{prop:only_one_type}. First, in addition to the gradient norm in Figure~\ref{fig:p1_grad}, we provide the value of $R^*$ and the regime that the equilibrium belongs to.

\begin{figure}[H]
\centering
\subfigure[Congestion level $R^*$ with various target task noise $\tilsigma_1$ and type-$1$ ratio $q$.]{\includegraphics[width=0.32\textwidth]{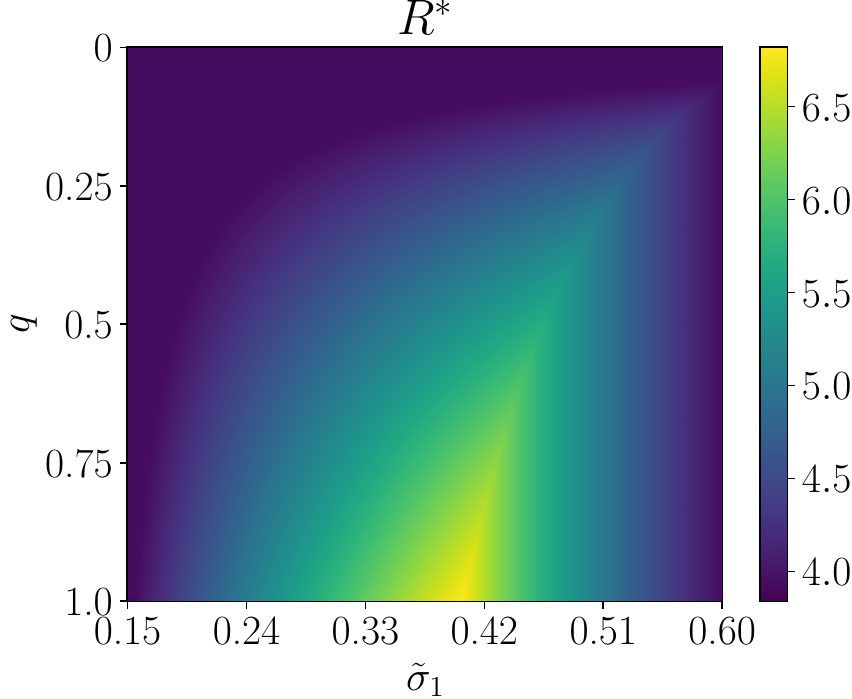}}
\subfigure[Congestion level gradient norms with various target task noise $\tilsigma_1$ and type-$1$ ratio $q$.]{\includegraphics[width=0.32\textwidth]{figures/hetero_R_grad_p10.pdf}}
\subfigure[Regime of the equilibrium with various target task noise $\tilsigma_1$ and type-$1$ ratio $q$.]{\includegraphics[width=0.32\textwidth]{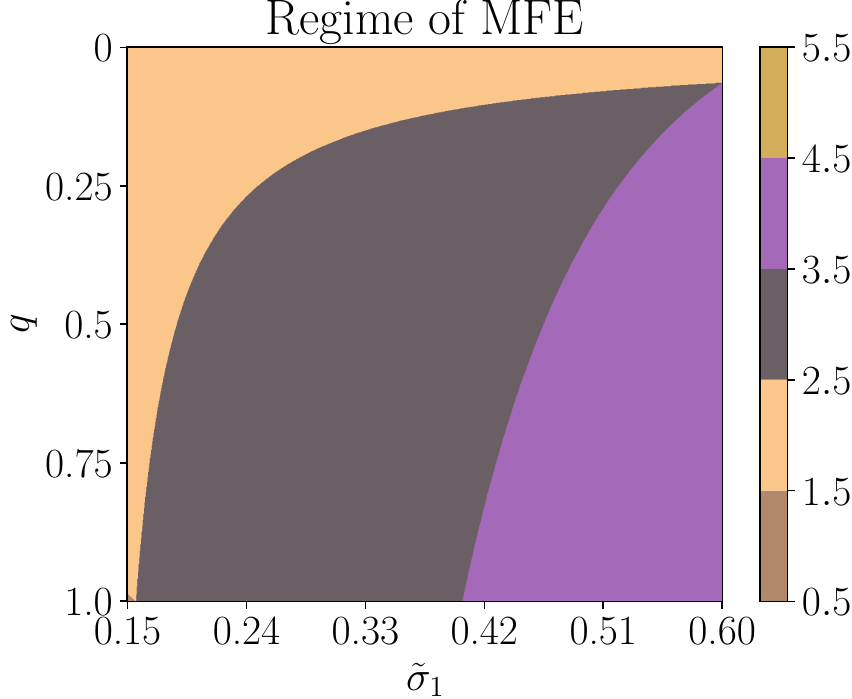}}
\caption{Figures (a)-(c) show the congestion level $R^*$, the gradient of $R^*$ (i.e., $|\partial_{q}R^{*}|+|\partial_{\tilsigma_1}R^{*}|$), and the regime of the equilibrium with various target task noise $\tilsigma_1$ and type-$1$ ratio $q$.}
\label{fig:pin_2}
\end{figure}
The regime of the equilibrium is labeled according to Theorem~\ref{thm:hetero_full}. For example, if $H^{\ast}=H_{\sep}^{\ast}(t_2)$, then the equilibrium belongs to regime $2$. Figure~\ref{fig:pin_2} shows the case where the equilibrium is pinned at the separator of type $2$ in Proposition~\ref{prop:only_one_type}. The following figures show the case where $R^*=0$.

\begin{figure}[H]
\centering
\subfigure[Congestion level $R^*$ with various target task noise $\tilsigma_1$ and type-$1$ ratio $q$.]{\includegraphics[width=0.32\textwidth]{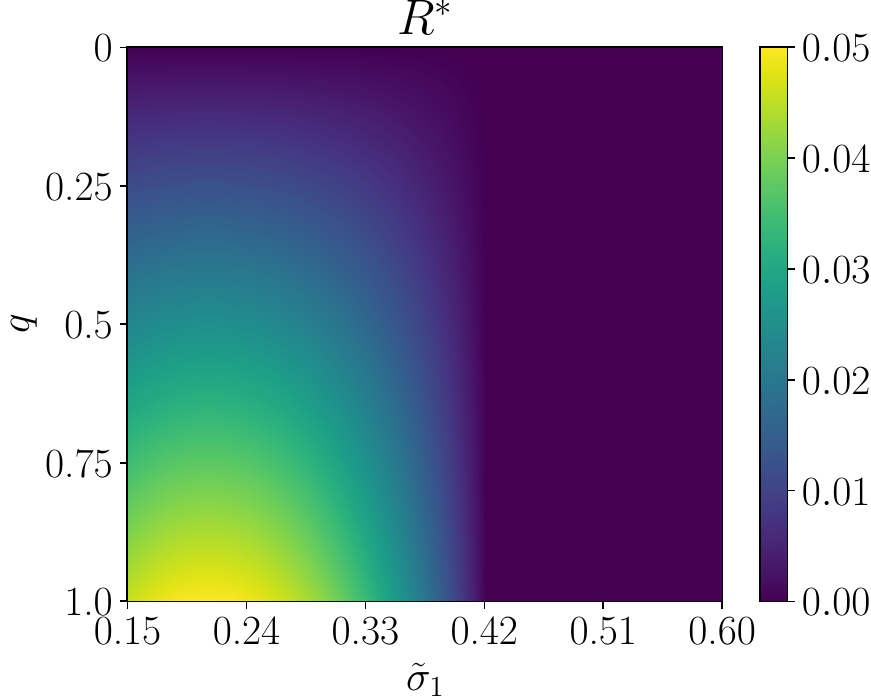}}
\subfigure[Congestion level gradient norms with various target task noise $\tilsigma_1$ and type-$1$ ratio $q$.]{\includegraphics[width=0.32\textwidth]{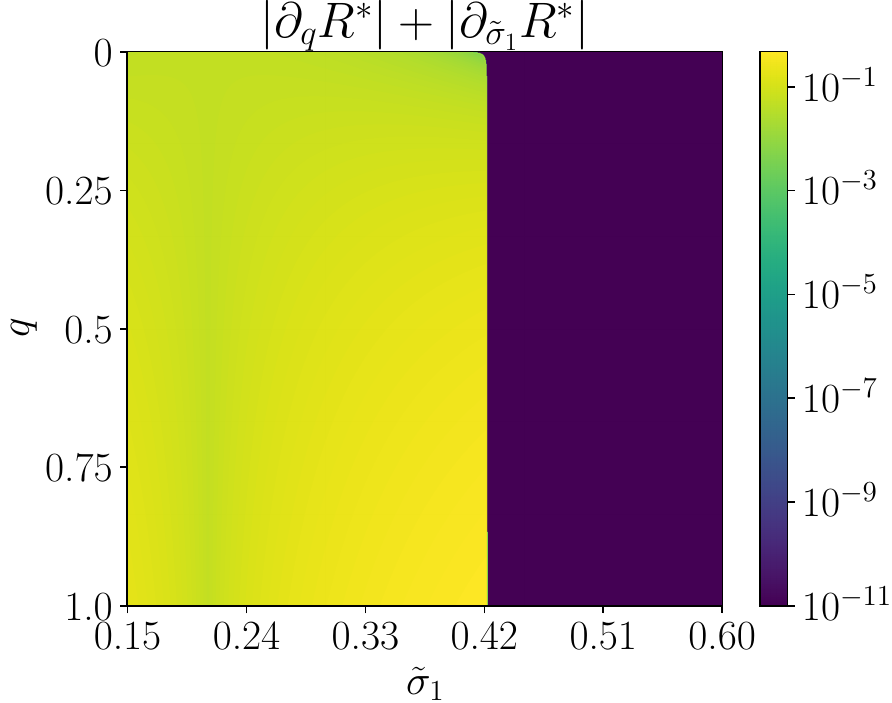}}
\subfigure[Regime of the equilibrium with various target task noise $\tilsigma_1$ and type-$1$ ratio $q$.]{\includegraphics[width=0.32\textwidth]{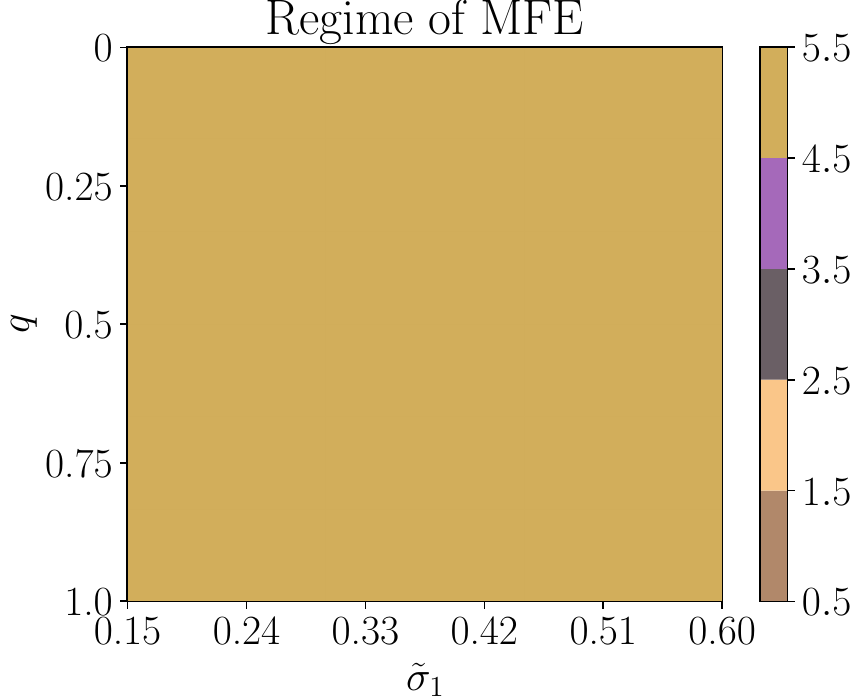}}
\caption{Figures (a)-(c) show the congestion level $R^*$, the gradient of $R^*$ (i.e., $|\partial_{q}R^{*}|+|\partial_{\tilsigma_1}R^{*}|$), and the regime of the equilibrium with various target task noise $\tilsigma_1$ and type-$1$ ratio $q$.}
\end{figure}
Here we set $p=11$ in these figures.

\subsection{Simulation Results for $h(x)=\max\{0,x\}$}\label{app:linear}
\begin{figure}[H]
\centering
\subfigure[The values of $R^{*}$ with various $\pi$.]{\includegraphics[width=\myfigwidth]{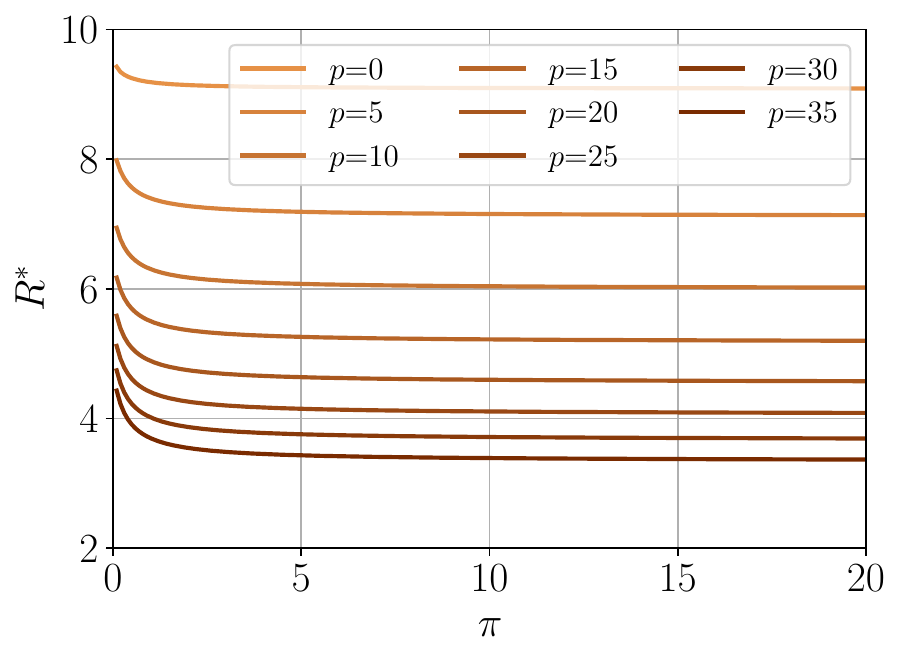}}
\hspace{0.3em}
\subfigure[The values of $R^{*}$ with various $r$.]{\includegraphics[width=\myfigwidth]{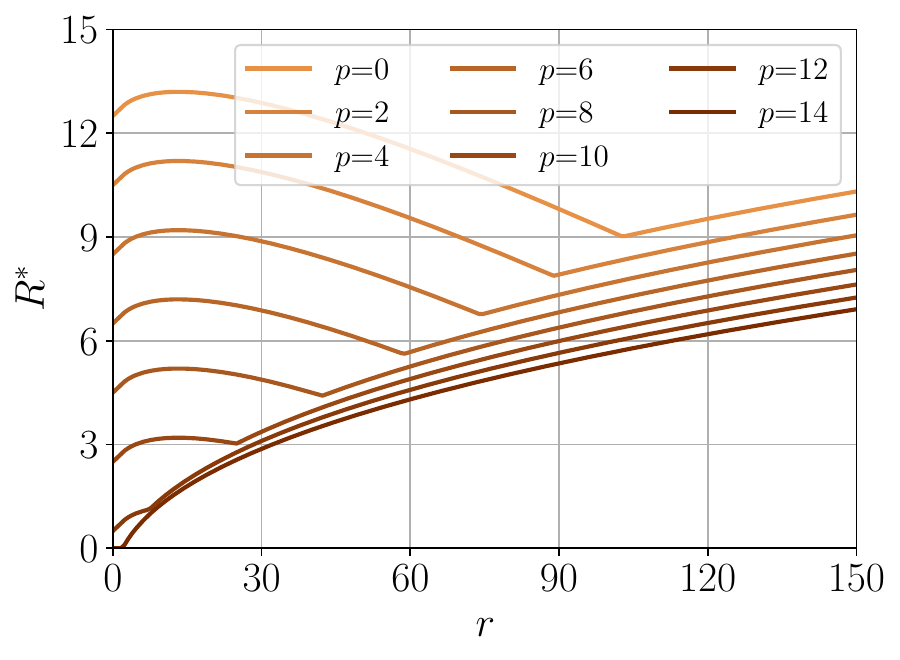}}
\vspace{-0.5em}
\caption{These figures illustrate how the equilibrium congestion level $R^{*}$ varies with prior precision $\pi$ and pretraining coverage $r$ of the pretrained LLM under $h(x)=\max\{0,x\}$.}
\vspace{-0.5em}
\end{figure}

\begin{figure}[H]
\centering
\subfigure[The values of $R^{*}$ with various $\tilde{\sigma}$.]{\includegraphics[width=\myfigwidth]{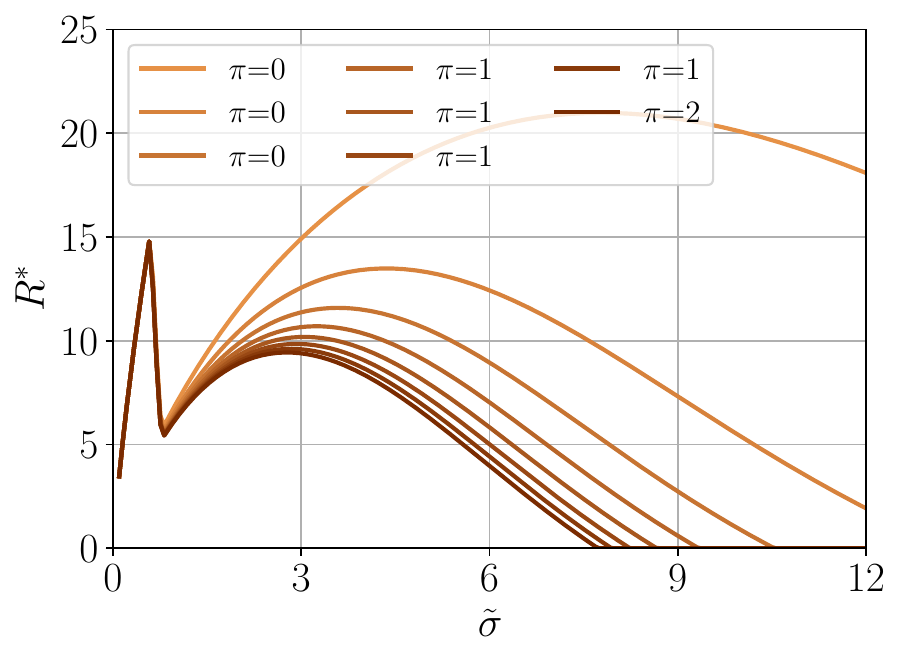}}
\hspace{0.3em}
\subfigure[The values of $R^{*}$ with various $R_{\text{SFT}}$.]{\includegraphics[width=\myfigwidth]{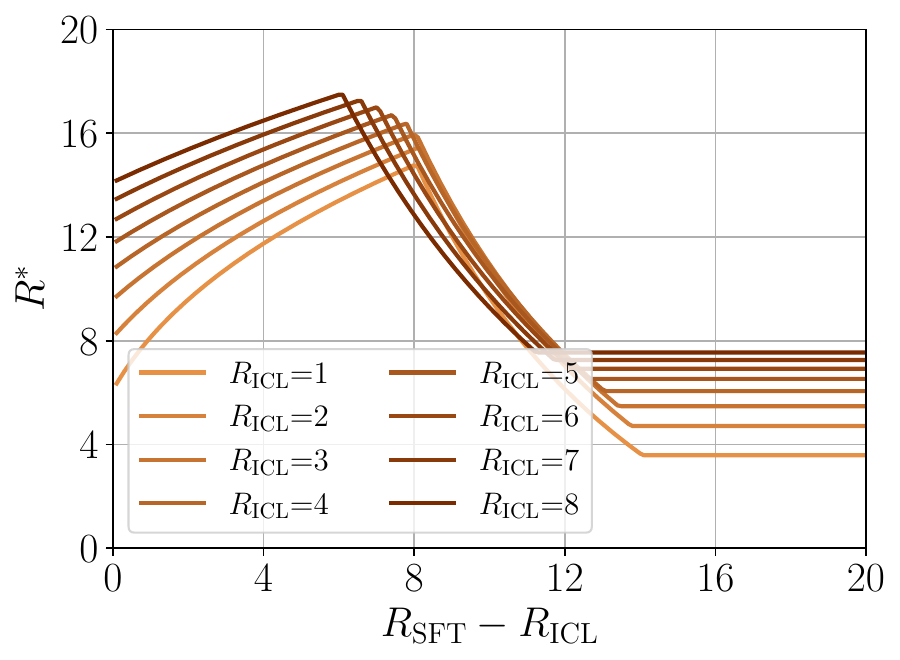}}
\vspace{-0.5em}
\caption{These figures illustrate how the equilibrium congestion level $R^{*}$ varies with the personalization data noise $\tilde{\sigma}$ and the per-sample resource consumption of SFT $R_{\text{SFT}}$ and ICL $R_{\text{ICL}}$ under $h(x)=\max\{0,x\}$.}
\vspace{-1.0em}
\end{figure}

\subsection{Simulation Results for $h(x)=\exp(x)$}\label{app:exponential}
\begin{figure}[H]
\centering
\subfigure[The values of $R^{*}$ with various $\pi$.]{\includegraphics[width=\myfigwidth]{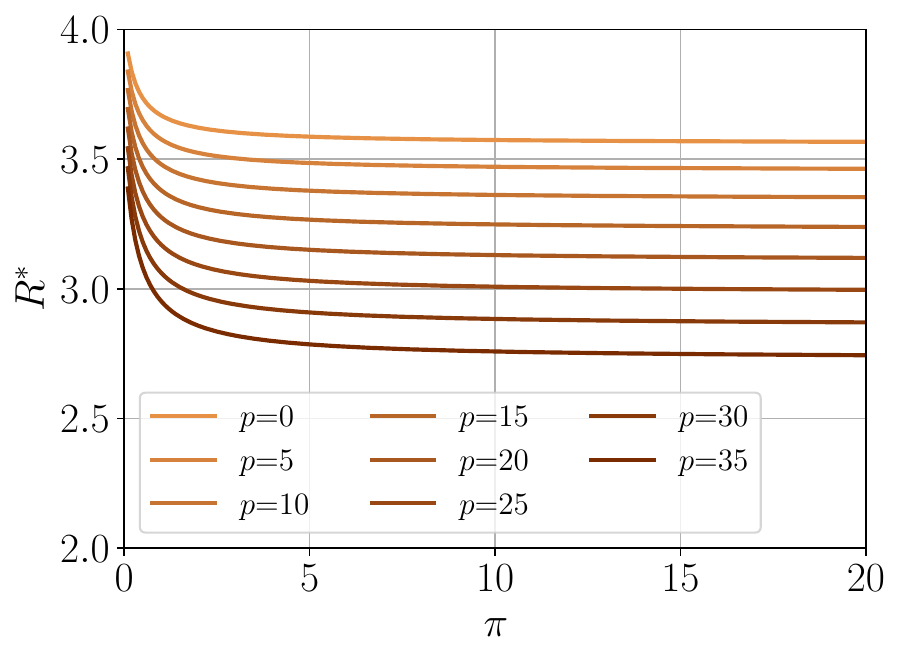}}
\hspace{0.3em}
\subfigure[The values of $R^{*}$ with various $r$.]{\includegraphics[width=\myfigwidth]{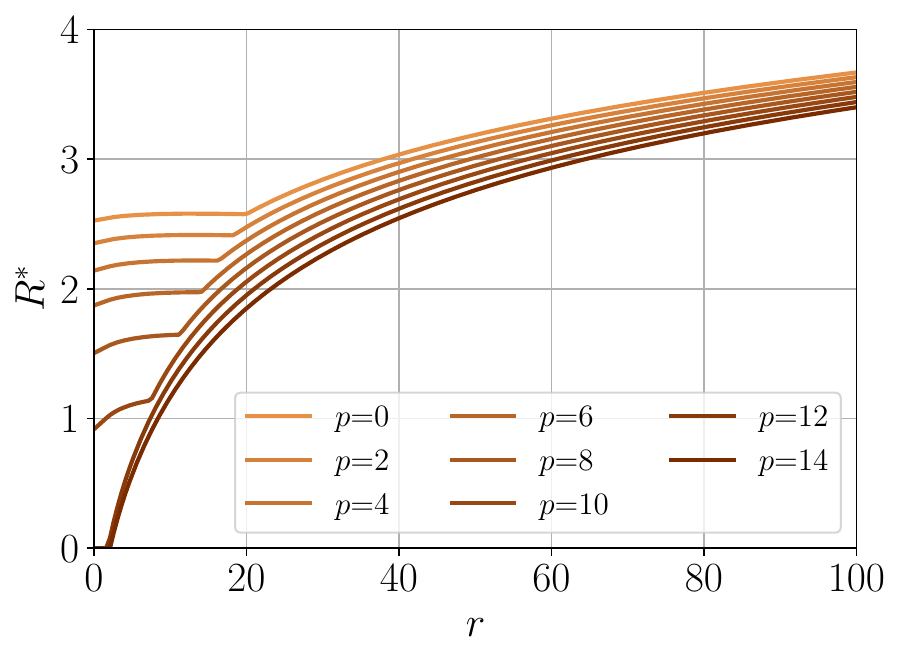}}
\vspace{-0.5em}
\caption{These figures illustrate how the equilibrium congestion level $R^{*}$ varies with prior precision $\pi$ and pretraining coverage $r$ of the pretrained LLM under $h(x)=\exp(x)$.}
\vspace{-0.5em}
\end{figure}

\begin{figure}[H]
\centering
\subfigure[The values of $R^{*}$ with various $\tilde{\sigma}$.]{\includegraphics[width=\myfigwidth]{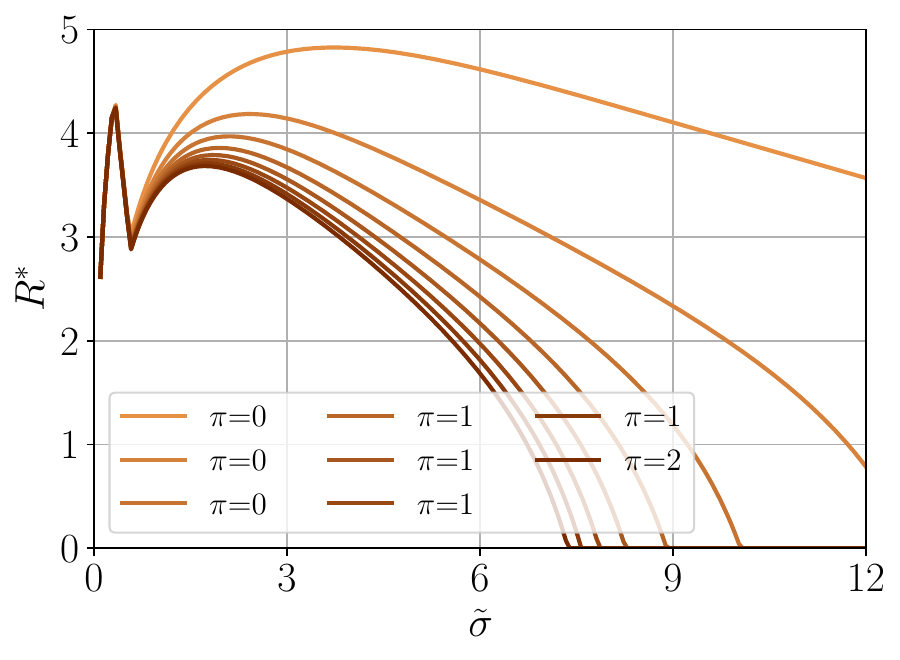}}
\hspace{0.3em}
\subfigure[The values of $R^{*}$ with various $R_{\text{SFT}}$.]{\includegraphics[width=\myfigwidth]{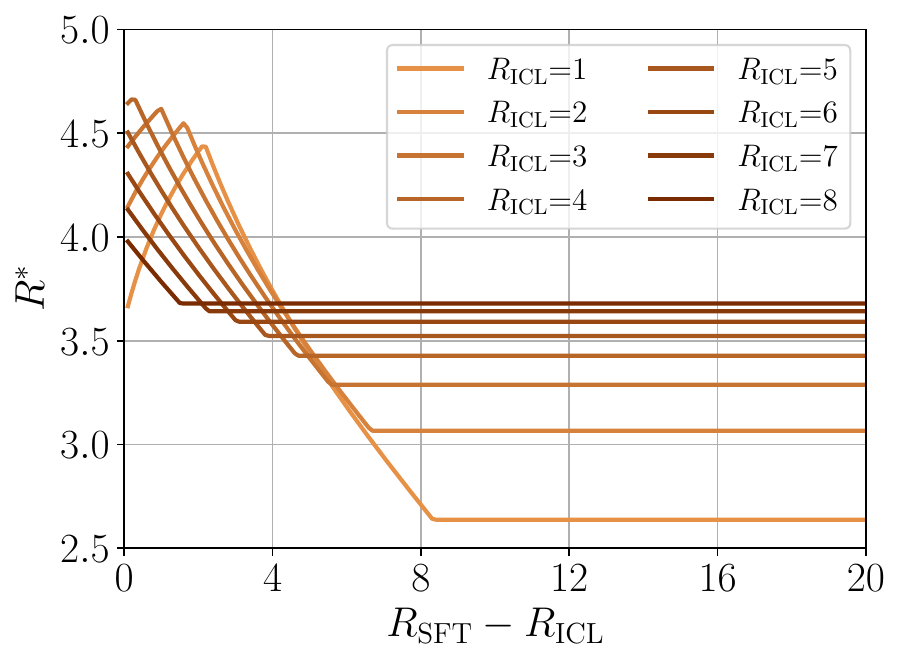}}
\vspace{-0.5em}
\caption{These figures illustrate how the equilibrium congestion level $R^{*}$ varies with the personalization data noise $\tilde{\sigma}$ and the per-sample resource consumption of SFT $R_{\text{SFT}}$ and ICL $R_{\text{ICL}}$ under $h(x)=\exp(x)$.}
\vspace{-1.0em}
\end{figure}

\subsection{\ac{sft} Services of AI Companies}

We document whether and when major AI companies provide \ac{sft} API services. Our analysis focuses on companies that released significant foundation models covered in the Artificial Intelligence Index Report 2024~\citep{maslej2024artificial} and 2025~\citep{maslej2025artificial}. Specifically, the companies in our sample are: AI21 Labs, Alibaba, Amazon, Anthropic, Apple, Baidu, ByteDance, Cohere, Databricks, DeepSeek, Google, Inflection AI, Meta, Midjourney, Mistral AI, Nvidia, OpenAI, Runway, Salesforce, Stability AI, and xAI.

In our analysis, we define “hosting a \ac{sft} API” as offering an interface that allows users to fine-tune the model through the company’s API. We do not count releasing model weights as hosting an SFT API, because making an open-weight model available does not, by itself, constitute an API-based fine-tuning service or directly generate revenue for the provider. The results are listed in Table~\ref{tab:sft_api}.
\begin{longtable}{p{2cm}p{3.5cm}p{2cm}p{7cm}}
\caption{Category and release date of SFT API}\\
\toprule
Company & Offering Category & Release date & Website Link \\
\midrule
\endfirsthead

\toprule
Company & Offering Category & Release date & Website Link \\
\midrule
\endhead

AI21 Labs & Self-hosted & Aug, 2021 &
\url{https://www.ai21.com/blog/announcing-ai21-studio-and-jurassic-1/} \\

Alibaba & Self-hosted & Dec, 2023 &
\url{https://www.alibabacloud.com/blog/analyticdb-a-vector-database-recommended-by-openai---more-than-just-a-vector-index_600692} \\

Amazon & Self-hosted & Nov, 2023 &
\url{https://aws.amazon.com/cn/blogs/aws/customize-models-in-amazon-bedrock-with-your-own-data-using-fine-tuning-and-continued-pre-training/} \\

Anthropic & Partner-hosted & Nov, 2024 &
\url{https://aws.amazon.com/cn/blogs/aws/fine-tuning-for-anthropics-claude-3-haiku-model-in-amazon-bedrock-is-now-generally-available/} \\

Apple & N.A. & N.A. &
N.A. \\

Baidu & Self-hosted & Sep, 2024 &
\url{https://x.com/Baidu_Inc/status/1831693797459112114} \\

ByteDance & Self-hosted & Sep, 2023 &
\url{https://www.volcengine.com/docs/82379/?lang=zh}, \url{https://technode.com/2023/06/29/bytedances-volcengine-unveils-ai-model-service-platform-volcano-ark/} \\

Cohere & Self-hosted & Nov, 2023 &
\url{https://cohere.com/blog/fine-tuning-suite} \\

Databricks & Self-hosted & Jul, 2024 &
\url{https://www.databricks.com/company/newsroom/press-releases/databricks-unveils-new-mosaic-ai-capabilities-help-customers-build} \\

DeepSeek & N.A. & N.A. & N.A. \\

Google & Partner-hosted & Aug, 2023 &
\url{https://docs.cloud.google.com/vertex-ai/generative-ai/docs/models/gemini-use-supervised-tuning} \\

Inflection AI & Self$+$Partner-hosted & Oct, 2024 &
\url{https://newsroom.intel.com/artificial-intelligence/inflection-ai-intel-launch-enterprise-ai-system} \\

Meta & N.A. & N.A. & N.A. \\

Midjourney & N.A. & N.A. &
N.A. \\

Mistral AI & Self-hosted & Jun, 2024 &
\url{https://mistral.ai/news/customization} \\

NVIDIA & Self-hosted & Sep, 2022&
\url{https://nvidianews.nvidia.com/news/nvidia-launches-large-language-model-cloud-services-to-advance-ai-and-digital-biology} \\

OpenAI & Self-hosted & Dec, 2021 &
\url{https://openai.com/index/customizing-gpt-3} \\

Runway & Self-hosted & Oct, 2025 &
\url{https://runwayml.com/product/model-fine-tuning} \\

Salesforce & Self-hosted & Sep, 2023 &
\url{https://www.salesforce.com/news/press-releases/2023/09/12/ai-einstein-news-dreamforce/} \\

Stability AI & N.A. & N.A. & N.A. \\

xAI & N.A. & N.A. & N.A. \\

\bottomrule
\label{tab:sft_api}
\end{longtable}

\end{document}